\documentclass{article}

\PassOptionsToPackage{round, authoryear}{natbib}

\usepackage[main, final]{neurips_2026}

\usepackage[utf8]{inputenc} %
\usepackage[T1]{fontenc}    %
\usepackage[hidelinks]{hyperref}       %
\usepackage{url}            %
\usepackage{booktabs}       %
\usepackage{amsfonts}       %
\usepackage{nicefrac}       %
\usepackage{microtype}      %
\usepackage{xcolor}         %
\usepackage{amsmath, amssymb, amsthm}
\usepackage{mathtools}
\usepackage[ruled,vlined,linesnumbered]{algorithm2e}
\usepackage{tcolorbox}
\usepackage{wrapfig}
\usepackage{capt-of}  %

\makeatletter
\newcommand{\appendixtoc}{\@starttoc{apx}}

\makeatother

\tcbuselibrary{skins, breakable}

\newtcolorbox{remark}[1][]{
  enhanced,
  breakable,
  colback=gray!10,
  colframe=gray!50,
  fonttitle=\bfseries\color{black},
  title=Remark,
  #1
}

\newtheorem{proposition}{Proposition}
\newtheorem{lemma}{Lemma}
\newtheorem{corollary}{Corollary}
\newtheorem{definition}{Definition}

\newcommand{\piref}{\pi_{\theta_{\mathit{ref}}}}

\newcommand{\KL}{\mathit{KL}}

\newcommand{\eqdef}{\overset{\Delta}{=}}

\DeclareMathOperator{\argmax}{\arg \max}
\DeclareMathOperator{\argmin}{\arg \min}

\newcommand{\indep}{\perp\!\!\!\perp}

\newcommand{\setA}{\mathcal{A}}
\newcommand{\setB}{\mathcal{B}}
\newcommand{\setC}{\mathcal{C}}

\newcommand{\setL}{\mathcal{L}}

\newcommand{\setO}{\mathcal{O}}

\newcommand{\setZ}{\mathcal{Z}}

\newcommand{\old}[1]{}

\newcommand{\bbE}{\mathbb{E}}



\title{Follow the Winners: Conservative Policy Improvement with the Cross-Entropy Method for Critic-Free RFT}

\author{%
  Joery A. de Vries$^{1}$ \qquad Neil D. Lawrence$^{1,2}$ \qquad Zhenwen Dai$^{1}$ \\[0.6em]
  $^{1}$Trent AI Limited, London, United Kingdom \\
  $^{2}$Computer Laboratory, University of Cambridge, Cambridge, United Kingdom \\[0.4em]
  \texttt{\{joery,neil,zhenwen\}@trent.ai} \qquad \texttt{ndl21@cam.ac.uk} \\
}

\begin{document}

\maketitle

\begin{abstract}
Critic-free reinforcement fine-tuning (RFT) for agentic large language models is often done through GRPO-style methods, which compute a group baseline over repeated rollouts to reduce target variance.
However, this setup is ill-suited to agents acting in stateful environments such as live services or security sandboxes, where repeated rollouts are impractical to obtain and aggressive updates entrench the noise of long, sparsely verified trajectories.
We propose \textit{Follow the Winners} (FTW), a critic-free policy-learning algorithm that adapts the cross-entropy method to RFT, replacing group rollouts with an ordinal filter on replay-buffer samples that yields polynomial concentration in the order statistic of returns.
We derive FTW through a control-as-inference lens, which also recovers GRPO and DPO as specific modelling choices, identifying GRPO as risk-neutral while DPO and FTW share a bounded risk-seeking offset that FTW controls.
We identify this offset as an inherent trade-off of variance reduction through ordinal filters on samples, whereas a critic model induces a different trade-off between bias and variance.
Scaled to agentic LLM post-training, FTW matches GRPO and PPO on Sokoban and Search-R1 baselines, showing a viable trade-off from a value model or group rollouts to CPU memory.
\end{abstract}

\section{Introduction}  \label{section:intro}

Agentic large language models (LLMs) are increasingly post-trained with reinforcement learning fine-tuning (RFT) to act on software systems \citep{zhang2026agenticllmsurvey, wei2025swerl}, where tasks range from writing or repairing code \citep{wei2025swerl}, to investigating security incidents through logs and configurations \citep{wang2024rcagent, roy2024rca}, to discovering, reproducing, and autonomously patching security vulnerabilities \citep{zhang2025cybench, lee2025secbench, wang2025cvebench}.
But, when an agent trajectory spans many tool calls and reasoning chains, reproducibility and variance become increasingly problematic due to sparse verifiable outcomes, partial observability, or context compression.
Even under deterministic rewards, the resulting trajectory returns exhibit high variance due to long horizons, partial observability, and state aliasing \citep{kaelbling1998planning, spaan_partially_2012, ouyang2022training, yao2023react}.
Policy-gradient methods such as PPO \citep{schulman2017proximal} partly reduce variance with a learned critic, at the cost of GPU memory for a second model and possible bias depending on how we bootstrap \citep{schulman2016gae}.
Critic-free methods such as GRPO \citep{shao2024deepseekmath} instead compute a baseline from a group of repeated rollouts per prompt, which moves the cost to throughput and requires every rollout in the group to have identical initial states.
This assumption is fine for static code-writing, but is unrealistic in live or production systems where agents operate in stateful sandboxes.
Moreover, despite these variance-reduction methods, most policy gradient methods still concentrate aggressively on absolute returns \citep{mei2020escaping}.
% Instead, a practical critic-free RFT algorithm should update more conservatively to reduce noise entrenchment while drawing variance reduction from sources without GRPO's or PPO's constraints.

These costs compound when trajectories are expensive to collect, comparable rollouts are difficult to obtain, and returns depend on sparse or approximate verification. 
For instance, in a security workflow, an attack-chain investigation may yield a single useful trajectory whose initial sandbox state is costly or impractical to reconstruct to gather more data.
Likewise, a long-horizon inspection of logs to locate a bug or vulnerability can drift, hallucinate, or reason incorrectly, so that two trajectories on similar incidents yield different returns.
Such data and post-training considerations suggest that a practical critic-free RFT algorithm should update conservatively to reduce noise entrenchment while drawing variance reduction from sources without GRPO's or PPO's stringent constraints.

We propose \textit{Follow the Winners} (FTW), a critic-free policy-learning algorithm that induces polynomial concentration in the order statistic of returns.
This behavior stems from our trajectory filtering procedure, which adapts the cross-entropy method \citep{rubinstein_optimization_1997, rubinstein1999cem, botev2013cemopt} to deep RL through minibatched filtering over a replay-buffer proposal that amortizes data generation and serves as the alternative source of variance reduction.
We show that the well-known post-training algorithm DPO \citep{rafailov2023direct} and FTW share a bounded risk-seeking offset, which we identify as inherent to reducing variance through ordinal filters on samples.
Furthermore, we show that GRPO is risk-neutral and that FTW controls its risk offset through its selection pressure.

Our analysis is grounded in the Control-as-Inference (CAI) framework \citep{ziebart_modeling_2010, levine2018reinforcement, Kappen2012}, which we extend by recasting critic-free RFT as inference over potentials defined on trajectory batches rather than individual trajectories.
The resulting variational objective motivates the cross-entropy method as a simple sampling-and-projection scheme for optimization, from which we derive FTW.
Under the same view, DPO emerges as the pairwise special case of a Plackett--Luce potential \citep{luce1959, plackett1975, rafailov2023direct}, while GRPO \citep{shao2024deepseekmath} optimizes a Rao--Blackwellized upper bound on the evidence of a multinomial potential.
This places GRPO, DPO, and FTW in a shared design space in which we can juxtapose their modelling decisions and risk offsets, a second theoretical axis of trade-offs next to the practical considerations.

We validate our theory predictions on diagnostic bandit problems and show that FTW extends as a general policy learning method in Brax continuous-control.
Scaled to agentic LLM post-training, FTW trains more stably and with lower seed variance than REINFORCE++ \citep{hu2025reinforcestabilizingcriticfreepolicy}, and matches GRPO and PPO on Sokoban and the published Search-R1 baselines but with our alternative variance-reduction scheme, i.e., no value model, group rollouts or environment duplication.
The LLM results show that the combined effect of FTW's polynomially concentrating learning rule and replay-buffer proposal is a viable alternative to critics or group rollouts for counteracting variance.

\section{Background}  \label{section:background}

\paragraph{Large Language Models.}
A large language model (LLM) defines an autoregressive distribution over token sequences that typically uses an attention-based neural architecture \citep{vaswani2017attention}. 
We write this as $p_\theta(X_{1:T}) = \prod_{t=1}^{T} p_\theta(X_t \mid X_{<t})$, where $X_{<t} = (X_1, \dots, X_{t-1})$ denotes the prefix of previously generated tokens, $X_t$ denotes the next token, and $\theta$ parameterizes the model. 
The parameters $\theta^*_{\mathrm{pre}}$ are fitted through supervised pre-training. 
Given a dataset of example sequences $\mathcal{D} = \{x_{1:T}^{(i)}\}_{i=1}^N$, this corresponds to solving $\theta^*_{\mathrm{pre}} = \arg\max_\theta \mathbb{E}_{x_{1:T} \sim \mu_{\mathcal{D}}}\ln p_\theta(X_{1:T}=x_{1:T})$. 
In practice, many additional techniques are used to stabilize and improve training, including positional encodings \citep{su2024roformer}, layer normalization \citep{ba2016layer}, gradient clipping, and tokenization.

\paragraph{Reinforcement Learning for Agentic LLMs.}
Agentic problems are partially observable Markov decision processes \citep[POMDPs, ][]{kaelbling1998planning, spaan_partially_2012}, specified by states $s \in \mathcal{S}$, observations $o \in \Omega$, actions $a \in \mathcal{A}$, a reward function $R(s_t, a_t) \eqdef R_t$, transition dynamics $p(S_{t+1} \mid S_t=s, A_t=a)$, observation kernel $p(O_t | S_t=s)$ and a discount factor $\gamma \in [0,1]$. 
In the LLM setting, observations and actions both occupy a textual medium, although the formulation naturally extends to mixed-modality settings. 
Given a pretrained LLM $p_\theta(x_{1:T})$ that maps histories to next-token predictions, we write it as a history-dependent policy $\pi_\theta(A_t \mid H_t)$, where $H_t = (O_0, A_0, \dots, O_{t-1}, A_{t-1}, O_t)$ denotes the interaction history available before taking action $A_t$.
Observations $O_t$ can be viewed as externally provided context tokens, which are masked in the output likelihood as they are not generated by the policy itself.
The policy induces the marginal,
\begin{align}  \label{eq:pomdp_marginal}
    p_{\pi_\theta}(H_{T})
    &= \int dS_{0:T}\, p(S_0)\,p(O_0 \mid S_0)\prod_{t=0}^{T-1} \pi_\theta(A_t \mid H_t)\,p(S_{t+1} \mid S_t, A_t)\,p(O_{t+1} \mid S_{t+1}) .
\end{align}
The goal of reinforcement learning is then to find parameters $\theta^*_{\mathrm{RFT}}$ that maximize expected return, initialized from pretrained weights $\theta^*_{\mathrm{pre}}$, i.e., $\theta^*_{\mathrm{RFT}} = \arg\max_\theta \mathbb{E}_{p_{\pi_\theta}(H_T, S_{0:T})}[\sum_{t=0}^{T} \gamma^t R(S_t, A_t)]$. 
We write the returns starting at any $t$ as $G_t=G(H_{t:T}) \sim \mathrm{Law}(\sum_{i=t}^T R_i \mid H_{t:T})$.
When rewards are immediately available they can also be absorbed in the observations for to enable adaptation.

\paragraph{Control-as-Inference.}
The Control-as-Inference (CAI) framework \citep{levine2018reinforcement, Kappen2012} reframes the optimization of $\theta^*_{\mathrm{RFT}}$ as a probabilistic inference problem. 
The central idea is to tilt a prior trajectory distribution $p_{\pi_\theta}$ by an energy function $\Phi : \mathcal{H} \to \mathbb{R}$ that biases towards the expected return. 
This is commonly written by introducing a binary optimality variable $\mathcal{O}_t \in \{0,1\}$ with conditional likelihood $\Phi_{1:T} \eqdef p(\mathcal{O}_{1:T} = 1 \mid H_T) \propto \exp \bbE_{S_{0:T} | H_T}G(H_T)=\prod_{t=1}^T\!\exp \bbE_{S_{t}\mid H_t}R_t$, so that high return trajectories receive larger posterior mass. 
By Bayes' rule, at any step $t$ this induces a posterior policy, which is optimality-conditioned from $t$ to $T$,
\begin{align} \label{eq:posterior_policy_cai}
    p_{\pi_\theta}(A_t \mid H_t, \mathcal{O}_{t:T}=1)
    =
    \pi_\theta(A_t \mid H_t)\,
    \frac{
        \mathbb{E}_{H_T \sim p_{\pi_\theta}(\cdot \mid H_t, A_t)}
        \!\left[
            \mathcal{Z}(H_T, A_t)^{-1}\exp \bbE_{S_{t:T}|H_T} G_t
        \right]
    }{
        \mathbb{E}_{H_T \sim p_{\pi_\theta}(\cdot \mid H_t)}
        \!\left[
            \mathcal{Z}(H_T)^{-1}\exp \bbE_{S_{t:T}|H_T}G_t
        \right]
    } ,
\end{align}
which reweights the prior policy $\pi_\theta(A_t \mid H_t)$ according to the exponentiated expected return. The partition function $\setZ$ normalizes the potential, and $G$ is a deterministic function over the trajectories.
A well known property of this posterior policy is that it is exactly the optimal policy for a KL-regularized augmentation of the base MDP (see Appendix~\ref{appendix:derivation}).
In particular, it solves the variational optimization,
\begin{align} \label{eq:posterior_policy_is_variational_optimization}
    p_{\pi_\theta}(\cdot \mid \mathcal{O}_{t:T}=1)
    = p_{q^*}(\cdot) = \argmax_{q \in \Pi} \mathbb{E}_{p_q}\!\left[\ln p(\mathcal{O}_{t:T}=1 \mid H_{t:T})\right]
        - \bbE_{p_q} D_{\mathrm{KL}}\!\left(q \,\|\, \pi_\theta \right),
\end{align}
where $\Pi = \{\pi \mid \pi : \mathcal{H} \to \Delta(\mathcal{A})\}$ denotes the set of admissible policies.
This makes the framework especially useful in deep RL, where regularized policy improvement has proven effective in practice.

\paragraph{The Cross-Entropy Method.}

\begin{wrapfigure}{r}{0.48\textwidth}
    \vspace{-0.2em}
    \begin{minipage}{0.48\textwidth}
        \begin{algorithm}[H]
            \caption{Cross-Entropy Method.}
            \label{alg:cem}
            \SetAlgoLined
            \SetAlgoVlined
            \DontPrintSemicolon
            \LinesNotNumbered
            \linespread{1.2}\selectfont
            \KwIn{Proposal $q_\psi$, score $S$, batch size $N$, elite fraction $\rho \in (0,1)$, iterations $T$.}
            \For{$t = 1, \ldots, T$}{
                Sample $\{x_i\}_{i=1}^N \sim q_{\psi_t}$\;
                Score each sample, $s_i \leftarrow S(x_i)$\;
                Set threshold $\mu_t$ to the $\lceil \rho N \rceil$-th best $s_i$\;
                Form elite set $\mathcal{E}_t \leftarrow \{x_i : s_i \geq \mu_t\}$\;
                Project $\psi_{t+1}\!\leftarrow\!\argmax_\psi \sum_{\mathcal{E}_t} \ln q_\psi(x)$\;
            }
            \textbf{Return} $q_{\psi_T}$
        \end{algorithm}
    \end{minipage}
    \vspace{-1.5em}
\end{wrapfigure}

To solve Eq.~\eqref{eq:posterior_policy_is_variational_optimization}, algorithms such as MPO \citep{abdolmaleki2018maximum} sample directly from Eq.~\eqref{eq:posterior_policy_cai} before projecting the resulting solution back onto a parametric model\footnote{This is orthogonal to the constrained optimization of the Lagrangian in MPO, and the fact that MPO iteratively updates the reference policy $\pi_\theta$ as part of an expectation-maximization loop \citep{Neal1998}.}. 
This reflects a broader \emph{bootstrap-and-project} scheme: trajectories are drawn from the proposal $p_{\pi_\theta}(H_T)$, reweighted by an unnormalized potential $\Phi(\cdot)$ encoding the optimization direction, and the resulting $\Phi$-tilted empirical distribution is projected back onto the parametric family of $\pi_\theta$ to define the next prior.
A useful parallel can be drawn here to the cross-entropy method (CEM) for optimization \citep{rubinstein_optimization_1997, rubinstein1999cem, botev2013cemopt}, see Algorithm~\ref{alg:cem}. 
CEM similarly draws a batch $\{x_i\}_{i=1}^N \sim q_{\psi_t}$, scores each sample by $S(x_i)$, and iteratively concentrates $q_{\psi_t}$ on high-scoring $x_i$ in order to approach $\argmax_x S(x)$. 
Concretely, it uses an indicator potential $\Phi(x) = \mathbf{1}_{S(x) \geq \mu_t}$, where $\mu_t$ is the $(1-\rho)$-quantile of $S$ under $q_{\psi_t}$ for $\rho \in (0,1)$. 
In practice, $\mu_t$ is estimated by the score of the $\lceil \rho N \rceil$-th best batch-sample. 
The indicator then retains the top-quantile samples, and the parametric family is refit to this elite set by cross-entropy minimization, typically in closed-form.
We adapt this for deep-RL in Section~\ref{section:ftw}.

\section{Follow the Winners: Model Derivation \& Elite Filtering} \label{section:derivation} \label{section:ftw}

We derive \textit{Follow the Winners} (FTW) from a CAI problem in which optimality is specified through ordinal comparisons rather than cardinal returns. 
The idea is to perform inference not over a single trajectory, but over a finite batch of trajectories, and to then condition this batch on an ordinal observation which specifies trajectory preference.
This yields an evidence lower bound whose optimization naturally suggests a sample-and-project method: tilt the sampling distribution toward trajectories that win local comparisons, then project the resulting distribution onto the policy class.

\paragraph{Model Derivation.} Standard CAI formulates policy optimization as inference over individual trajectories $H_T$, typically by conditioning on binary optimality variables $\setO_t \in \{0,1\}$. 
In agentic RFT, however, scalar trajectory returns can be sparse, noisy, and difficult to calibrate across tasks, while relative comparisons among sampled rollouts are often more reliable. 
We therefore replace inference over individual trajectories with inference over finite collections $\setB_N = (H_T^{1}, \dots, H_T^{N})$.
Under a base sampling policy $\pi$, a collection is distributed according to the product measure $p_\pi(\setB_N) = \prod_{n=1}^{N} p_\pi(H_T^{n})$.
The collection $\setB_N$ provides a local comparison set for defining ordinal feedback.

To encode this feedback, let $\setC(\setB_N) = (i_1,\dots,i_N)$ 
denote the ranking induced over the collection, where $H_T^{i_1} \succ \cdots \succ H_T^{i_N}$. 
The binary optimality variable of standard CAI is then replaced by an ordinal observation variable $\setO_T$, whose support depends on the likelihood model. 
For instance, under a Plackett--Luce model, $\setO_T$ denotes a full permutation $\sigma$ of the trajectories, while for a multinomial model, $\setO_T \in \{1,\dots,N\}$ denotes the identity of a preferred trajectory. 
We write $\Phi_T(\setB_N)$ for the likelihood, or unnormalized potential, assigned to the batch $\setB_N$ under the ordinal observation. 

For simplicity, we focus on the sparse final-reward setting and condition only on $\setO_T$, avoiding partial trajectory feedback through $\setO_t$ for $t \neq T$. 
We further assume that the ranking rule $\mathcal{C}$ is consistent across subsets, satisfying independence of irrelevant alternatives (IIA), which implies a total ordering over the trajectory domain \citep{mcfadden1974conditional, luce1959}. 
This assumption is important for our algorithmic construction: IIA ensures that the per-minibatch rankings used by FTW's elite-filtering step are consistent with an underlying total ordering of conditional returns.

Under our ordinal CAI construction, the marginal is lower-bounded as (compare with Eq.~\eqref{eq:posterior_policy_is_variational_optimization}):
\begin{align}
    \ln p_{\pi_\theta}(\setO_{T})
    = \ln \bbE_{p_{\pi_\theta}(\setB_N)} \Phi_{T}(\setB_N)
    \geq
    \mathbb{E}_{p_q}
    \left[
        \ln \Phi_{T}(\setB_N)
        -
        N \cdot D_{\mathrm{KL}}\!\left(q \,\|\, \pi_\theta\right)
    \right]
    = \setL_\Phi(q, \pi_\theta).
    \label{eq:ell_lowerbound_updated}
\end{align}
The ELBO $\setL_\Phi$ has a direct (regularized) policy-improvement interpretation. 
The potential $\Phi_T$ tilts the proposal $q$ toward the preferred in-batch items, while the KL term constrains the update to the reference $\pi_\theta$. 
Optimizing the bound therefore alternates between two simple operations: filtering better trajectories and projecting the tilted distribution onto the policy class.
This structure is exactly what underpins the cross-entropy method (CEM) for optimization, as shown in Algorithm~\ref{alg:cem}.

\begin{figure}[t]
    \centering
    \begin{minipage}[t]{0.55\textwidth}
        \vspace{0pt}
        \begin{algorithm}[H]
            \caption{Follow the Winners (FTW).}
            \label{alg:ftw}
            \SetAlgoLined
            \SetAlgoVlined
            \DontPrintSemicolon
            \LinesNotNumbered
            \linespread{1.15}\selectfont
            \KwIn{Policies $\pi_\theta, \pi_{\mathrm{ref}}$, buffer size $M$, batch size $C$, winners $K$, batches $J$, KL weight $\lambda$.}
            Initialize FIFO queue $\mathcal{B}_M \leftarrow \emptyset$ s.t. $|\mathcal{B}_M| \le M$\;
            \While{not converged}{
                Collect (async) $H_{0:T} \sim \pi_{\theta_t}$ and append to $\mathcal{B}_M$\;
                \For{$j = 1, \ldots, J$}{
                    Sample $\mathcal{B}_C^{(j)}=\{H_{0:T}^{(i)}\}_{i=1}^C \sim \mathrm{Unif}(\mathcal{B}_M)$\;
                    Score each $H^{(i)}_{0:T} \in \mathcal{B}_C^{(j)}$ by $G(H^{(i)}_{0:T})$\;
                    Elite-set $\mathcal{E}^{(j)} \leftarrow$ top-$K$ trajectories under $G$\;
                }
                $\nabla_\theta \mathcal{L} \leftarrow \mathbb{E}_{\bigcup_j \mathcal{E}^{(j)}}\!\left[\nabla_\theta \ln \pi_\theta\right] - \lambda \nabla_\theta D_{\mathrm{KL}}$ \hspace{\fill} Eq.\eqref{eq:ell_lowerbound_updated}\;
                Project $\theta_{t+1} \leftarrow \theta_t + \alpha \nabla_\theta \mathcal{L}$\;
            }
            \textbf{Return} $\pi_\theta$
        \end{algorithm}
    \end{minipage}\hfill
    \begin{minipage}[t]{0.42\textwidth}
        \vspace{0pt}
        \includegraphics[width=.95\linewidth]{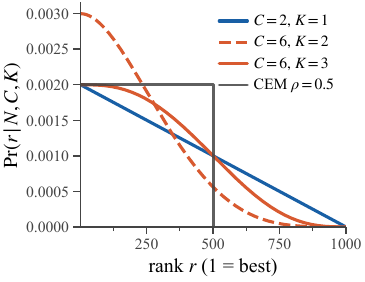}
        \captionof{figure}{
            Normalized marginal occupancy of rank $r$ under FTW's elite-filtering, Eq.~(\ref{eq:probability_of_tourney_win}), for $N=1000$ and varying $C$, $K$. 
            Larger $C$ sharpens concentration on top-ranked trajectories; larger $K$ softens it.
        }
        \label{fig:elite_marginal_main}
    \end{minipage}
\end{figure}

\paragraph{Algorithm Construction.}
FTW instantiates an adaptation of CEM for agentic RFT.
As motivated in Section~\ref{section:intro}, we highlight three desiderata on the choice of proposal and ordinal potential relevant for the RFT setting.
First, updates should be conservative: trajectory returns are high variance, and poor intermediate updates should not destabilize subsequent learning by collapsing onto noisy locally-strong trajectories.
Second, variance reduction should come from the samples alone, for example through batch-level baselines \citep{hu2025reinforcestabilizingcriticfreepolicy}, antithetic sampling, or rank transforms of returns, to avoid both a critic model and the environment duplication required for group rollouts.
Third, the algorithm should reuse past trajectories to amortize data-collection costs and benefit from regularization and stability from replay \citep{mnih_humanlevel_2015, fedus2020replayfundamentals, arnal2026efficientrltrainingllms}.

We show our FTW pseudocode in Algorithm~\ref{alg:ftw} as our adaptation of CEM (Alg.~\ref{alg:cem}).
The proposal distribution $q$ is the mixture of the $h$ past policies, approximated in practice by the empirical occupancy of a FIFO replay buffer, $\hat{q} = \tfrac{1}{\min(h,\, t+1)} \sum_{i=\max(0,\, t-h+1)}^{t} \pi_{\theta_i}$, where $h$ is the maximum age of samples retained in the buffer \citep{cornuet2012amis}.
Then, the closed-form projection step of classical CEM is replaced by a gradient-based projection onto the policy class.
The main departure from standard CEM is how FTW constructs the elite set (filtering).
Rather than sorting the entire replay buffer and retaining a global top $(1{-}\rho)$-quantile, FTW instead samples $J$ independent minibatches of size $C$ from the buffer and retains the top $K = \lceil \rho C \rceil$ per batch.
The resulting union of $J \cdot K$ elites are then used in a cross-entropy gradient update.
Thus, trajectories that are repeatedly preferred over other trajectories obtain larger weight in consecutive proposal distributions.

\paragraph{Filtering Probability.}
Our elite-filtering rule defines a combinatorial tilting distribution $\Phi$ over the replay buffer without explicit evaluation (which would require sorting the buffer).
For a $N$-size buffer $\setB_N$, let $r$ denote the rank of an item under the buffer's theoretical ordering (lower is better).
Then, the probability that an item of rank $r$ is selected as one of the $K$ elites in a minibatch of size $C$,
\begin{align}
    \mathrm{Pr}(r \mid N, C, K)
    =
    \frac{1}{K\binom{N}{C}}
    {\textstyle \sum_{k=1}^{K}}
    \binom{r-1}{k-1}
    \binom{N-r}{C-k}.
    \label{eq:probability_of_tourney_win}
\end{align}
The derivation is given in Appendix~\ref{appendix:proofs:elite_filtering}, Lemma~\ref{lemma:elite_filtering_marginal}.
The induced potential $\Phi$, as shown in Figure~\ref{fig:elite_marginal_main}, is a degree-$(C-1)$ polynomial in $r$ and is monotonically decreasing for all $K \geq 1$ with $K<C$ \citep{goldberg1991comparative, david2003order}.
Increasing $K$ spreads probability mass from the highest-ranked trajectories toward lower-ranked ones, softening selection pressure.
The minibatch size $C$ acts as a concentration parameter: $C=K$ recovers uniform sampling over the buffer, while increasing $C$ for fixed $K$ sharpens concentration on high-ranked trajectories.
Together, the ratio $\rho = K/C$ controls the elite fraction and $C$ controls how sharply that fraction is enforced.
The limiting case $C=N$ with $J=1$ recovers standard CEM filtering exactly, but amortizes data-collection poorly since the $N-K$ non-elite trajectories in the buffer are discarded each update.
For intermediate $K < C < N$, the induced selection profile becomes smoother and more sigmoid-like, as illustrated for $C=6$ and $K=3$ in Figure~\ref{fig:elite_marginal_main}.
The special case $K=1$, $C=2$ reduces to binary tournament selection, connecting FTW to classical selection mechanisms in genetic algorithms \citep{goldberg1989genetic, goldberg1991comparative}.
Thus, our polynomial profile meets two of the aforementioned desiderata, conservative updates and sample-only variance reduction, purely through the $(C, K)$ filtering parameters.

\section{Theoretical Analysis} \label{section:recovering_algorithms}

This section diagnoses the characteristics of our algorithm, specifically its risk profile relative to baselines.
We first recover GRPO and DPO within our extension of the CAI framework (Section~\ref{section:derivation}) as the multinomial and Plackett-Luce potentials over sampled batches.
We then use this recovery to derive the implicit risk profile of each method through the certainty equivalent of its selection rule.
We show that GRPO is risk-neutral and that its otherwise (non-Jensen bounded) exponential multinomial potential is unboundedly risk-seeking.
Then we show that DPO and FTW are boundedly risk-seeking toward a selection threshold that varies by the selection ratio $K/C$.
We assume reader familiarity with GRPO and DPO and provide a self-contained primer in Appendix~\ref{appendix:primer}.

\subsection{Recovering GRPO and DPO} \label{section:recover_grpo_dpo}

We treat the two potentials mentioned in Section~\ref{section:derivation} in turn.
The Plackett--Luce potential at $|\setB_N|=2$ recovers DPO directly, while the multinomial potential admits a Jensen upper bound whose looser surrogate is exactly GRPO's group-mean advantage.
We close the subsection by noting that, within our framework, the remaining difference between the two methods is the choice of baseline normalizer.

We first recover DPO through our formalism simply by plugging in the Plackett--Luce model for $N=2$, under minor technical assumptions.
Note that the authors of DPO \citep{rafailov2023direct} also sketch a Plackett--Luce generalization in their Appendix~A.5, our result is complementary.

\begin{corollary}[DPO from a $N{=}2$ PL]\label{cor:dpo}
    At $|\setB_N|=2$ with ordering $H_T^{(1)} \succ H_T^{(2)}$, the Plackett--Luce potential with weighting $g(H_T) = \exp(s_\theta(H_T) / \beta)$ and log-ratio score $s_\theta^{(i)} \eqdef \log \tfrac{\pi_\theta(H_T^{(i)})}{\pi_\text{ref}(H_T^{(i)})}$ reduces to $p(\setO_T = (1,2) \mid \setB_2) = \sigma(\tfrac{1}{\beta}(s_\theta^{(1)} - s_\theta^{(2)}))$, the Bradley--Terry sigmoid.
    The ELBO in Eq.~\eqref{eq:ell_lowerbound_updated} then recovers DPO's objective, with the prior $\pi_\theta$ in the KL term playing the role of DPO's reference $\pi_\text{ref}$ and temperature $\beta$ matching DPO's up to $N{=}2$ rescaling of our ELBO's KL (derived in App.~\ref{appendix:derivation:recovery}).
\end{corollary}

The multinomial potential at general $N$ requires more work.
The log-partition $\ln \sum_k g_k$ in the ELBO couples all trajectories in the batch, so its policy gradient is dense.
This can be seen by plugging the multinomial potential into Eq.~\eqref{eq:ell_lowerbound_updated} and averaging across all indices $i \in \{1, \ldots, N\}$ of $\setB_N$,
\begin{align} \label{eq:batch_policy_gradient}
    \nabla_q \mathcal{L}_\Phi(q,\pi_\theta)
    \propto \mathbb{E}_{p_q}\Biggl[ \frac{1}{N} \sum_{n=1}^{N} \underbrace{\left(\ln g(n) - b_n\right)}_{\eqdef \alpha_n(\setB_N)} \sum_{m=1}^N \nabla_q \ln p_q(H_T^m) - \nabla_q D_{\mathrm{KL}}\!\left(q \,\|\, \pi_\theta\right) \Biggr],
\end{align}
which produces the double-summation over coefficients $\alpha_n$ and policy gradients.
We sidestep this double-sum through a Rao--Blackwell argument that decouples the gradient into per-trajectory coefficients $\alpha_n$, upon which a Jensen upper bound on the log-partition reveals GRPO's baseline.

\begin{proposition}[GRPO as a Jensen upper bound]\label{prop:grpo}
    Under the multinomial potential with exponential weight $g(H_T) = \exp G(H_T)$, Rao--Blackwellization of Eq.~\eqref{eq:batch_policy_gradient} yields the per-trajectory form
    \begin{align} \label{eq:group_policy_gradient}
        \nabla_q \widehat{\mathcal{L}_\Phi}(q,\pi_\theta)
        \propto
        \mathbb{E}_{p_{q}}\!\left[
            \frac{1}{N} \sum_{n=1}^{N} \alpha_n(\setB_N) \sum_{t=1}^T \nabla \ln q(A_t^n \mid H_t^n)
        \right]
        - \nabla_q D_{\mathrm{KL}}\!\left(q \,\|\, \pi_\theta\right),
    \end{align}
    with coefficients $\alpha_n(\setB_N) = \ln g_n - \ln \sum_k g_k$ admitting the Jensen upper bound $\alpha_n \le G(H_T^n) - \tfrac{1}{N}\sum_k G(H_T^k)$, whose right-hand side is GRPO's group-mean advantage.
\end{proposition}
\begin{proof}[Proof sketch (full proof in Appendix~\ref{appendix:proofs:rao_blackwell}).]
    \emph{Rao--Blackwell.} 
    The multinomial potential applies symmetrically to any item in $\setB_N$, so we can hold one trajectory fixed and marginalize over the rest. 
    This yields per-trajectory coefficients $\alpha_n(\setB_N) = \ln g_n - \ln \sum_k g_k$, recovering an advantage form with log-partition baseline.
    \emph{Jensen bound.} 
    By concavity of $\ln$, $\alpha_n \le \ln g_n - \tfrac{1}{N}\sum_k \ln g_k$. 
    Where the right-hand side reduces to $G(H_T^n) - \tfrac{1}{N}\sum_k G(H_T^k)$, which is GRPO's group-mean advantage.
\end{proof}

The Jensen bound comes at a cost: maximizing an upper bound does not guarantee maximization of the underlying objective.
The gap is precisely the risk offset that the multinomial potential induces under uncertain returns, which we profile in the next subsection.
More broadly, when viewed from this framework, the key difference between the multinomial and Plackett--Luce potentials, from which GRPO and DPO are positioned, is the choice of baseline normalizer.
The multinomial potential applies a uniform group-mean across all items, while the Plackett--Luce model induces rank-dependent baselines that dampen updates at the extremes of the batch.

\subsection{Implicit Risk Profiles of RFT Algorithms} \label{section:risk_profiles}

The CAI posterior in Eq.~\eqref{eq:posterior_policy_cai} marginalizes the return $G$ over the return law.
In deterministic, perfectly observable settings this collapses to $\mathbb{E}_{S_{t:T} \mid H_T} \exp G(H_{t:T}) = \exp G(H_{t:T})$.
Under aleatoric noise or epistemic uncertainty from partial observability, the marginalization on returns induces an implicit risk profile in the algorithm.
In critic-based methods such as SAC, we learn a $Q$-function surrogate to move the expectation inside $\exp \mathbb{E}_{S_{t:T} \mid H_T}G_{t:T}$, in critic-free RFT we have no such fallback.

To make this risk-profile precise, we reduce each potential $\Phi$ to a per-entry selection weight $w_\Phi(H_T) = \mathbb{E}_{G \mid H_T}[\varphi(G)]$ with an increasing utility $\varphi$ through Rao-Blackwellization.
We then define the risk offset $\mathcal{R}_\varphi(H_T) = J_\varphi(H_T) - \mu(H_T)$, where $J_\varphi = \varphi^{-1}(w_\Phi)$ is the certainty equivalent \citep{pratt1964risk} of the entry's return (see the full setup in Appendix~\ref{appendix:proofs:risk}).
The offset illustrates the deviation from the conditional mean return $\mu=\mathbb{E}[G | H_T]$ as a function of its spread $\sigma^2=\mathrm{Var}[G | H_T]$.

\begin{proposition}[Risk profiles of RFT potentials under Gaussian noise]\label{prop:risk}
    For Gaussian returns $G \mid H_T \sim \mathcal{N}(\mu, \sigma^2)$ and an assumed Gaussian buffer marginal $F_0 = \mathcal{N}(\mu_0, \sigma_0^2)$, the risk offset $\mathcal{R}_\varphi$ of each potential is as follows (proof and other return families in Appendix~\ref{appendix:proofs:risk}):
    \begin{enumerate}
        \item[(a)] \emph{Multinomial potential \& GRPO (CAI on cardinal returns):}
        $\mathcal{R} = \tfrac12 \sigma^2$ exactly.
        The Jensen step that positions GRPO (Prop.~\ref{prop:grpo}) replaces $\ln \sum_k \exp G_k$ with $\tfrac{1}{N}\sum_k G_k$, which is an affine utility, so $\mathcal{R} = 0$ for every return law. GRPO is risk-neutral.
        \item[(b)] \emph{Bradley--Terry potential (DPO or Plackett--Luce at $N=2$):}
        $\mathcal{R} = (\mu - \mu_0)\big(\tfrac{\sigma_0}{\sqrt{\sigma_0^2 + \sigma^2}} - 1\big)$ exactly. 
        The offset is bounded by $|\mu - \mu_0|$ and $J_\varphi \rightarrow \mu_0$ as $\sigma \to \infty$.
        \item[(c)] \emph{Rank-based potential (FTW, top-$K$ of $C$):}
        $\mathcal{R} \approx (\mu - t)\big(\tfrac{\sigma_0 s_{C,K}}{\sqrt{\sigma^2 + \sigma_0^2 s_{C,K}^2}} - 1\big)$, using a normal approximation of the $K$-th order statistic.
        The same as (b) with $\mu_0$ replaced by the threshold $t = \mu_0 + \sigma_0 a_{C,K}$, where $a_{C,K}$ and $s_{C,K}$ are the standardized mean and standard deviation of that order statistic.
        Aggressive filters $K \ll C$ have $a_{C,K} > 0$ and so $t > \mu_0$, mild filters $K \ge C/2$ have $a_{C,K} \le 0$ and $t \le \mu_0$, finally $(C, K) = (2, 1)$ recovers (b) exactly.
    \end{enumerate}
\end{proposition}

For the ordinal potentials (b)~and~(c), the offset is paid in how often noise swaps preference orderings within the batch rather than in the magnitude of the returns, which is also what makes it bounded.
This can be desirable when $\sigma$ is epistemic, since for returns with $\mu$ below the attractor ($\mu_0$ or $t$) the offset is positive and acts as an exploration term to encourage knowledge-seeking. %
However, when $\sigma$ is aleatoric, the same offset biases the fixed point away from the mean-optimal policy, treating reward noise as if it were information \citep{vries2026varibased}.
GRPO's risk-neutrality avoids both effects at the cost of neutrality.
In FTW, the elite ratio $K/C$ determines the attractor of FTW.
Mild filters $K \ge C/2$ place it at or below the buffer mean $\mu_0$, it equals DPO at $(C,K)=(2,1)$, and aggressive filters $K \ll C$ place it above $\mu_0$.
For both FTW and DPO, the offset is bounded by $|\mu - t|$.

A separate consequence is the cardinal treatment of returns, which can be sensitive to the tail of the return distribution.
The offset in (a) requires the moment-generating function of the return to be finite, which fails for heavy-tailed returns (Pareto-$\alpha \le 1$, Student-$t$ with $\nu \le 2$).
Although the Jensen bound that positions GRPO removes the offset itself, the same heavy tails destabilize GRPO's cardinal-return baseline numerically.
Rank-based potentials don't suffer from this as they only depend on within-batch order statistics, which remain bounded for any continuous distribution.

\section{Validation on Bandit \& Deep Reinforcement Learning}  \label{section:experiments_mdp}

We test the predictions of Section~\ref{section:recovering_algorithms} on a controlled testbed before scaling up to LLM post-training.
The predictions are that GRPO is risk-neutral but numerically unstable under heavy-tailed returns and that DPO and FTW are boundedly risk-seeking toward an elite threshold set by the selection ratio $K/C$, coinciding at $(C, K) = (2, 1)$.
Furthermore, our FTW method is a general method that can be applied in the deep-RL setting, however its ordinal and conservative nature make it both slower and more stable in terms of the learning curves (compared to a baseline SAC).

\paragraph{Bandit settings.}

\begin{figure}[h]
    \centering
    \includegraphics[width=\linewidth]{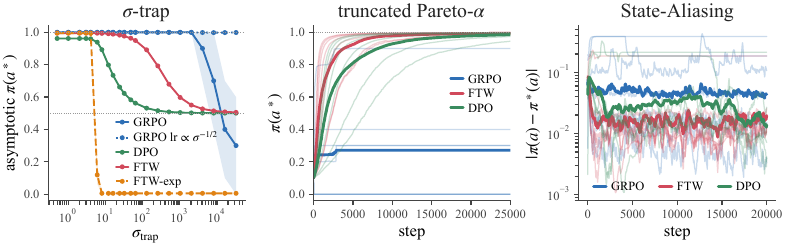}
    \caption{
        Diagnostic bandits isolating algorithm properties (Sec.~\ref{section:recovering_algorithms}). 
        \textbf{Left ($\sigma$-trap):} asymptotic $\pi(a^*)$ vs.\ distractor variance.
        GRPO is only stable for large $\sigma$ when proportionally decreasing the learning rate.
        \textbf{Center (Pareto-$\alpha$):} under heavy tail returns, GRPO learning destabilizes whereas rank-based FTW/DPO remain stable.
        \textbf{Right (state-aliasing):} FTW achieves lower L1 distance $|\pi(a) - \pi^*(a)|$ between the learned and optimal action distributions in the POMDP.
        Left all lines show mean and 95\% bootstrapped CIs over replay buffer sizes $N\in \{16, 64, 256, 1024, 4096\}$ and 10 seeds. 
        Center and right: bold lines are averages over sampled learning-rates, translucent lines show sample trajectories.
    }
    \label{fig:tabular_auc}
\end{figure}

We construct three diagnostic tasks.
The \textit{$\sigma$-trap} is a multi-armed bandit in which a zero-mean distractor arm has much higher return variance than the optimal arm with mean $10$, $\sigma_{\mathrm{trap}} \gg \sigma^\ast$.
A risk-neutral learner commits to the optimal arm, while a sufficiently risk-seeking learner gets distracted by the noise of the suboptimal arm as $\sigma$ grows.
Secondly, the \textit{Pareto-$\alpha$} bandit draws rewards from a truncated Pareto-$\alpha$ distribution and sweeps the tail index $\alpha$ to test whether each method's fixed point remains well-defined as the moment-generating function diverges at $\alpha \leq 1$.
Finally, the \textit{state-aliasing} task collapses two distinct latent states onto the same observation with a deterministic reward, so maintaining policy stochasticity is needed for low regret.
This isolates the concentration rate from the risk profile.
We describe the bandit setup in more detail in Appendix~\ref{appendix:experimental_setup:environments}.

The results in Figure~\ref{fig:tabular_auc} agree with the theory on all three items.
In the $\sigma$-trap (left), GRPO converges on the optimal arm so long as its update rule remains numerically stable given the increasing gradient magnitude, which is only achieved when we decrease the learning rate as $\sigma^{-1/2}$.
We also see that DPO is distracted earlier by noise than FTW, but both degrade to the same uniform policy eventually.
This is the coin-flip limit which Prop.~\ref{prop:risk} predicts for bounded operators once the distractor's draws land above and below the optimal mean with equal probability (Appendix~\ref{appendix:proofs:risk}).
Although FTW at $(2,1)$ shares DPO's risk offset, DPO's earlier collapse is due to its more aggressive exponential concentration in the update rule, a dynamical entrenchment effect of DPO's loss which our fixed-point theory does not capture.
As expected, the exponential potential over the buffer (FTW-exp) collapses onto the trap arm the fastest, see also Appendix~\ref{appendix:results:sigma_toggle} for more variations of the elite-ratio parameters $(C, K)$.
On the Pareto-$\alpha$ bandit (center), GRPO fails to identify the optimal arm in the heavy-tailed regime, while DPO and FTW remain stable.
This matches the prediction that the rank-based potential has a well-defined fixed point for any continuous return distribution whereas the cardinal-exponential potential does not.
On the aliasing task (right), FTW achieves lower L1 distance $|\pi(a) - \pi^*(a)|$ between the learned and optimal action distributions compared to GRPO and DPO, when averaged over a learning-rate population.
This is consistent with polynomial concentration retaining more stochasticity under partial observability.
See Appendix~\ref{appendix:results} for additional Bandit results and Appendix~\ref{appendix:belief} for a knowledge-seeking experiment left out of the main paper.

\begin{wrapfigure}{r}{0.47\linewidth}
    \centering
    \vspace{-1.5em}
    \includegraphics[width=.97\linewidth]{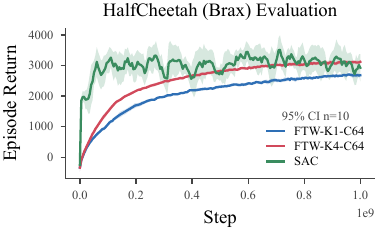}
    \vspace{-1em}
    \caption{
        Evaluation learning curves on HalfCheetah after 1B steps of training.
        Shaded regions are 95\% Gaussian CIs over 10 seeds.
    }
    \label{fig:deep_rl}
    \vspace{-2em}
\end{wrapfigure}
\paragraph{Deep RL.}
We also evaluate on Brax HalfCheetah \citep{freeman2021brax} to check that FTW functions as a more general deep-RL algorithm.
Figure~\ref{fig:deep_rl} compares FTW-K1-C64 and FTW-K4-C64 against SAC \citep{haarnoja_sac_2018}.
FTW is less sample-efficient than SAC in the early phase of training, as expected from the ordinal treatment of returns that discards return magnitude.
By the end of the $10^9$-step budget, FTW-K4-C64 reaches mean return comparable to SAC with visibly tighter seed-to-seed variation.
Thus, our conservative ordinal update also extends to deep-RL settings without a critic, which is the regime of LLM post-training.

\section{LLM Post-Training Experiments}  \label{section:experiments_posttraining}

We evaluate Follow the Winners (FTW) on two agentic LLM post-training benchmarks: Sokoban \citep{SchraderSokoban2018}, a visual box-pushing puzzle that requires multi-turn planning, and Search-R1 \citep{jin2025searchr1}, an open-domain question-answering task with a retrieval tool.
We interleave chain-of-thought with tool or action outputs and use sparse terminal rewards together with a $-0.1$ penalty for formatting errors.
We reuse the experimental setup of \citet{feng2025groupingroup}, see Appendix~\ref{appendix:experimental_setup:environments}.

We ablate the FTW filtering operator for $C=5$ and $K \in \{1, 2, 3\}$ against GRPO, PPO, and REINFORCE++ on Sokoban using the Qwen2.5 model-line \citep{qwen2025qwen25technicalreport} with a replay buffer of size 2000.
On Search-R1 we use FTW with $C=4$ and $K=1$ with a 10k replay buffer.
GRPO uses $n=8$ rollouts per prompt and is the only method that needs duplicated group rollouts.

Figure~\ref{fig:llm_post_training} (left) shows test success rate curves on Sokoban.
For $K\in\{1, 2\}$ FTW clusters with GRPO and PPO in learning speed and performance, showing that FTW performs competitively despite a much simpler update, without PPO's value function model or GRPO's duplicated rollouts.
Notably, for $K=3$ FTW collapses after peaking around $0.4$.
Proposition~\ref{prop:risk} shows that this filter is relatively mild, its elite threshold for selection sits exactly at the buffer mean.
Under continuous returns this means that there should be monotonic improvement (cf.\ the $K{=}3$, $C{=}4$ bandit in Appendix~\ref{appendix:results:sigma_toggle}).
Under Sokoban's binary returns, however, a minibatch with fewer than $K$ successes fills its remaining elite slots with zero-return trajectories that the cross-entropy loss then imitates, whereas at $K=1$ a single reward per minibatch would suffice.
As the success rate falls this becomes more likely and the collapse reinforces itself.
We attribute this to a finite-sample effect that larger batches and replay buffers dampen at the cost of slower learning (detailed in Appendix~\ref{appendix:results:training}).
Finally, we also see that REINFORCE++, exhibits large seed variance and slower learning, implying that dropping group rollouts or a value function without some variance reduction replacement is insufficient on Sokoban.

The table in Figure~\ref{fig:llm_post_training} (right) reports average exact-match accuracy on the 50k+ Search-R1 test set from \citet{jin2025searchr1}.
We see that our FTW variant matches (within noise) the published Search-R1 PPO and GRPO baselines reported by \citet{jin2025searchr1}.
This again supports our claim that FTW is a viable alternative to PPO, which needs more than double the GPU resources, and to GRPO, which needs group duplication (Appendix~\ref{appendix:experimental_setup:compute}).
However, we do caution that the bandit experiments isolate the theoretical mechanisms more cleanly than what the LLM post-training setting can provide.

\begin{figure}[t]
    \begin{minipage}{0.5\textwidth}
      \centering
      \includegraphics[width=.9\textwidth]{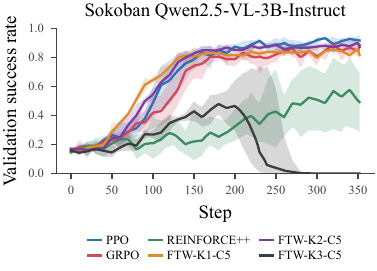}
    \end{minipage}
    \hspace{0.03\textwidth}
    \begin{minipage}{0.35\textwidth}
        \vspace{-2.5em}
        \centering
        \begin{tabular}{l|c}
            \multicolumn{2}{c}{\textbf{Qwen2.5-3B-Instruct}} \\
            \addlinespace[0.4em]
            \textbf{Method} & \textbf{Avg. Test Acc. \%} \\
            \midrule
            Search-R1 (PPO)  & 32.5 \\
            Search-R1 (GRPO) & 33.6 \\
            \midrule
            FTW-K1-C4        & 35.46 $\pm$ 0.15 SD \\
        \end{tabular}
    \end{minipage}
    \caption{
        Sokoban test success rate curves (left, 128 puzzles, mean $\pm$ one standard deviation over 5 seeds) and final exact-match accuracy $\pm$ one standard deviation on the Search-R1 test set over 3 seeds (right, 50k+ examples).
        PPO and GRPO results are from \citet{jin2025searchr1}.
    } \label{fig:llm_post_training}
\end{figure}

\section{Related Work}  \label{section:related_work}

The CAI framework on which we build has a long history in stochastic optimal control \citep{Kappen2012, todorov2009efficient, rawlik2012stochastic, theodorou_generalized_2010} and reinforcement learning \citep{levine2018reinforcement, ziebart_modeling_2010, dayan1997em, toussaint_robot_2009}, where it has motivated methods such as MPO \citep{abdolmaleki2018maximum}, REPS \citep{peters2010reps}, and SAC \citep{haarnoja_sac_2018}.
We extend this line by recasting the tilting potential as conditioning on permutations or winner indices over sampled batches, rather than on a binary variable for individual trajectories.

Beyond GRPO and DPO, our framework subsumes a broader family of bootstrap-and-project methods.
RAFT \citep{dong2023raft} corresponds to FTW with $C=N$, $K=1$, and a buffer refreshed every prompt. 
ReST \citep{gulcehre2023reinforcedselftrainingrestlanguage} and ReST-EM \citep{singh2024beyond} bootstrap toward an annealed threshold potential before projecting back onto the model parameters.
Variational best-of-$n$ \citep{amini2025variational} shares the single-batch top-1 selection structure ($C=N$, $K=1$, no replay), though its KL-regularized variational objective differs from our cross-entropy projection, and STaR \citep{zelikman2022star} is closely related as a hard top-1 selection criterion.
Listwise preference methods such as LiPO \citep{liu2025lipo}, PRO \citep{song2024pro}, and RSO \citep{liu2024statistical} use the PL likelihood directly as a training objective on the policy.
A complementary approach to damping concentration is generalized-entropy regularization through Tsallis entropy \citep{Lee2020tsallisRL} or general greedification operators \citep{chan_greedification_2022}, which for the right choice of greedification can also achieves polynomial concentration through the regularizer rather than directly through the potential.
Recent work by \citet{wu2026takestwogrposecretly} also reach a related GRPO/DPO correspondence in a bottom-up way by reframing GRPO's group-relative advantage as a contrastive loss. 
Our derivation arrives at the same $N=2$ Bradley--Terry limit top-down as the first-stage marginal of the PL model under CAI.
We provide a more comprehensive discussion of the literature and its positioning to our work in Appendix~\ref{appendix:related_work}.

\section{Conclusion}  \label{section:conclusion}

We introduced Follow the Winners (FTW), a critic-free policy-learning algorithm for agentic LLM post-training that is adapted from the cross-entropy method for optimization.
FTW induces polynomial concentration to limit aggressive policy updates which can lead to entrenchment, and uses a replay-buffer proposal to reduce sample variance without the need for a critic model or group rollouts.
On agentic LLM benchmarks, FTW matches GRPO and PPO on Sokoban and the published Search-R1 baselines, and trains more stably and faster than REINFORCE++ on Sokoban.
This supports our algorithm motivation that FTW offers a viable trade-off, spending CPU memory in place of a value model or the environment duplication that group rollouts require.
To derive the algorithm, we extended the Control-as-Inference framework to encompass contemporary critic-free RFT methods by reinterpreting the tilting potential as conditioning on permutations or winner indices over sampled batches.
This recovers DPO exactly under the Plackett--Luce potential and positions GRPO as the Jensen upper bound on the corresponding multinomial objective, exposing GRPO as risk-neutral while DPO and FTW share a bounded risk offset that FTW controls through its selection pressure (Proposition~\ref{prop:risk}).
The framework places GRPO, DPO, and FTW on a common design axis, and provides a top-down basis for principled algorithm design in the critic-free RFT regime.

\paragraph{Limitations.}
FTW's risk offset acts as an exploration term under epistemic uncertainty, but the same offset also biases the fixed point away from the mean-optimal policy under aleatoric noise (Section~\ref{section:risk_profiles}), making FTW (and DPO) less suited to environments with large irreducible noise.
The replay-buffer-based proposal amortizes data-collection costs and provides additional regularization, but at the price of memory overhead.
That said, CPU memory is often cheaper and more accessible than stricter GPU requirements.
Finally, our IIA assumption (Section~\ref{section:ftw}) underlying the elite-filtering interpretation requires a total ordering on the trajectory domain, and thus needs further adaptation for non-transitive settings such as optimizing for a Nash equilibrium.

\section*{References}
\bibliographystyle{plainnat_custom}
\bibliography{references}

\appendix
\clearpage

\newpage
\section*{Appendix: Table of Contents}
\appendixtoc

\let\oldsection\section
\let\oldsubsection\subsection
\makeatletter
\renewcommand{\section}{\@ifstar{\oldsection*}{\@appsection}}
\newcommand{\@appsection}[1]{%
  \oldsection{#1}%
  \hypertarget{appsec.\thesection}{}%
  \addtocontents{apx}{%
    \protect\contentsline{section}%
      {\protect\numberline{\thesection}#1}%
      {\thepage}%
      {appsec.\thesection}%
  }%
}
\renewcommand{\subsection}{\@ifstar{\oldsubsection*}{\@appsubsection}}
\newcommand{\@appsubsection}[1]{%
  \oldsubsection{#1}%
  \hypertarget{appsec.\thesubsection}{}%
  \addtocontents{apx}{%
    \protect\contentsline{subsection}%
      {\protect\numberline{\thesubsection}#1}%
      {\thepage}%
      {appsec.\thesubsection}%
  }%
}
\makeatother

\clearpage

\section{Disclosure Statements}  \label{appendix:statements}

\paragraph{LLM Usage.}  \label{section:llm_usage}
LLMs were used for writing assistance (language and notation), literature search, and experiment infrastructure via agentic coding tools (Claude Code \& Codex).
All ideas, derivations, and results are solely the authors'.

\paragraph{Broader Impact.}  \label{section:broader_impact}
Our work contributes a general-purpose policy-learning algorithm for critic-free reinforcement fine-tuning of large language models.
The contribution is methodological and is not tied to any specific deployed application, dataset, or released model.

As with most advances in LLM post-training, improvements in agentic capability can in principle support beneficial applications (such as research assistants, accessibility tooling, and scientific discovery) as well as harmful ones (such as automated disinformation, scam generation, or surveillance).
The conservative, polynomially-concentrating update rule introduced by FTW does not differentially advantage either direction, and we clearly document limitations (namely the risk-seeking profile under aleatoric noise) that can cause unintended effects in practice.
Our experiments use publicly available pre-trained models from the Qwen2.5 family and benchmarks intended for research, and we do not release post-trained model weights.

Relative to GRPO-style baselines, FTW incurs additional memory overhead from the replay buffer.
This is partly offset by amortized data-collection costs, so the net compute footprint depends on the specific deployment regime; we compare the compute and memory footprints of FTW, GRPO, and PPO in Appendix~\ref{appendix:experimental_setup:compute}.
Running, debugging, and tuning large-scale RL experiments is computationally costly and carries a non-trivial environmental footprint, which we acknowledge as a known concern of this line of research.

\paragraph{Reproducibility.}  \label{section:reproducibility}
We provide full pseudocode for FTW in Algorithm~\ref{alg:ftw} as the method is intentionally simple.
Appendix~\ref{appendix:experimental_setup} documents the full experimental setup, covering hardware (Appendix~\ref{appendix:experimental_setup:hardware}), software (Appendix~\ref{appendix:experimental_setup:software}), exact environment and benchmark specifications (Appendix~\ref{appendix:experimental_setup:environments}), and all hyperparameters (Appendix~\ref{appendix:experimental_setup:hyperparameters}).
All experiments build on publicly available components: the Qwen2.5 model family, known environments like Sokoban, Search-R1, and Brax.
Our code is a fork of \citet{feng2025groupingroup}.
Additional results supporting the main-text claims are reported in Appendix~\ref{appendix:results}.
We do not release code, models, or data until the manuscript is accepted and internal legal procedures are completed.

\section{Extended Related Work} \label{appendix:related_work}

We situate our method relative to four strands of work: control as inference; RLHF and preference optimization; policy improvement in the critic-free setting; and CEM methods, together with a brief discussion of order statistics and their implications for our hyperparameter choices.

\subsection{Control as Inference.}

The Control-as-Inference framework \citep{levine2018reinforcement, ziebart_modeling_2010, ziebart_maxent_2010, Kappen2012, toussaint_robot_2009, attias2003planning, dayan1997em, rawlik2012stochastic, todorov2007linearly, todorov2009efficient, theodorou_generalized_2010} recasts sequential decision-making as probabilistic inference in a graphical model.
Under specific modelling assumptions, the posterior in this formulation coincides with the optimal or soft-optimal control law, making tools from probabilistic inference and stochastic optimal control directly applicable to the control problem.
This dual formulation enables an interchange of methods between the two formalisms: examples include sequential Monte Carlo and particle filtering for planning \citep{piche2018probabilistic, macfarlane2024spo, vries2025trustregion}, expectation-maximization for policy search as used by REPS and MPO \citep{dayan1997em, abdolmaleki2018maximum, peters2010reps}, and amortized variational inference for control \citep{zintgraf_varibad_2020, vries2026varibased}.
A further merit is that the probabilistic interpretation extends naturally to the deep reinforcement learning setting with high data throughput, since the free-energy formulation yields natural regularization for approximate policy improvement \citep{haarnoja_sac_2018, chan_greedification_2022, geist_theory_2019}.

The formalism usually builds on the idea of exponentiated reward potentials that tilt a prior trajectory distribution so that the resulting posterior concentrates on high-return trajectories.
Under specific choices of potential and dynamics, this recovers maximum-entropy reinforcement learning \citep{ziebart_maxent_2010} and, in the continuous-time setting with KL or quadratic control costs, a linearized form of the Hamilton--Jacobi--Bellman equation \citep{Kappen2012, todorov2009efficient, theodorou_generalized_2010}.
The framework is usually applied in the Markov decision process setting \citep{levine2018reinforcement}; we instead develop a formulation for the partially observable case, for which related treatments can be found in \citet{abdulsamad2025sequentialmontecarlopolicy, vries2026varibased}.
Our contribution differs from the standard single-trajectory exponential tilting view in two ways: we allow ordinal likelihoods over finite sets of sampled trajectories, and we replace exponential return-based tilting with a polynomial, order-statistic-based tilting mechanism.
The combination with a tailored cross-entropy method for optimization yields the conservative, approximate, policy-iteration method from the main text.

\subsection{Reinforcement Learning from Human Feedback (RLHF).}

RLHF \citep{christiano2017deep, stiennon2020learning, ouyang2022training} has driven much of the recent LLM post-training literature.
The shift toward critic-free methods such as GRPO and its variants \citep{shao2024deepseekmath, yu2025dapoopensourcellmreinforcement, liu2025understanding, liu2026tricks}, REINFORCE++ \citep{hu2025reinforcestabilizingcriticfreepolicy}, and RLOO \citep{kool2019buy, ahmadian2024backtobasics}, and toward direct preference optimization without an explicit reward model such as DPO \citep{rafailov2023direct} and listwise extensions \citep{liu2025lipo, song2024pro}, has produced a fragmented landscape that our derivation seeks to unify.
Methods specifically targeting multi-turn agentic settings such as StarPO \citep{wang2025ragenunderstandingselfevolutionllm}, Multi-Turn GRPO \citep{wang2025a, wei2025reinforcing}, or GiGPO \citep{feng2025groupingroup} share our operational regime but treat the multi-turn structure through specialized advantage estimators or credit-assignment heuristics, rather than through a top-down probabilistic framework.

Within this landscape, our framework places GRPO-family and DPO-family methods on a common axis (subset size $N$ and choice of likelihood) and identifies common heuristics such as group baselines, shared prefixes, and contrastive losses as consequences of specific top-down modelling choices rather than fundamental algorithmic differences.
Although we discussed RAFT \citep{dong2023raft} and similar methods in the main text, the off-policy extension was introduced by RAFT++ \citep{xiong2025minimalistapproachllmreasoning} which is parallel to our use of the replay buffer.
A closely related thread is STaR \citep{zelikman2022star}, which bootstraps a reasoning policy by iteratively fine-tuning on self-generated rationales that lead to correct answers, and can be viewed as a degenerate form of bootstrap-and-project with a hard top-1 selection criterion.
Recent work by \citet{wu2026takestwogrposecretly} reaches a related correspondence bottom-up by reframing GRPO's group-relative advantage as a contrastive loss, and shows a gradient-level equivalence in the $N=2$ case.
Our derivation arrives at the same $N{=}2$ Bradley-Terry limit top-down as the first-stage marginal of the PL model under CAI, and takes a broader perspective through the lens of tilting potentials (i.e., the Multinomial or Plackett-Luce).

\subsection{Learning to Rank and Listwise Preference Optimization.}

The Plackett-Luce (PL) family is one of several ordinal potentials our framework admits, and the one through which we recover DPO \citep{rafailov2023direct} as the $N{=}2$ special case.
This places our work in contact with a long line of research on PL-based ranking outside of LLM alignment.
Classical learning-to-rank \citep{cao2007listwise, xia2008listwise, oosterhuis2022plrank} fits a parameterized scoring function to observed rankings via maximum likelihood under PL or related listwise losses.
Our formulation is closest to PL-Rank \citep{oosterhuis2022plrank}, but differs in two ways.
First, we do not parameterize the PL model on an external scoring function.
Instead, our PL likelihood arises from the Control-as-Inference construction \citep{levine2018reinforcement} over sampled trajectory collections, and the policy is updated through the resulting posterior rather than fit to PL labels directly.
Second, the CAI derivation supplies a KL-regularization term against a reference policy that is absent in standard listwise learning-to-rank.

Preference-based policy learning \citep{akrour2011preference}, and its deep RL version \citep{christiano2017deep}, form another branch in reinforcement learning that learns a preference model on trajectory data as a surrogate for the reward.
This surrogate is then optimized akin to a value function instead of the direct returns.
DPO \citep{rafailov2023direct} collapses this two-stage procedure into a single objective under a Bradley-Terry model, and sketches a PL generalization in their Appendix A.3.
Our derivation can be read as making their sketch more explicit.
The PL likelihood is recovered top-down from CAI rather than imposed as a ranking loss, which yields the ELBO and its policy gradient for optimization as consequences rather than design choices.

Listwise preference optimization methods such as LiPO \citep{liu2025lipo} and PRO \citep{song2024pro} use the same PL likelihood as a training objective directly on the LLM's response distribution.
Rather than fitting an external scoring function as in classical LTR, they treat the policy itself as the PL ranker over candidate responses, perhaps similar to standard learning-to-rank.
The two methods differ in their loss construction.
LiPO uses a rank-aware listwise loss inspired by LambdaRank \citep{burges2006learning}, which weights each pair by its contribution to a target ranking metric.
PRO uses sequential elimination with a Bradley-Terry-style contrast at each rank position, which is one approach to sample from the standard PL factorization.
RSO \citep{liu2024statistical} is closely related and uses Boltzmann acceptance to source preference data closer to the optimal policy, which connects to our (ordinal) exponential-tilting view in Section~\ref{section:background} but without the ELBO grounding.

\subsection{Policy Improvement Operators for Critic-Free RL.}

The classical policy improvement theorem \citep{howard1960dynamic, sutton_reinforcement_2018} states that replacing actions by those with higher Q-value under the current policy yields a non-degrading update.
This guarantee is the foundation of approximate policy iteration, and it depends on having access to a $Q$-estimate against which to greedify.
In conventional deep reinforcement learning that estimate comes from a parametric value function trained alongside the policy, and a substantial body of work studies how to effectively estimate $Q$ \citep{mnih_humanlevel_2015, vanhasselt2016deep, munos_retrace_2016} and how aggressively to move the policy toward its argmax \citep{kakade2002approximately}.
For instance, \citet{chan_greedification_2022} formalize greedification through forward and reverse KL projections, showing that the choice of divergence materially affects stability.
\citet{oren2025value} extend the view to actor-critic methods and show that more conservative greedification operators improve performance in the deep setting.
Trust-region methods such as TRPO and PPO \citep{schulman2017trust, schulman2017proximal}, and relative-entropy policy search \citep{peters2010reps}, take a complementary route by constraining the magnitude of the per-step policy update against a reference policy.
On the theoretical side, \citet{mei2020escaping, mei2020global, beirami2025theoretical} show that softmax policy gradient and best-of-n methods concentrate exponentially fast on the argmax-action for the current value estimate, which can be desirable for swift convergence when the estimate is well-calibrated.

Critic-free RFT for LLMs breaks the assumption underlying this toolbox.
Methods such as GRPO \citep{shao2024deepseekmath}, REINFORCE++ \citep{hu2025reinforcestabilizingcriticfreepolicy}, and RAFT \citep{dong2023raft} discard the value function for compute reasons and replace it with batch-relative statistics: the group-mean baseline in GRPO, the global advantage normalization in REINFORCE++, the top-$1$ winner in RAFT.
The estimate against which the policy is greedified is no longer a stable function approximator but a Monte Carlo statistic over $N$ samples per prompt.
For small $N$ this estimate is high-variance, and the exponential concentration of softmax policy gradients amplifies that variance into the policy update.
This can cause \textit{entrenchment}, where one trajectory with large positive empirical advantage receives an outsized push, and the policy commits to it before subsequent samples can correct the estimate.
KL-regularization slows this commitment but does not eliminate it \citep{geist_theory_2019}, since the asymptotic concentration profile of the soft-optimal policy is unchanged by the size of the KL coefficient.
The view that conservative updates protect against estimator noise is well-established for value-based methods, where double Q-learning \citep{vanhasselt2016deep} and clipped twin Q-targets \citep{fujimoto2018addressing} address overestimation bias from value-function noise.
Our argument transposes this concern to the policy side and to the critic-free setting.

Our work proposes a different lever for controlling concentration in this critic-free regime through the shape of the tilting potential.
\textit{Follow the Winners} induces a polynomial concentration in the order statistic of returns rather than the exponential concentration of softmax-based methods (see Section~\ref{section:ftw}).
In essence, our polynomial potential approach dampens the concentration target, and the KL regularizer bounds how quickly the policy moves toward it.
A complementary mechanism for damping concentration is generalized-entropy regularization, where Tsallis entropy \citep{Lee2020tsallisRL} replaces the standard KL or Shannon term and yields polynomial-tailed policies in the limit.
This achieves polynomial concentration through the regularizer rather than the potential, and the two approaches are formally distinct: Tsallis-regularized policies are polynomial in the advantage relative to a reference, whereas our potential is polynomial in the order statistic of within-batch returns.
The two could in principle be combined, although we leave that to future work.

One asymptotic benefit of the polynomial potential over the exponential is reduced sensitivity to noisy return estimations.
Firstly, the update depends on within-batch ranking rather than on the magnitude of the empirical advantage.
Secondly, the implicit risk preference induced by exponential tilting grows quadratically with the noise scale of the empirical mean, whereas under a rank-based potential it grows only linearly and saturates at the upper quantile of the noise distribution.
We sketch this argument informally in Appendix~\ref{appendix:proofs}.
This effect is less crucial in critic-augmented settings, where value estimation can absorb part of the noise, but it is directly relevant in critic-free RFT, where the only fallback is the batch average itself.

\subsection{Cross-Entropy Method for Deep RL.}
The cross-entropy method \citep{rubinstein_optimization_1997, rubinstein1999cem, botev2013cemopt} has been applied to reinforcement learning along two main lines, neither of which matches our usage.
The first treats CEM as a black-box policy-search method, where the elite-filtering step operates over policy parameters directly.
\citet{mannor2003cem} introduced this for fast policy search, and \citet{pourchot2019cemrl} extend it by combining CEM with gradient-based actor-critic updates to obtain a population-based hybrid.
The second uses CEM as a short-horizon planner over learned dynamics models, where the elite-filtering step operates over candidate action sequences rather than policy parameters.
PETS \citep{chua2018pets}, PlaNet \citep{hafner2019learning}, the differentiable CEM of \citet{amos2020dcem}, and the sample-efficient iCEM of \citet{pinneri2021icem} all fall into this category.

Our usage differs from both lines: \textit{Follow the Winners} instead applies CEM elite-filtering as an optimization step for an ELBO derived from a Control-as-Inference graphical model, where the optimization targets either an amortized variational inference or expectation-maximization scheme.
Furthermore, we approximate the elite set for the projection step as a collection of \textit{trajectories} drawn from the replay buffer rather than parameters or action sequences.
The CEM step then acts as the bootstrap-and-project component of an approximate policy-improvement operator, with the empirical cross-entropy projection mapping the filtered trajectories back onto the policy parameters via SGD.
This connects the CEM elite-filtering operator to the order-statistics analysis of Section~\ref{section:ftw}, which are not present in any of the prior CEM-RL applications.
A practical consequence is that our method does not require a learned dynamics model (unlike planner-style CEM) and does not maintain a population of policies (unlike CEM-RL), which keeps it compatible with the critic-free RFT setting.

\paragraph{Order Statistics, Subset Selection, and Genetic Algorithms.}

The tournament-selection scheme used by Follow the Winners (Section~\ref{section:ftw}, Lemma~\ref{lemma:elite_filtering_marginal}) draws on classical literature in order statistics and evolutionary computation.
The marginal probability of selection by tournament rank is a standard order-statistics result \citep{david2003order, johnson1997discrete}, and tournament selection itself is a foundational operator in genetic algorithms (GAs) \citep{goldberg1989genetic, goldberg1991comparative} where the trade-off between tournament size and selection pressure has been studied extensively.
The $K=1$, $C=2$ case of our Eq.~(\ref{eq:elite_filtering_marginal}) recovers binary tournament selection exactly.
Relative to the GA literature, our contribution is to embed tournament selection inside a probabilistic framework: the resulting empirical occupancy is shown in Section~\ref{section:ftw} to act as an implicit tilting distribution that targets the upper quantile of returns without explicit potential evaluation.
This connects the GA-style operator to the cross-entropy method for optimization \citep{rubinstein_optimization_1997, rubinstein1999cem, botev2013cemopt} discussed in the previous subsection.
When $C=N$, our scheme reduces to standard CEM elite-filtering applied to the buffer, making the GA-style and CEM-style perspectives literally equivalent in this limit.
A separate but adjacent thread is sampling without replacement from a Plackett-Luce distribution via the Gumbel-top-$k$ trick \citep{gumbel1954, kool2019stochastic}, which we use to sample full PL rankings efficiently.

\section{GRPO and DPO Primer}
\label{appendix:primer}

For self-containedness, we shortly reintroduce the key components of GRPO and DPO in our notation, so that any reader jumping into our paper has them available.

\paragraph{GRPO.}
Group-Relative Policy Optimization \citep{shao2024deepseekmath} is a critic-free policy-gradient method built on PPO \citep{schulman2017proximal}.
For each prompt it samples a group of $N$ trajectories from the current policy, scores their returns $G_1, \ldots, G_N$, and replaces the learned value baseline of actor-critic methods with the group statistics through the advantage $A_i = (G_i - \bar{G}) / \sigma_G$, where $\bar{G}$ and $\sigma_G$ are the group mean and standard deviation.
The policy maximizes the PPO-style clipped objective with a KL penalty toward a reference policy,
\begin{equation}
    \mathcal{L}_{\text{GRPO}}(\theta)
    =
    \mathbb{E}\Bigl[ \tfrac{1}{N} {\textstyle\sum_{i=1}^{N}} \min\bigl(r_i(\theta) A_i,\, \mathrm{clip}(r_i(\theta), 1 \pm \epsilon) A_i\bigr) \Bigr]
    - \lambda\, D_{\mathrm{KL}}\bigl(\pi_\theta \,\|\, \piref\bigr),
    \label{eq:grpo_primer}
\end{equation}
with importance ratio $r_i(\theta) = \pi_\theta(H^{(i)}) / \pi_{\theta_{\text{old}}}(H^{(i)})$.
Trajectories above the group mean are reinforced and those below are suppressed, which makes the update risk-neutral with respect to the return distribution but sensitive to its scale.
The latter is also why variants like Dr.~GRPO \citep{liu2025understanding} remove the standard-deviation normalization, since it undercredits good trajectories in high-variance groups and overcredits nearly indistinguishable ones in low-variance groups.
Many further GRPO variants exist, as covered in Appendix~\ref{appendix:related_work}.

\paragraph{DPO.}
Direct Preference Optimization \citep{rafailov2023direct} learns from pairwise preferences instead of raw returns.
It assumes a Bradley--Terry preference model (from the RLHF framework), in which a latent reward $r$ generates preferences between two trajectories with probability $\Pr(H^{(w)} \succ H^{(l)}) = \sigma\bigl(r(H^{(w)}) - r(H^{(l)})\bigr)$.
Substituting the closed-form solution of the KL-regularized RL objective, $r(H) = \beta \ln \tfrac{\pi_\theta(H)}{\piref(H)}$, eliminates the reward model and yields the loss
\begin{equation}
    \mathcal{L}_{\text{DPO}}(\theta)
    =
    -\mathbb{E}_{(H^{(w)},\, H^{(l)})}\Bigl[ \ln \sigma\Bigl( \beta \ln \tfrac{\pi_\theta(H^{(w)})}{\piref(H^{(w)})} - \beta \ln \tfrac{\pi_\theta(H^{(l)})}{\piref(H^{(l)})} \Bigr) \Bigr],
    \label{eq:dpo_primer}
\end{equation}
which increases the likelihood of the preferred trajectory $H^{(w)}$ relative to the dispreferred $H^{(l)}$, anchored to the reference policy $\piref$.
In our experiments, preferences are induced by comparing sampled returns, and Appendix~\ref{appendix:derivation:recovery} recovers both objectives top-down from the CAI framework.

\section{Proofs} \label{appendix:proofs}

\subsection{Rao-Blackwell Argument and GRPO upper-bound} \label{appendix:proofs:rao_blackwell}
\begin{proof}[Proof of Proposition~\ref{prop:grpo}]
    Fix a trajectory $H_T^n$ and consider the conditional law of $\setB_N$ under $Q(\setB_N)=\prod_{j=1}^N q(H_T^j)$ given that $H_T^n$ is held fixed. Under this conditional law, the remaining $N-1$ batch elements are redrawn independently from $q$ (or replay buffer). By definition,
    \begin{align}
        \beta_n(H_T^n)=\mathbb{E}_{Q(\setB_N)}[\alpha_n(\setB_N)\mid H_T^n].
    \end{align}
    Now let $\setB_N^{(1)},\dots,\setB_N^{(M)}$ be i.i.d. conditional batch completions around the fixed trajectory $H_T^n$. Then $\alpha_n(\setB_N^{(m)})$ is an i.i.d. sequence, so by the law of large numbers,
    \begin{align}
        \frac{1}{M}\sum_{m=1}^M \alpha_n(\setB_N^{(m)})
        \xrightarrow[M\to\infty]{}
        \mathbb{E}_{Q(\setB_N)}[\alpha_n(\setB_N)\mid H_T^n]
        = \beta_n(H_T^n).
    \end{align}
    Hence $\alpha_n(\setB_N)$ is an unbiased Monte Carlo estimator of the per-sample Rao--Blackwell coefficient under conditional batch completion.

    The Jensen bound follows by upper-bounding the coefficient $\alpha_n$,
    \begin{align}  \label{eq:grpo_upperbounds_cai}
        \alpha_n(\setB_N)
        = \ln g_n - \ln \sum_{k=1}^N g_k
        \le \ln g_n - \frac{1}{N} \sum_{k=1}^N \ln g_k - \ln N
        \le \ln g_n - \frac{1}{N} \sum_{k=1}^N \ln g_k.
    \end{align}
    Then, plugging in the exponential potential over the cardinal returns $g_n = \exp G(H^n_T)$, we obtain $G(H^n_T) - \frac{1}{N} \sum_{k=1}^N G(H^k_T)$ for the right-hand side, which is exactly GRPO's group averaged baseline.
\end{proof}

\subsection{Probability of Selection under Elite-Filtering} \label{appendix:proofs:elite_filtering}

\begin{lemma}[Elite-Filtering Marginal]\label{lemma:elite_filtering_marginal}
    Consider a buffer of $N$ trajectories with a total ranking, where $r=1$ denotes the highest-return trajectory.
    An elite-filtering step draws a subset of size $C$ uniformly without replacement from the buffer and retains the top-$K$ ranked trajectories as elites.
    The normalized marginal occupancy of the trajectory at rank $r$ among the retained elites is
    \begin{align}
        \mathrm{Pr}(r \mid N, C, K)
        = \frac{1}{K\binom{N}{C}} \sum_{k=1}^{K} \binom{r-1}{k-1}\binom{N-r}{C-k},
        \label{eq:elite_filtering_marginal}
    \end{align}
    which sums to $1$ over $r \in \{1, \ldots, N\}$.
\end{lemma}

\begin{proof}
    \textbf{Step 1: Decompose by elite position.}
    Item $r$ is retained iff it is included in the sampled subset and lies among the top-$K$ ranks within that subset.
    It occupies exactly elite position $k \in \{1, \ldots, K\}$ iff exactly $k-1$ of the other sampled items have rank $< r$ and exactly $C-k$ have rank $> r$.
    These events are mutually exclusive across $k$, so
    \begin{align*}
        \mathrm{Pr}(\text{item } r \text{ retained})
        = \sum_{k=1}^{K} \mathrm{Pr}(\text{item } r \text{ at elite position } k).
    \end{align*}

    \textbf{Step 2: Apply the multivariate hypergeometric distribution.}
    The buffer is partitioned into three groups: $\{r\}$ (size 1), the $r-1$ better items, and the $N-r$ worse items. 
    By the multivariate hypergeometric distribution \citep[Ch.~6]{johnson1997discrete}, the probability of drawing exactly 1 item from $\{r\}$, exactly $k-1$ from the $r-1$ better items, and exactly $C-k$ from the $N-r$ worse items in a uniform draw of $C$ items without replacement is,
    \begin{align*}
        \mathrm{Pr}(\text{item } r \text{ occupies position } k) 
        = \frac{\binom{1}{1}\binom{r-1}{k-1}\binom{N-r}{C-k}}{\binom{N}{C}}
        = \frac{\binom{r-1}{k-1}\binom{N-r}{C-k}}{\binom{N}{C}}.
    \end{align*}
    Summing over $k = 1, \ldots, K$ gives the unnormalized marginal
    $\sum_{k=1}^{K}\binom{r-1}{k-1}\binom{N-r}{C-k} / \binom{N}{C}$.
    
    \textbf{Step 3: Normalization.}
    Summing over all buffer items and exchanging the order of summation,
    \begin{align*}
        \sum_{r=1}^{N} \sum_{k=1}^{K} 
        \binom{r-1}{k-1}\binom{N-r}{C-k}
        = \sum_{k=1}^{K} \sum_{r=1}^{N} 
        \binom{r-1}{k-1}\binom{N-r}{C-k}
        = \sum_{k=1}^{K} \binom{N}{C}
        = K\binom{N}{C},
    \end{align*}
    where the inner sum follows from the Vandermonde identity $\sum_{r=1}^{N}\binom{r-1}{k-1}\binom{N-r}{C-k} = \binom{N}{C}$ \citep[Ch.~5]{graham1994concrete}.
    Hence the unnormalized marginal sums to $K$ over all $r$, and dividing by $K$ yields a proper probability distribution summing to 1, confirming our result in Eq.~\eqref{eq:elite_filtering_marginal}.
\end{proof}

\paragraph{Boundary Cases.}
A few boundary cases are immediate.
For $K=1$, only the $k=1$ term survives, giving $\mathrm{Pr}(r \mid N,C,1) = \binom{N-r}{C-1}/\binom{N}{C}$, which is strictly decreasing in $r$ and satisfies $\mathrm{Pr}(1 \mid N,C,1) = C/N$.
For $K=C$, Vandermonde's identity yields $\sum_{k=1}^{C}\binom{r-1}{k-1}\binom{N-r}{C-k} = \binom{N-1}{C-1}$, so $\mathrm{Pr}(r \mid N,C,C) = 1/N$ for all $r$, recovering uniform sampling.

\begin{remark}[title={Remark. Variations of the Marginal \& Duplicates in Elite-Filtering.}]
    Lemma~\ref{lemma:elite_filtering_marginal} assumes trajectories in the buffer are distinct and that each elite-filtering step draws $C$ items uniformly without replacement from the full buffer of size $N$, with repeated steps drawn independently.
    Duplicate trajectories, sampling with replacement within a draw, or sampling without replacement across repeated elite-filtering steps would slightly alter the marginal probabilities.
    We include Lemma~\ref{lemma:elite_filtering_marginal} mainly to illustrate the induced tilting distribution underlying our policy optimization procedure; a full treatment of all such variants is outside the scope of this work.
\end{remark}

\subsection{Risk Offsets across Return Distributions}
\label{appendix:proofs:risk}
\label{appendix:proofs:risk_table}

We compare the implicit risk preferences of the algorithms when the return $G(H_T)$ is genuinely stochastic given the observable trajectory, rather than reduced to its conditional mean as in Eq.~\eqref{eq:posterior_policy_cai}.
The noise may be aleatoric or epistemic, and the critic-free setting has no value function to fall back on as a surrogate for the mean return.
We compute the risk offset of each algorithm for four concrete return families under a single recipe, summarized in Table~\ref{appendix:tab:risk_offsets}.

\textbf{Setup.}
Each buffer entry carries a trajectory $H_T^i$ and a return $G_i \sim P(\cdot \mid H_T^i)$, drawn independently across entries.
Write $\mu(H_T) = \bbE[G \mid H_T]$ and $\sigma^2(H_T) = \mathrm{Var}[G \mid H_T]$.
When we resample data from the buffer, we treat sampled returns as i.i.d.\ from the population marginal $F_0$ induced by the buffer's occupancy (density $f_0$ where the return family is continuous). 
The buffer's empirical distribution estimates $F_0$ up to finite-sample error, and we assume for methods like GRPO that the buffer is simply refreshed at each iteration to fit within this framing.
The ordinal operators compare returns, and discrete families produce ties, which we break uniformly at random.
This tie-breaking is captured by the randomized grade of a return value $g$,
$U(g) = \Pr(G' < g) + V \cdot \Pr(G' = g)$ with $G' \sim F_0$ and $V \sim \mathrm{Unif}(0,1)$ refreshed at each comparison and $U(g) \in [0, 1]$.
In words, the grade is the fraction of buffer mass strictly below $g$, plus a uniformly random position within the tie mass at $g$, so comparing returns with tie coin-flips is the same as comparing grades.
For continuous families the grade reduces to the buffer quantile $F_0(g)$, and the grade of a buffer draw is $\mathrm{Unif}(0,1)$.

Every batch-level potential $\Phi(\setB_N)$ below reduces to a per-entry selection weight
\begin{align}
    w_\Phi(H_T) \;=\; \bbE_{G \mid H_T}[\varphi(G)],
\end{align}
with $\varphi$ an increasing single-return utility.
The reduction holds one entry $H_T$ fixed and marginalizes the remaining batch members over the buffer (exactly like the Rao--Blackwell conditioning of Proposition~\ref{prop:grpo}).
For example take the MNL/softmax, the batch weight $e^{G_n/\tau} / \sum_k e^{G_k/\tau}$ has a denominator shared by the whole batch that concentrates around its mean over batch completions, so to first order this reduces to $w_\Phi$ with $\varphi(g) = e^{g/\tau}$.
The induced utilities are:
\begin{itemize}
    \item \textbf{MNL-exp} (softmax on cardinal returns, temperature $\tau$): $\varphi(g) = e^{g/\tau}$, and $\tau = 1$ recovers the potential $g_n = \exp G$ of Proposition~\ref{prop:grpo}.
    \item \textbf{GRPO}: The Jensen surrogate of the Rao-Blackwellized MNL-exp (Proposition~\ref{prop:grpo}) which scores each entry by the group-mean advantage, so $\varphi(g) = g/\tau$.
    \item \textbf{DPO} (pairwise): the winner of a pair is the higher grade in randomly drawn comparisons. Thus, an entry with grade $U(g)$ gives $\varphi(g) = \bbE_V[U(g)]$, the win rate averaged over ties.
    \item \textbf{FTW top-$K$/$C$} (ranks): the entry is retained iff at most $K - 1$ of its $C - 1$ competitors beat it, and each does so independently with probability $1 - U(g)$, giving $\varphi(g) = \bbE_V[B_{C,K}(U(g))]$ with $B_{C,K}(u) = \Pr(\mathrm{Bin}(C-1, 1-u) \le K-1)$, which matches Lemma~\ref{lemma:elite_filtering_marginal} in the large-buffer limit up to the entry-independent inclusion factor.
    \item \textbf{FTW softmax-on-ranks} (temperature $\tau$): $\varphi(g) = \bbE_V[e^{U(g)/\tau}]$, the Boltzmann weight on ranks (determined through sorting returns).
\end{itemize}

\begin{definition}[Certainty equivalent and risk offset]
\label{ap:def:ce}
Let $\varphi$ be an increasing utility.
Define the certainty equivalent
$J_\varphi(H_T) \eqdef \varphi^{-1}\big(\bbE_{G \mid H_T}[\varphi(G)]\big)$ and the risk offset as $\mathcal{R}_\varphi(H_T) \eqdef J_\varphi(H_T) - \mu(H_T)$.
\end{definition}

A few remarks.
First, $\mathcal{R}_\varphi \neq 0$ requires a non-degenerate return law and a non-affine $\varphi$.
Second, as the noise vanishes, $G \to \mu(H_T)$ and every cell of Table~\ref{appendix:tab:risk_offsets} gives $\mathcal{R} = 0$.
Third, an affine utility commutes with the expectation, so $J_\varphi = \mu$ and $\mathcal{R}_\varphi = 0$ exactly, which settles the GRPO row in every column of Table~\ref{appendix:tab:risk_offsets}.
Then finally, the score can be ambiguous.
For example, take two buffers with values $\{0, 3, 10\}$ and $\{0, 7, 10\}$: the grade would be equal for items 3 and 7 despite taking on different actual values.
In words, ordinal operators resolve only the ordering of returns and not their values, so the inverse $\varphi^{-1}$ covers an interval rather than a single number.
In these cases, we write $0^\dagger$ whenever the operator still ranks entries properly by their mean return, to distinguish an inexpressible offset from a computed zero.

To derive generic risk profiles $\mathcal{R}$ for the algorithms/method variations on varying noise distributions, we first layout a general starting point in the following lemma.
This is followed by a proposition where we derive each risk profile of interest.

\begin{lemma}[Theory recipe]
\label{ap:lem:ce_recipe}
Let $w_\Phi = \bbE_{G|H_T}[\varphi(G)]$ and let $\sigma^2 = \mathrm{Var}[G \mid H_T]$ as in the setup.
\begin{enumerate}
    \item[(i)] For $\varphi(g) = e^{g/\tau}$, the certainty equivalent is the scaled log moment-generating function $J_\varphi = \tau \ln \bbE_{G \mid H_T}[e^{G/\tau}]$, exactly.
    \item[(ii)] For $\varphi$ twice differentiable near $\mu$ and $\sigma^2 \to 0$ we get the certainty-equivalent expansion,
    \begin{align}
        J_\varphi \;=\; \mu + \frac{\sigma^2}{2} \frac{\varphi''(\mu)}{\varphi'(\mu)} + o(\sigma^2),
    \end{align}
    with Arrow--Pratt coefficient $\varphi''/\varphi'$ \citep{pratt1964risk}.
    \item[(iii)] Grade utilities depend on returns only through $U(g)$, so both FTW and DPO are invariant to any strictly increasing transformation of returns.
\end{enumerate}
\end{lemma}

\begin{proof}
(i) Simply invert the exponential.
(ii) Take the Taylor expansion for $\varphi(G)$ around $\mu$ inside the expectation, the linear term vanishes since $\bbE[G - \mu] = 0$, leaving $w_\Phi = \varphi(\mu) + \tfrac{1}{2}\sigma^2 \varphi''(\mu) + o(\sigma^2)$, which works when the higher central moments are $o(\sigma^2)$. 
Then recall that $\varphi(J_\varphi) = w_\Phi$. For $\varphi'(\mu) \neq 0$, expanding $\varphi(J_\varphi)$ around $\mu$ gives $\varphi(\mu) + \varphi'(\mu)(J_\varphi - \mu) + o(\sigma^2)$.
Simple rearrangement, cancelling $\varphi(\mu)$ and dividing by $\varphi'(\mu)$, gives the claim.
(iii) $U(g)$ depends on $g$ only through the events $\{G' < g\}$ and $\{G' = g\}$, which any strictly increasing transformation of returns preserves.
\end{proof}

\begin{proposition}[Risk offsets per return family]
\label{ap:prop:risk_table}
For each distribution (column) in Table~\ref{appendix:tab:risk_offsets}, denote the conditional return law as $\mathrm{Law}(G_t \mid H_T)$ and assume that the buffer marginal $F_0$ bootstraps from this parametric law, writing the buffer's implicit parameters as the subscript `$0$', so that $F_0 = \mathcal{N}(\mu_0, \sigma_0^2)$ in the Gaussian column, and likewise $p_0$, $\lambda_0$, $\theta_0$ for the other columns.
Then the risk offset $\mathcal{R}_\varphi(H_T)$ of each operator is as given in Table~\ref{appendix:tab:risk_offsets}.
Cells are exact unless marked $\approx$ (a normal approximation, of the Poisson family or of the elite threshold) or written with an explicit $o(\sigma^2)$ remainder.
\end{proposition}

\begin{table}[ht]
\centering
\caption{
    Risk offsets $\mathcal{R} = J - \mu$ per operator (rows) and return family (columns) ordered from least aggressive (top) to most aggressive (bottom). 
    Sparse rewards use scaled returns $m$ and success rate $p$, such that $\mu = mp$ allowing the mean to be invariant to larger or smaller $p$ for properly chosen $m$. 
    Similarly, counts use $G = m \cdot \mathrm{Poi}(\lambda)$, so $\mu = m\lambda$ and $\sigma^2 = m^2 \lambda$. 
    The exponential family has rate $\theta$ and mean $1/\theta$ and requires the MNL-exp cell to satisfy $\theta\tau > 1$ to be defined. 
    GRPO's zeros are exact, while $0^\dagger$ marks cells whose ranking is mean-consistent but whose offset is not expressible in return units (the certainty equivalent is interval-valued). The constants $a_{C,K}$ and $s_{C,K}^2$ are the mean and variance of the $K$-th largest value among $C-1$ buffer returns standardized to mean $0$ and variance $1$, $t = \mu_0 + \sigma_0 a_{C,K}$ denotes the elite threshold, $H_n$ and $H^{(2)}_n$ are the harmonic numbers of first and second order, and $f_0$ is the density of the population marginal $F_0$.
}
\label{appendix:tab:risk_offsets}
\setlength{\tabcolsep}{3pt}
\resizebox{\linewidth}{!}{%
    \begin{tabular}{lllll}
    \toprule
     & Gaussian $\mathcal{N}(\mu, \sigma^2)$ & Sparse $m\,\mathrm{Ber}(p)$ & Counts $m\,\mathrm{Poi}(\lambda)$ & Exponential($\theta$) \\
    \midrule
    GRPO & $0$ & $0$ & $0$ & $0$ \\
    [2pt]
    DPO & $(\mu - \mu_0)\Big(\tfrac{\sigma_0}{\sqrt{\sigma_0^2 + \sigma^2}} - 1\Big)$ & $0^\dagger$ & $\approx$ Gaussian cell, $\sigma^2 = m^2\lambda$ & $\tfrac{1}{\theta_0} \ln\!\big(1 + \tfrac{\theta_0}{\theta}\big) - \tfrac{1}{\theta}$ \\
    [2pt]
    FTW top-$K$/$C$ & $\approx (\mu - t)\Big(\tfrac{\sigma_0 s_{C,K}}{\sqrt{\sigma^2 + \sigma_0^2 s_{C,K}^2}} - 1\Big)$ & $0^\dagger$ & $\approx$ Gaussian cell, $\sigma = m\sqrt{\lambda}$ & $\approx$ Gaussian cell, $(a_{C,K}, s_{C,K})$ via $H, H^{(2)}$ \\
    [2pt]
    FTW smx-ranks & $\dfrac{\sigma^2}{2}\Big(\dfrac{f_0(\mu)}{\tau} - \dfrac{\mu - \mu_0}{\sigma_0^2}\Big) + o(\sigma^2)$ & $0^\dagger$ & $\approx$ Gaussian cell & $\dfrac{\sigma^2 \theta_0}{2}\Big(\dfrac{e^{-\theta_0 \mu}}{\tau} - 1\Big) + o(\sigma^2)$ \\
    [2pt]
    MNL-exp & $\dfrac{\sigma^2}{2\tau}$ & $\tau \ln(1 - p + p e^{m/\tau}) - mp$ & $\tau \lambda (e^{m/\tau} - 1) - m\lambda$ & $-\tau \ln\!\big(1 - \tfrac{1}{\theta\tau}\big) - \tfrac{1}{\theta}, \theta\tau>1$ \\
    \bottomrule
    \end{tabular}
}
\end{table}

\begin{proof}
\textbf{GRPO row.} 
Conditioned on $H_T$, the batch mean averages the entry's own return with its $N - 1$ competitors, $\bbE[\bar G \mid H_T] = \tfrac{1}{N}\mu(H_T) + \tfrac{N-1}{N}\mu_0$, so $\bbE[G - \bar G \mid H_T] = (1 - \tfrac{1}{N})\,(\mu(H_T) - \mu_0)$.
This is linear in $\mu(H_T)$ with a shared constant.
Marginally over the buffer the advantage averages to zero, but the conditional average determines the ranking in again a linear fashion.
Thus $\mathcal{R}=0$.

\textbf{(Sparse) Bernoulli column.} 
On the support $\{0, m\}$ any utility gives $w_\Phi = (1-p)\,\varphi(0) + p\,\varphi(m)$, which is affine in $p$.
Thus, every operator ranks by the mean $mp$. 
The grade utilities are flat between any other value between the two payoffs and admit no unique $J_\varphi$, so their cells record the mean-rank as $0^\dagger$ rather than actually zero. 
The next step shows the only exception in the MNL-exp cell.

\textbf{MNL-exp row.} 
By Lemma~\ref{ap:lem:ce_recipe}(i), $J_\varphi = \tau \ln \bbE_{G \mid H_T}[e^{G/\tau}]$, this is the moment generating function and has a standard closed form for each of the considered distributions: Gaussian $e^{\mu/\tau + \sigma^2/2\tau^2}$, scaled Bernoulli $1 - p + p e^{m/\tau}$, scaled Poisson $e^{\lambda(e^{m/\tau} - 1)}$, exponential $(1 - \tfrac{1}{\theta\tau})^{-1}$ for $\theta\tau > 1$ and infinite otherwise.
Subtracting the means gives the cells.
All four reduce to $\sigma^2/2\tau$ at small noise.
The only difference for MNL-exp in the (sparse) Bernoulli column is the $e^{m/\tau}$ factor, the weight is sensitive to the return's value, whereas the ordinal operators depend only on its rank.

\textbf{DPO row.} 
At a sure return $g$, the utility is an independent average over opponents, $\varphi(g) = \bbE_V[U(g)] = \Pr(G' < g) + \tfrac{1}{2}\Pr(G' = g)$ with $G' \sim F_0$, i.e., the win rate of $g$ against any one buffer draw. 
Then, averaging over the entry's own return gives the win rate of the noisy variant, $w_\Phi = \bbE_{G \mid H_T}[\varphi(G)] = \Pr(G' < G) + \tfrac{1}{2}\Pr(G' = G)$ with $G'$ independent of $G$.
For continuous values $\Pr(G' = G) = 0$.

\emph{Gaussian:} 
The weight is equivalently $w_\Phi = \Pr(G' < G) = \Pr(G - G' > 0)$.
The difference of independent Gaussians is Gaussian, $G - G' \sim \mathcal{N}(\mu - \mu_0, \sigma^2 + \sigma_0^2)$, so standardizing gives $w_\Phi = \Phi\big((\mu - \mu_0)/\sqrt{\sigma^2 + \sigma_0^2}\big)$, with $\Phi$ here the standard normal CDF.
The certainty equivalent solves $\varphi(J_\varphi) = w_\Phi$, where for the sure return $J_\varphi$ only the opponent's draw is random.
Thus, its utility is the quantile $\varphi(J_\varphi) = \Phi\big((J_\varphi - \mu_0)/\sigma_0\big)$.
Applying $\Phi^{-1}$ to both sides and solving for $J_\varphi$ gives $J_\varphi = \mu_0 + \sigma_0 (\mu - \mu_0)/\sqrt{\sigma^2 + \sigma_0^2}$, and subtracting $\mu$ factors into the cell $\mathcal{R} = (\mu - \mu_0)\big(\sigma_0/\sqrt{\sigma_0^2 + \sigma^2} - 1\big)$.

\emph{Exponential:} 
The Gaussian steps carry over with two substitutions, namely the expressions for the win rate and the CDF.
The win rate becomes $w_\Phi = \Pr(G > G') = \int_0^\infty e^{-\theta x}\, \theta_0 e^{-\theta_0 x}\, dx = \theta_0/(\theta + \theta_0)$.
At the sure return $J_\varphi$ only the opponent is random, so the utility becomes the CDF, $\varphi(J_\varphi) = \Pr(G' < J_\varphi) = 1 - e^{-\theta_0 J_\varphi}$.
Setting this equal to $w_\Phi$ gives $e^{-\theta_0 J_\varphi} = \theta/(\theta + \theta_0)$, hence $J_\varphi = \tfrac{1}{\theta_0} \ln(1 + \theta_0/\theta)$, and subtracting the mean $1/\theta$ gives the table result.

\emph{Counts:} 
Let $G = mk$ with $k$ the count. Averaging the win-and-tie probability over the opponents gives the exact weight
$w_\Phi = \sum_k \mathrm{Poi}_\lambda(k) \big[ F_{\lambda_0}(k - 1) + \tfrac{1}{2}\mathrm{Poi}_{\lambda_0}(k) \big]$,
where $F_{\lambda_0}(k - 1) = \Pr(G' < mk)$ is the Poisson CDF and the scale $m$ cancels due to the invariance as per Lemma~\ref{ap:lem:ce_recipe}(iii).
This sum has no closed form, and the utility is a step function.
Thus the table only reports the Gaussian approximation with moment-matched laws $\mathcal{N}(m\lambda, m^2\lambda)$ and $\mathcal{N}(m\lambda_0, m^2\lambda_0)$, which is valid for large $\lambda$ and $\lambda_0$.

\textbf{FTW top-$K$/$C$ row.} 
An entry is retained iff its return exceeds the $K$-th largest of its $C - 1$ competitors.
To formalize this, we will standardize an entry's returns according to its own law $G = \mu + \sigma Z$ with $Z \eqdef (G - \mu)/\sigma$.
This assumes nothing but gives us $Z$ with mean $0$, variance $1$, while preserving the underlying family (Table~\ref{appendix:tab:risk_offsets}).
Each of the $C - 1$ competitors are then also standardized but by the buffer's $(\mu_0, \sigma_0)$.
This holds exactly for the continuous columns, since every Gaussian standardizes to $\mathcal{N}(0,1)$.
For the exponential law we know that the mean and standard deviation are both $1/\theta$, so $Z = (G - 1/\theta)/(1/\theta) = \theta G - 1$, and $\theta G \sim \mathrm{Exp}(1)$ for any rate $\theta$.
This recipe fails for the discrete columns since a linear standardization can only fix the first two moments and their third and higher moments will still depend on the parameter.

Now write the random threshold as $T \eqdef \mu_0 + \sigma_0 Z_{(K)}$, the $K$-th largest competitor return, where $Z_{(K)}$ is the $K$-th largest of the $C - 1$ standardized competitor draws.
Retention is the event $G \ge T$, so the utility is the threshold's CDF, $\varphi(g) = \Pr(T \le g) \eqdef F_T(g)$, and the certainty equivalent solves $F_T(J_\varphi) = \bbE_{G \mid H_T}[F_T(G)]$.
For the pairwise case $(C, K) = (2, 1)$ where $T$ is a single buffer draw we obtain exactly the DPO setting.
In every other case we have that the law of $T$ is known exactly, $F_T(g) = B_{C,K}(F_0(g))$ with $B_{C,K}(u) = \Pr(\mathrm{Bin}(C-1, 1-u) \le K-1)$, but it is inconvenient and not Gaussian for $C > 2$.
Moreover, $J_\varphi = F_T^{-1}(\bbE[F_T(G)])$ has no closed form.
For larger $C$ we therefore approximate $T$ by a Gaussian, and any residual error caused by the higher moments of $Z_{(K)}$ (which is typically skewed with a longer tail on the high returns) is what the table's $\approx$ marks.
Writing $a_{C,K} \eqdef \bbE[Z_{(K)}]$ and $s_{C,K}^2 \eqdef \mathrm{Var}[Z_{(K)}]$ for the mean and variance of the $K$-th largest of $C - 1$ standardized draws, we obtain $t \eqdef \bbE[T] = \mu_0 + \sigma_0 a_{C,K}$ and $\mathrm{sd}(T) = \sigma_0 s_{C,K}$.
Approximating $T$ by a Gaussian with these two exact moments allows us to reuse the DPO steps with $(\mu_0, \sigma_0)$ replaced by $(t, \sigma_0 s_{C,K})$.
The weight becomes $w_\Phi = \Pr(G - T \ge 0) = \Phi\big((\mu - t)/\sqrt{\sigma^2 + \sigma_0^2 s_{C,K}^2}\big)$, since $G - T \sim \mathcal{N}(\mu - t, \sigma^2 + \sigma_0^2 s_{C,K}^2)$ for independent Gaussians.
At the sure return only the threshold is random, so $\varphi(J_\varphi) = \Pr(T \le J_\varphi) = \Phi\big((J_\varphi - t)/(\sigma_0 s_{C,K})\big)$.
Setting the two equal, applying $\Phi^{-1}$, and solving gives $J_\varphi = t + \sigma_0 s_{C,K} (\mu - t)/\sqrt{\sigma^2 + \sigma_0^2 s_{C,K}^2}$, and subtracting $\mu$ factors into
\begin{align} \label{eq:rank_shrinkage}
    \mathcal{R} \;=\; (\mu - t)\left( \frac{\sigma_0 s_{C,K}}{\sqrt{\sigma^2 + \sigma_0^2 s_{C,K}^2}} - 1 \right).
\end{align}
This is the same shrinkage form as the DPO row, with the buffer mean $\mu_0$ attractor substituted by the elite threshold $t$.
The remaining paragraphs show the $(a_{C,K}, s_{C,K})$ pair for each distribution.

\emph{Gaussian:}
Due to standardization, we have that $a_{C,K}$ and $s_{C,K}$ are simply the mean and standard deviation of the $K$-th largest among $C - 1$ independent $\mathcal{N}(0, 1)$ draws.
Neither has a closed form, but both follow from one-dimensional quadrature against the order-statistic density,
$a_{C,K} = \tfrac{(C-1)!}{(K-1)!\,(C-1-K)!} \int_{\mathbb{R}} z\, \Phi(z)^{C-1-K} (1 - \Phi(z))^{K-1}\, \mathrm{d}\Phi(z)$, where the same integral with $z^2$ in the integrand gives the second moment and $s_{C,K}^2$ follows by subtracting $a_{C,K}^2$.
Concretely, this gives example values $(a, s) = (0, 1)$ at $(2,1)$, $(0.85, 0.75)$ at $(4,1)$, $(0, 0.67)$ at $(4,2)$, $(-0.85, 0.75)$ at $(4,3)$, and $(1.35, 0.63)$ at $(8,1)$, which we verified numerically.
Note that $(4,2)$ shares the attractor $t = \mu_0$ with $(2,1)$, the median of three symmetric draws has mean zero, but its smaller $s_{C,K}$ makes the shrinkage factor smaller at every $\sigma$, so the pull toward $\mu_0$ is stronger.
At $(C, K) = (2, 1)$ the threshold is a single Gaussian buffer draw, so nothing is approximated, $t = \mu_0$, $s_{2,1} = 1$, and the cell coincides with the DPO row exactly.
For very large $C$ the mean grows like $\sqrt{2 \ln(C/K)}$, in the sense that $a_{C,K} / \sqrt{2 \ln(C/K)} \to 1$ as $C \to \infty$ \citep[Ch.~10]{david2003order}.

\emph{Exponential:}
Both constants are closed-form.
The $K$-th largest of $C - 1$ i.i.d.\ $\mathrm{Exp}(1)$ variables has mean $H_{C-1} - H_{K-1}$ and variance $H^{(2)}_{C-1} - H^{(2)}_{K-1}$, with $H^{(2)}_n \eqdef \sum_{i=1}^{n} i^{-2}$ \citep[Ch.~4]{david2003order}, so standardizing gives $a_{C,K} = H_{C-1} - H_{K-1} - 1$ and $s^2_{C,K} = H^{(2)}_{C-1} - H^{(2)}_{K-1}$ at any finite $C$.
The cell plugs these into the Gaussian coefficient, and it stays approximate despite the exact constants because exponential order statistics are more strongly skewed than normal ones, so the Gaussian law of $T$ is coarser here.
At $(2, 1)$ the pair is again $(0, 1)$ which follows DPO.

\emph{Counts:}
As noted before, the standardized shape of $m\,\mathrm{Poi}(\lambda)$ still depends on $\lambda$, so the location-scale step does not hold exactly for this family.
We therefore replace the Poisson laws by their moment-matched normals, $\mathcal{N}(m\lambda, m^2\lambda)$ for the entry and $\mathcal{N}(m\lambda_0, m^2\lambda_0)$ for the buffer, which is accurate when $\lambda$ and $\lambda_0$ are large.
Under this substitution the Gaussian cell applies verbatim, with $\sigma = m\sqrt{\lambda}$ and the Gaussian pair $(a_{C,K}, s_{C,K})$.
Although it is possible to compute the exact moments of the Poisson elite threshold, the cell's error comes from treating $T$ as Gaussian rather than from the two constants, so plugging exact constants into the approximate formula would not make the result exact.
For reference, this would give $a_{C,K} = (\bbE[k_{(K)}] - \lambda_0)/\sqrt{\lambda_0}$, where the mean follows from the tail sum $\bbE[k_{(K)}] = \sum_{j \ge 0} \Pr(\mathrm{Bin}(C-1, q_j) \ge K)$ with $q_j = \Pr(\mathrm{Poi}(\lambda_0) > j)$ and $k_{(K)}$ the $K$-th largest of $C-1$ i.i.d.\ $\mathrm{Poi}(\lambda_0)$ draws.

\textbf{FTW softmax-on-ranks row.}
Start from the utility $\varphi(g) = \bbE_V[e^{U(g)/\tau}]$.
In the continuous columns ties have probability zero, so the grade is the plain quantile $U(g) = F_0(g)$ and $\varphi(g) = e^{F_0(g)/\tau}$.
For Lemma~\ref{ap:lem:ce_recipe}(ii) we differentiate, $\varphi' = \tfrac{f_0}{\tau}\,\varphi$ and $\varphi'' = \big(\tfrac{f_0^2}{\tau^2} + \tfrac{f_0'}{\tau}\big)\varphi$, so the Arrow--Pratt coefficient is $\varphi''/\varphi' = f_0/\tau + f_0'/f_0$ and the lemma gives
$\mathcal{R} = \tfrac{\sigma^2}{2} \big( f_0(\mu)/\tau + (\ln f_0)'(\mu) \big) + o(\sigma^2)$.
We can then insert the PDFs.
\emph{Gaussian:} $(\ln f_0)'(\mu) = -(\mu - \mu_0)/\sigma_0^2$, giving the cell.
\emph{Exponential:} $f_0(\mu) = \theta_0 e^{-\theta_0 \mu}$ and $(\ln f_0)' = -\theta_0$, giving the cell.
\emph{Counts:} the grade is a step function and the expansion cannot be applied directly. 
For the Poisson specifically, however, we can assume the normal limit $\mathcal{N}(m\lambda_0, m^2\lambda_0)$ and reuse the Gaussian result, with the caveat that $\lambda_0$ must be large.
\end{proof}

\paragraph{Discussion.}
Read top to bottom, Table~\ref{appendix:tab:risk_offsets} orders the operators by risk sensitivity, GRPO exactly neutral, the three grade operators boundedly risk-seeking, MNL-exp unboundedly so.
GRPO's utility is affine, so it is exactly neutral in every column, this was also briefly shown empirically in Figure~\ref{fig:tabular_auc}.
It says nothing about learning stability due to exploding variance however.
One takeaway from Prop.~\ref{ap:prop:risk_table} is that the grade-based operators are always boundedly risk-seeking due to being scale agnostic.
For FTW the attractor is the elite threshold $t = \mu_0 + \sigma_0 a_{C,K}$.
Mild filters with $K$ near $C$ have $a_{C,K} \le 0$ and put it at or below the buffer mean $\mu_0$, DPO and the half filters $K = C/2$ have $a_{C,K} = 0$ and put it at the buffer mean exactly, and aggressive filters with $K \ll C$ have $a_{C,K} > 0$ and raise it above.
The pull grows with the entry's own $\sigma$ but saturates at the attractor, so no amount of noise scores an entry above $t$.
MNL-exp uses the full return magnitude and is unboundedly risk-seeking, which also makes it fragile.

The Gaussian cell also predicts the $\sigma$-trap of Figure~\ref{fig:tabular_auc} (left).
So long as both arms remain represented in the buffer, a noisy Gaussian distractor lands above or below the optimal mean with probability $\tfrac12$ as its variance grows.
Therefore, every ordinal comparison converges towards a coin flip and the distilled policy should approach a uniform distribution.
Aggressive filters with $a_{C,K} > 0$ fail before this regime, because the buffer spread $\sigma_0$ grows with the distractor's noise, so already at moderate noise their threshold $t = \mu_0 + \sigma_0 a_{C,K}$ rises above the optimal mean.
This excludes the near-deterministic arm more and more often, and eventually the noisy arm crowds the buffer.
We show this in the supplementary results in Figure~\ref{appendix:figure:sigma_toggle}, where for aggressive $(C, K)$ variants the recovered policy converges to the trap-arm.

We close with limitations of this theoretical section.
Firstly, the Arrow-Pratt expansion behind the FTW softmax-on-ranks row holds in the small-noise regime $\sigma \ll \sigma_0$, which is not typically realistic, whereas the other rows hold at every $\sigma$.
Secondly, the elite threshold $T$ is approximated by a normal to obtain a closed form, its exact law is skewed, so the FTW top-$K$/$C$ cells are exact only in the pairwise Gaussian case.
Thirdly, the buffer marginal is in general a mixture over conditional return laws, in practice these laws are not guaranteed homogenous.
Our single-family assumption on $F_0$ is thus a homogeneity idealization.
Lastly, on sparse or stepwise rewards the grade operators admit no unique certainty equivalent, the $0^\dagger$ cells.
On these discrete rewards they are also jackpot-blind, at a fixed mean $mp$ their weight for a rare large payoff collapses to that of a sure zero as $p \to 0$ (which implies that the jackpot $m$ must grow), whereas MNL-exp instead diverges.

\section{Extended Derivation} \label{appendix:derivation}

\subsection{Preliminaries} \label{app:cai_preliminaries}

Let $\pi_{\theta}$ be a history-dependent parametric policy acting in a partially observable Markov decision process. 
We write the observable history at time $t$ as $H_t = (O_0, A_0, \dots, O_{t-1}, A_{t-1}, O_t)$, and the full observable trajectory as $H_{0:T} = (O_0, A_0, O_1, A_1, \ldots, O_T)$. 
The policy induces a distribution over observable trajectories by marginalizing over latent states,
\begin{align}
    p_{\pi_\theta}(H_{0:T})
    = \int p(S_0)p(O_0 \mid S_0)\prod_{t=0}^{T-1} \pi_\theta(A_t \mid H_t)\, p(S_{t+1} \mid S_t, A_t)\, p(O_{t+1} \mid S_{t+1})\, dS_{0:T}.
\end{align}
Denote the set of admissible policies by $\Pi \eqdef \{\pi \mid \pi : \mathcal{H} \times \Theta \rightarrow \mathcal{P}(\mathcal{A})\}$, so that $\pi_\theta \in \Pi$.

Since rewards depend on the latent state-action process, the return is defined as
\begin{align}
    G(S_{0:T}, A_{0:T-1}) \eqdef \sum_{t=0}^{T} r(S_t, A_t) \in \mathbb{R}.
\end{align}
For notational convenience, we write $G(H_{0:T})$ for the return random variable induced by the latent state-action trajectory underlying $H_{0:T}$, i.e. $G(H_{0:T}) \sim \mathrm{Law}\!\left(\sum_{t=0}^{T} r(S_t,A_t)\,\middle|\, H_{0:T}\right).$
The objective is then
\begin{align}
    \theta^* \eqdef \argmax_\theta \, \mathbb{E}_{p_{\pi_\theta}(S_{0:T},H_{0:T})}\!\left[G(S_{0:T}, A_{0:T-1})\right].
\end{align}

For convenience, in the infinite-horizon setting $T \to \infty$, assume that the transition kernel satisfies $\Pr(S_{t+1}=s_{\mathit{absorb}} \mid S_t, A_t) > 0$ for all $S_t \in \mathcal{S}$, where $s_{\mathit{absorb}}$ is an absorbing state that transitions to itself with probability $1$ and yields zero reward. 
Equivalently, one may work with a discount factor $\gamma \in [0,1)$, which can be interpreted as an effective termination probability $p_{\mathit{term}} = 1-\gamma$.

Define a binary optimality variable $\mathcal{O}_T \in \{0,1\}$, which we use to condition the graphical model at the trajectory level. 
For simplicity, this corresponds to a sparse-reward setting in which rewards are evaluated only on completion. 
Most importantly, the optimality variable depends only on the return, so that
\begin{align}
    \mathcal{O}_T \indep (S_{0:T}, H_{0:T}) \mid G(S_{0:T}, A_{0:T-1}).
\end{align}
We can then perform model selection by maximizing the evidence
\begin{align}
    \max_\theta p_{\pi_\theta}(\mathcal{O}_T = 1)
    =
    \max_\theta \int p_{\pi_\theta}(\mathcal{O}_T = 1, S_{0:T}, H_{0:T})\, dS_{0:T}\, dH_{0:T}.
\end{align}
For brevity, we write $\Pr(\mathcal{O}_T = 1) \equiv \Pr(\mathcal{O}_T)$, since the case $\mathcal{O}_T = 0$ is not of direct interest.

\begin{remark}[title={Remark. Meaning of the Optimality Variable $\mathcal{O}_T$.}]
    The variable $\mathcal{O}_T$ serves as a soft optimality indicator at the trajectory level, and is best understood as a formal device rather than an observed quantity.
    Rather than directly maximizing expected return, we maximize the evidence $p_{\pi_\theta}(\mathcal{O}_T = 1)$, treating $\theta$ as a parameter of the generative model, an empirical Bayes perspective in which the policy is optimized by maximizing the marginal likelihood of optimality.
    The actual values of $\mathcal{O}_T$ are never used directly; at optimality, $\pi_\theta$ would generate only high-return trajectories, so the outcome vector collapses to $[1,1,\ldots,1]$, while a suboptimal policy produces a mixture of zeros and ones reflecting the proportion of good trajectories.
\end{remark}

\paragraph{KL-Regularized Policy Gradient.}
To derive optimization objectives, let $q_\phi(H_{0:T}) \in \Pi$ be a variational policy and lower-bound the log-evidence via Jensen's inequality,
\begin{align}
    \ln p_{\pi_\theta}(\mathcal{O}_T)
    &= \ln \int p(\mathcal{O}_T \mid H_{0:T})\, p_{\pi_\theta}(H_{0:T})\, dH_{0:T} \notag\\
    &\geq \mathbb{E}_{q_\phi}\!\left[\ln p(\mathcal{O}_T \mid H_{0:T})\right]
        - D_{\mathrm{KL}}\!\left(q_\phi(H_{0:T}) \,\|\, p_{\pi_\theta}(H_{0:T})\right)
    \eqdef \mathcal{L}(\theta, \phi).
\end{align}
Assuming an exponential likelihood,
\begin{align}
    p(\mathcal{O}_T \mid H_{0:T}) = \frac{\exp G(H_{0:T})}{Z(\pi_\theta)},
    \qquad
    Z(\pi_\theta) = \mathbb{E}_{p_{\pi_\theta}}\!\left[\exp G(H_{0:T})\right],
\end{align}
and noting that $Z(\pi_\theta)$ is constant with respect to the maximization over $\phi$, we obtain the KL-regularized objective and its gradient,
\begin{align}
    \max_\phi\, \mathcal{L}(\theta, \phi)
    &= \max_\phi\; \mathbb{E}_{q_\phi}\!\left[G(H_{0:T})\right]
        - D_{\mathrm{KL}}\!\left(q_\phi \,\|\, p_{\pi_\theta}\right), \\
    \nabla_\phi\, \mathcal{L}(\theta, \phi)
    &= \mathbb{E}_{q_\phi}\!\left[G(H_{0:T})\, \nabla_\phi \ln q_\phi\right]
        - \nabla_\phi D_{\mathrm{KL}}\!\left(q_\phi \,\|\, p_{\pi_\theta}\right),
\end{align}
which recovers a policy-gradient objective regularized by the KL divergence between the variational policy and the model. 
A locally optimal policy can then be recovered through coordinate ascent by alternating until convergence:
\begin{align}
    \theta^* = \argmax_\theta\, \mathcal{L}(\theta, \phi^*),
    \qquad
    \phi^* = \argmax_\phi\, \mathcal{L}(\theta^*, \phi).
\end{align}
In our setting this can be implemented simply by setting $\theta^* \leftarrow \phi^*$, since both parameterize policies in $\Pi$. 
Under the assumptions that $\mathcal{L}$ is bounded above, each coordinate update is exact, and $G$ is bounded (ensuring $Z(\pi_\theta) < \infty$), each step does not decrease $\mathcal{L}$, usually converging to a local optimum.

\paragraph{Gibbs identity.}
Maximizing the ELBO is equivalent to minimizing the KL divergence to the tilted posterior $q^*$, by the Gibbs identity
\begin{align}
    \mathbb{E}_{q_\phi}[G(H_{0:T})] - D_{\mathrm{KL}}\!\left(q_\phi \,\|\, p_{\pi_\theta}\right)
    &= \ln Z(\pi_\theta) - D_{\mathrm{KL}}\!\left(q_\phi \,\|\, q^*\right),
    \\
    q^*(H_{0:T}) &\eqdef \frac{p_{\pi_\theta}(H_{0:T}) \exp G(H_{0:T})}{Z(\pi_\theta)},
\end{align}
so the evidence bound is tight exactly at $q^*$, provided the policy class contains it.
Maximizing over $\phi$ therefore projects $q^*$ onto the policy class, from which we can observe an alternative optimization strategy.
Draws from $q^*$ are simply draws from $p_{\pi_\theta}$ reweighted, or resampled, by the potential $\exp G$.
The projection then becomes a cross-entropy fit on these tilted samples, which never needs a gradient of the potential.
This duality between the policy gradient and the cross-entropy fit is what motivates FTW's sample-and-project approach.

\subsection{Batch-Level CAI and Ordinal Potentials}
\label{ap:batch_cai}

We now extend the trajectory-level CAI of Appendix~\ref{app:cai_preliminaries} to inference over finite collections of trajectories, conditioned on ordinal feedback. 
This section develops the framework abstractly, subsequent subsections instantiate it with the Plackett--Luce model and multinomial model.

\paragraph{Batch-level ELBO (Eq.~\ref{eq:ell_lowerbound_updated}).}

We modify the base graphical model marginalizing not over the single trajectory $H_T$ but a batch $\mathcal{B}_N = (H_T^1, \dots, H_T^N)$ and the binary $\mathcal{O}_T \in \{0, 1\}$ with an ordinal observation $\mathcal{O}_T \in \mathbb{Z}_+$ or $\mathcal{O}_T \in \mathbb{S}_N$ where $\mathbb{S}_N$ is the symmetric group of size $N$ (overloading notation) supported on rankings or winner indices, the evidence becomes $\ln p_{\pi_\theta}(\mathcal{O}_T) = \ln \mathbb{E}_{p_{\pi_\theta}(\mathcal{B}_N)}[\Phi_T(\mathcal{B}_N)]$, where $p_{\pi_\theta}(\mathcal{B}_N) = \prod_{n=1}^N p_{\pi_\theta}(H_T^n)$ by the i.i.d. batch construction (Section~\ref{section:ftw}). 
Introducing a variational $q \in \Pi$ at the trajectory level and lifting it to the batch as $p_q(\mathcal{B}_N) = \prod_n p_q(H_T^n)$, Jensen's inequality gives
\begin{equation}
    \ln p_{\pi_\theta}(\mathcal{O}_T) \geq \mathbb{E}_{p_q(\mathcal{B}_N)}[\ln \Phi_T(\mathcal{B}_N)] - D_{\mathrm{KL}}\!\left(p_q(\mathcal{B}_N) \,\|\, p_{\pi_\theta}(\mathcal{B}_N)\right).
\label{eq:batch-elbo-full}
\end{equation}
The KL between product measures factorizes as $D_{\mathrm{KL}}(\prod_n p_q \,\|\, \prod_n p_{\pi_\theta}) = N \cdot D_{\mathrm{KL}}(p_q \,\|\, p_{\pi_\theta})$, recovering Eq.~\ref{eq:ell_lowerbound_updated} from the main text. 
The factor $N$ is not a tunable weight but a direct consequence of the i.i.d. batch construction, however, it can be absorbed into the KL coefficient $\lambda$ in Algorithm~\ref{alg:ftw} which is equivalent to adjusting the temperature of the likelihood $\ln \Phi_T(\mathcal{B}_N)$.

\paragraph{Ordinal feedback.}
We consider two cases for ordinal feedback on a batch $\mathcal{B}_N = \{H_{0:T}^{(i)}\}_{i=1}^N$:
\begin{enumerate}
    \item \textbf{Full ranking.} $\mathcal{O}_T = \sigma \in \mathbb{S}_N$ specifies a complete permutation of $\mathcal{B}_N$. 
    The likelihood is the Plackett--Luce (PL) model (\S\ref{app:pl-model}), and methods such as DPO and LiPO operate in this regime.
    \item \textbf{Winner-only.} $\mathcal{O}_T = n \in \{1, \dots, N\} \subset \mathbb{Z}_+$ identifies a single preferred trajectory. 
    The likelihood is the multinomial (MNL) model (\S\ref{app:mnl-model}), which arises as the first-stage marginal of PL, and methods such as GRPO operate in this regime.
\end{enumerate}
Both yield well-defined ordinal CAI problems. 
The substantive question is whether ordinal feedback elicited on a sampled batch $\mathcal{B}_N$ is \emph{globally consistent}.
In other words, whether the preference structure on $\mathcal{B}_N$ agrees with an underlying total order on $\mathcal{H}$. This is established by Luce's choice axiom.

\paragraph{Consistency via Luce's Choice Axiom.}
Both the full-ranking PL likelihood and the winner-only MNL marginal inherit independence of irrelevant alternatives (IIA) from the PL family. 
For any two trajectories $a, b \in \mathcal{B}_N$, define the pairwise preference probability
\begin{equation}
p(a \succ b \mid \mathcal{B}_N) \;\equiv\; p(\mathcal{O}_T = a \mid \mathcal{O}_T \in \{a, b\}, \mathcal{B}_N).
\end{equation}

\begin{proposition}[Luce's choice axiom \citep{luce1959}]
\label{prop:luce}
Under the PL model with positive weights $g$, pairwise preferences are independent of the enclosing batch $\mathcal{B}_N$: for any $a, b \in \mathcal{B}_N$,
\begin{equation}
p(a \succ b \mid \mathcal{B}_N) \;=\; \frac{g(a)}{g(a) + g(b)}.
\end{equation}
\end{proposition}

\begin{proof}
By Bayes' rule, and using that $\{\mathcal{O}_T = a\}$ and $\{\mathcal{O}_T = b\}$ are mutually exclusive,
\begin{align}
p(\mathcal{O}_T = a \mid \mathcal{O}_T \in \{a, b\}, \mathcal{B}_N)
&= \frac{p(\mathcal{O}_T = a \mid \mathcal{B}_N)}{p(\mathcal{O}_T = a \mid \mathcal{B}_N) + p(\mathcal{O}_T = b \mid \mathcal{B}_N)} \\
&= \frac{g(a) / Z}{g(a)/Z + g(b)/Z}
\;=\; \frac{g(a)}{g(a) + g(b)},
\end{align}
where $Z = \sum_{i=1}^N g(H_{0:T}^{(i)})$ cancels.
\end{proof}

Since the right-hand side is independent of $\mathcal{B}_N$, ordinal feedback elicited on any batch is a consistent marginal of an underlying total ordering on $\mathcal{H}$. This has two operational implications.

\emph{For the PL likelihood (full ranking)}, IIA guarantees that the relative probability of any two rankings $\sigma, \sigma'$ that agree on their top-$k$ elements is independent of the remaining elements.
Thus, subset rankings are consistent permutations of the full-domain ordering, so optimizing on any batch yields a gradient aligned with the full-domain objective.

\emph{For the MNL likelihood (winner-only)}, IIA guarantees that the winner identity is stable regardless of which other trajectories appear in $\mathcal{B}_N$: if $H_{0:T}^{(n)}$ is the best trajectory globally, it wins in any batch containing it. The winner-only model is thus a principled marginal of the full PL model.

Both guarantees mean that Monte Carlo approximation over sampled batches $\mathcal{B}_N^{(1)}, \dots, \mathcal{B}_N^{(M)} \sim p_{\pi_\theta}$ is well-posed for either likelihood choice. 
This justifies the per-minibatch filtering operation of our algorithm FTW (Section~\ref{section:ftw}) since per-batch rankings under either potential are consistent estimators of the full-domain ranking, IIA ensures the elite set is unbiased in expectation.

\subsection{The Plackett--Luce Model}
\label{app:pl-model}
The Plackett--Luce (PL) model defines a probability distribution over rankings of a finite set of alternatives. Given a batch $\mathcal{B}_N = \{H_{0:T}^{(i)}\}_{i=1}^N$, each trajectory is assigned a positive weight $g(H_{0:T}^{(i)}) > 0$ through a weight function $g: \mathcal{H} \to \mathbb{R}_{>0}$ (for instance, $g(H_T^{(i)}) = \exp G(H^i_T)$). 
The probability of observing a complete ranking $\sigma = (\sigma(1), \sigma(2), \dots, \sigma(N))$ is given by sequential elimination,
\begin{equation}
    p(\mathcal{O}_T = \sigma \mid \mathcal{B}_N) \;=\; \prod_{n=1}^N \frac{g\!\left(H_{0:T}^{(\sigma(n))}\right)}{\sum_{k=n}^N g\!\left(H_{0:T}^{(\sigma(k))}\right)}.
    \label{eq:pl-likelihood}
\end{equation}
Here $\sigma(1)$ denotes the top-ranked (winner) trajectory and $\sigma(N)$ the bottom-ranked.
The first-stage marginal, which is the probability that trajectory $n$ is ranked first, follows by summing over all rankings consistent with $\sigma(1) = n$, yielding the familiar softmax form
\begin{equation}
    p(\mathcal{O}_T = n \mid \mathcal{B}_N) \;=\; \frac{g\!\left(H_{0:T}^{(n)}\right)}{\sum_{i=1}^N g\!\left(H_{0:T}^{(i)}\right)}.
    \label{eq:pl-marginal}
\end{equation}
If the generating policy is already optimal: $p(\mathcal{O}_T = n \mid \mathcal{B}_N) = 1/N$, and the ranking is uniform. 
The weight function $g$ is a modeling choice that determines the inductive bias of the likelihood; independent of this choice, the PL model satisfies IIA (Proposition~\ref{prop:luce}), so pairwise preferences are consistent across batches.

\subsection{The Multinomial (Winner) Model}
\label{app:mnl-model}
The winner-only model takes $p(\mathcal{O}_T = n \mid \mathcal{B}_N)$ as the probability that $H_{0:T}^{(n)}$ beats all others in $\mathcal{B}_N$,
\begin{equation}
    p(\mathcal{O}_T = n \mid \mathcal{B}_N) \;\equiv\; p\!\left(H_{0:T}^{(n)} \succ H_{0:T}^{(k)},\ \forall k \neq n \,\middle|\, \mathcal{B}_N\right) \;=\; \frac{g^{(n)}}{\sum_{j=1}^N g^{(j)}},
    \label{eq:mnl-likelihood}
\end{equation}
which is exactly the first-stage PL marginal (Eq.~\ref{eq:pl-marginal}). 
The first-stage marginal of the full PL ranking model coincides with the standalone MNL choice model, so optimizing on winner-only feedback is principled and inherits the global consistency guarantees of PL.

\paragraph{Exchangeability.}

Since the batch construction $p_{\pi_\theta}(\mathcal{B}_N) = \prod_n p_{\pi_\theta}(H_T^n)$ is permutation-invariant, the batch indices $\{1, \dots, N\}$ are a priori exchangeable. Conditioning on a specific winner index $n$ is therefore arbitrary, and we instead work with the average over all $N$ possible winner identities,
\begin{equation}
    \ell(\mathcal{O}_T \mid \mathcal{B}_N) \;\triangleq\; \frac{1}{N}\sum_{n=1}^N p(\mathcal{O}_T = n \mid \mathcal{B}_N),
    \label{eq:mnl-averaged-likelihood}
\end{equation}
which substitutes into the ELBO (Eq.~\ref{eq:batch-elbo-full}) to give the averaged form,
\begin{equation}
    \widehat{\mathcal{L}}(\theta, \phi \mid \mathcal{B}_N) \;=\; \frac{1}{N}\sum_{n=1}^N \mathbb{E}_{p_q(\mathcal{B}_N)}\!\left[\ln p(\mathcal{O}_T = n \mid \mathcal{B}_N)\right] - N \cdot D_{\mathrm{KL}}(p_q \,\|\, p_{\pi_\theta}).
    \label{eq:mnl-averaged-elbo}
\end{equation}
This is the objective from which the per-trajectory MNL policy gradient is derived in \S\ref{app:pg-recovery}.

\subsection{Policy Gradients and Recovery of GRPO and DPO} \label{appendix:derivation:recovery}
\label{app:pg-recovery}

We now derive the policy gradient under each likelihood with the Boltzmann weight function $g(H_{0:T}) = \exp G(H_{0:T})$, and use these gradients to recover GRPO and DPO as special cases.
The discussion below derives the gradient directly from the per-trajectory log-likelihood.
Throughout, we adopt the convention that $\sigma(1)$ denotes the top-ranked (winner) trajectory and $\sigma(N)$ the bottom-ranked.

\paragraph{Policy gradient under the full PL model.}
Plugging the Boltzmann weight into the log-likelihood of a complete ranking $\sigma$,
\begin{equation}
    \ln p(\mathcal{O}_T = \sigma \mid \mathcal{B}_N) \;=\; \sum_{n=1}^N \left[ G^{(\sigma(n))} - C_n(\mathcal{B}_N) \right], \qquad C_n(\mathcal{B}_N) \;\triangleq\; \ln \sum_{k=n}^N \exp G^{(\sigma(k))}.
\end{equation}
Substituting into the ELBO~(\ref{eq:batch-elbo-full}) and applying the log-derivative trick,
\begin{equation}
    \nabla_\phi \mathcal{L}(\theta, \phi \mid \mathcal{B}_N) = \mathbb{E}_{p_q(\mathcal{B}_N)}\!\left[\sum_{n=1}^N \left(G^{(\sigma(n))} - C_n(\mathcal{B}_N)\right) \nabla_\phi \ln p_q(\mathcal{B}_N)\right] - N \nabla_\phi D_{\mathrm{KL}}(q_\phi \,\|\, p_{\pi_\theta}).
\end{equation}
The baselines $C_n(\mathcal{B}_N)$ shrink monotonically down the ranking, with $C_1(\mathcal{B}_N) \geq C_2(\mathcal{B}_N) \geq \cdots \geq C_N(\mathcal{B}_N) = G^{(\sigma(N))}$.
This produces a non-uniform gradient signal across rank positions.
At the bottom ($n = N$), the baseline reduces to $G^{(\sigma(N))}$ itself, so the advantage is exactly zero.
At the top ($n = 1$), the log-sum-exp is dominated by the largest return $G^{(\sigma(1))}$, so the advantage is non-zero but small in magnitude.
Middle-ranked trajectories carry the largest advantages, since their returns are smaller than the log-sum-exp but the log-sum-exp itself remains dominated by the top of the ranking.
Gradient signal is therefore diluted at both extremes and concentrated on the middle of the ranking.

\paragraph{Policy gradient under the MNL model.}
Starting from the averaged ELBO of Eq.~\ref{eq:mnl-averaged-elbo} and plugging in the Boltzmann weight, the per-winner log-likelihood is
\begin{equation}
    \ln p(\mathcal{O}_T = n \mid \mathcal{B}_N) \;=\; G\!\left(H_{0:T}^{(n)}\right) - C(\mathcal{B}_N), \qquad C(\mathcal{B}_N) \;\triangleq\; \ln \sum_{i=1}^N \exp G\!\left(H_{0:T}^{(i)}\right).
\end{equation}
Substituting and applying the log-derivative trick,
\begin{equation}
    \nabla_\phi \mathcal{L}(\theta, \phi \mid \mathcal{B}_N) = \mathbb{E}_{p_q(\mathcal{B}_N)}\!\left[\frac{1}{N}\sum_{n=1}^N \left(G\!\left(H_{0:T}^{(n)}\right) - C(\mathcal{B}_N)\right) \nabla_\phi \ln p_q(\mathcal{B}_N)\right] - N \nabla_\phi D_{\mathrm{KL}}(q_\phi \,\|\, p_{\pi_\theta}).
    \label{eq:mnl-grad}
\end{equation}
Unlike typical policy gradient methods where the baseline allows advantages of both signs, our derivation yields $C(\mathcal{B}_N) \geq G(H_{0:T}^{(n)})$ for all $n$.
All advantages are therefore non-positive, so the learning signal operates through repulsion, with the KL term anchoring $q_\phi$ to $p_{\pi_\theta}$.
In contrast to the full PL gradient, all $N$ items share a single baseline $C(\mathcal{B}_N)$, so the relative magnitudes of the per-trajectory advantages are determined entirely by the spread of returns rather than by rank position.

\paragraph{Recovery of GRPO.}
Eq.~\ref{eq:mnl-grad} is structurally close to GRPO \citep{shao2024deepseekmath} but differs in two ways.
The baseline is the log-sum-exp $C(\mathcal{B}_N)$ rather than the group mean $\mu_G = \frac{1}{N}\sum_i G^{(i)}$, and the KL coefficient carries an explicit factor of $N$.
Both differences are addressed by results already established.

\begin{itemize}
    \item \textbf{Baseline.}
    Proposition~\ref{prop:grpo} (Appendix~\ref{appendix:proofs:rao_blackwell}) shows that the per-trajectory coefficient $\alpha_n = \ln g_n - \ln \sum_k g_k$ admits a Jensen upper bound.
    The right-hand side of this bound is exactly GRPO's group-mean advantage $G(H_T^n) - \frac{1}{N}\sum_k G(H_T^k)$.
    Applying this bound to Eq.~\ref{eq:mnl-grad} replaces $C(\mathcal{B}_N)$ with $\mu_G + \ln N$ and yields GRPO's advantage.
    The Jensen slack is precisely the risk bonus profiled in Proposition~\ref{prop:risk}.
    \item \textbf{KL scaling.}
    The factor $N$ on the KL term is a consequence of the i.i.d. batch construction (\S\ref{ap:batch_cai}).
    It absorbs into the KL coefficient $\lambda$, equivalent to tuning GRPO's $\beta$.
\end{itemize}

A residual structural difference is that GRPO assumes trajectories in $\mathcal{B}_N$ share a common prompt prefix, corresponding to branched rollouts from a single state.
Our derivation treats trajectories as fully independent draws from $q_\phi$, which is strictly weaker.
The $\ln N$ offset between the log-sum-exp and group-mean baselines is independent of $\phi$ and does not affect the critical points of the optimization.

\paragraph{Recovery of DPO.}
At $N = 2$, the full PL and winner-only likelihoods coincide, since there is only one ranking consistent with trajectory 1 winning.
The shared log-likelihood reduces to a single pairwise comparison,
\begin{equation}
    \ln p(\mathcal{O}_T = \sigma \mid \mathcal{B}_N)\big|_{N=2} \;=\; \ln p(\mathcal{O}_T = n \mid \mathcal{B}_N)\big|_{N=2} \;=\; \ln \frac{\exp G^{(1)}}{\exp G^{(1)} + \exp G^{(2)}},
    \label{eq:pl-mnl-coincide}
\end{equation}
which is the Bradley--Terry log-likelihood.
Under DPO's own reparameterization $g(H_T) = \exp(s_\theta(H_T)/\beta)$ with $s_\theta = \log \pi_\theta / \pi_{\mathrm{ref}}$, the right-hand side becomes $\ln \sigma\!\left(\frac{1}{\beta}(s_\theta^{(1)} - s_\theta^{(2)})\right)$, recovering DPO's loss.
The remaining bookkeeping concerns the role of the reparameterization, the factor of 2 absorbed into the KL coefficient, and the identification of $\pi_\theta$ with $\pi_{\mathrm{ref}}$ via the EM coordinate-ascent argument of Appendix~\ref{app:cai_preliminaries}.

\paragraph{Summary.}
Our derivation provides a holistic account of two prominent RFT objectives.
Both GRPO and DPO optimize a KL-regularized ELBO over trajectory batches under a Boltzmann likelihood.
They differ in batch size ($N = 2$ vs. $N > 2$) and choice of baseline (group-mean Jensen surrogate vs. exact pairwise normalizer).
Within this framework, the apparent diversity of recent preference-optimization methods reduces to modeling choices on $\ell$, $g$, $\beta$, and the projection step, rather than to fundamentally different algorithmic principles.
We do not claim a complete reduction of all such methods here.
The framework is permissive enough to accommodate them, but the substantive recovery results are GRPO and DPO.

\section{Belief-Policy Mirror Descent under Batch CAI}
\label{appendix:belief}

This section describes our graphical model extension for the contextual-bandit study (see also Appendix~\ref{appendix:experimental_setup:environments}).
We extend our control-as-inference formulation for batch potentials to deal with uncertainty given past data, and use our FTW approach to perform mirror descent in policy-space with a Gaussian belief.
This structure enables closed form policy iteration, where each mirror-descent step extracts maximum value from past data given what the learning operator allows.

\paragraph{Setup.}
We use an empirical Bayes type approach to fit a batch version of the variational distribution $q(a, W)$ to approximate our batch-CAI posterior $p(a_t | H_{1:t-1}, \mathcal{O}_{t:T})$ from Eq.~\eqref{eq:ell_lowerbound_updated}.
We write the policy as a linear softmax over bias-augmented features\footnote{The features are defined in Appendix~\ref{appendix:experimental_setup:environments} as part of the environment description.} $\bar\phi(x) = (1, \tilde\phi(x)) \in \mathbb{R}^{d+1}$ that depend on some context $x \in \mathcal{X}$, that is $\pi_W(a \mid x) \propto \exp\bigl((W\bar\phi(x))_a/\tau\bigr)$ with $W \in \mathbb{R}^{|\setA| \times (d+1)}$ and temperature $\tau$.
Then, $\bar q$ represents the posterior predictive $\bar q(a_t \mid x_t, H_{1:t-1}) = \bbE_{q(W\mid H_{1:t-1})}[\pi_W(a_t \mid x_t)]$ where $q(W | H_{1:t-1})$ is the variational belief state given a past history of interactions $H_{1:t-1} = \{(x_i, a_i, r_i)\}_{i=1}^{t-1} \cup \{x_t\}$ for some reward $r_i$ (see also Section~\ref{section:background} from the main text).
Thus, sampling from $\bar q$ works through posterior sampling where we first draw a hypothesis $W$ from the belief $q(W_t\mid H_{1:t-1})$, and then draw an action $a_t$ from $\pi_{W_t}$ given hypothesis $W_t$ and context $x_t$.
Write the batch $\setB_N = \bigl((x_1, a_1),\dots,(x_N, a_N)\bigr)$ as some sample originating from one proposal policy over multiple contexts, or from multiple actions on the same context.
This gives the batch joint $\bar q(\setB_N \mid \cdot ) = \prod_i \bar q(a_i \mid x_i, \cdot)p(x_i)$ where `$\cdot$' implies any history and $p(x_i)$ is some arbitrary context-process.
We use a simple Gaussian belief for $q(W \mid H_{1:t-1})$ that factorizes independently over the actions $q(W) \eqdef \prod_a \mathcal{N}(w_a; \mu_a, \Lambda_a^{-1})$ such that epistemic uncertainty over actions is perfectly captured but there is still generalization across contexts.
We initialize each $w_a$ at $\mu_a = 0$, $\Lambda_a = I$, which induces a uniform $\bar q$ since $\text{Law}(W_0)$ induces exchangeability across arms.

\paragraph{The evidence lower-bound splits in two penalties.}
Applying Eq.~\eqref{eq:ell_lowerbound_updated} to the joint $q(\setB_N, W) = q_W(\setB_N) q(W)$ against the reference $p(\setB_N, W) = \pi_W(\setB_N) p(W)$, and splitting up the divergence gives the conditional ELBO, 
\begin{align}
    \ln p_\pi(\mathcal{O}_{t} \mid H_{<t}) &= \ln \mathbb{E}_{p_\pi(\mathcal{B}_N, W_t \mid H_{<t})}\Phi_t(\mathcal{B}_N) \ge \mathbb{E}_{p_q(\mathcal{B}_N, W_t \mid H_{<t})}[\ln \Phi_t(\mathcal{B}_N)] - \KL(p_q \Vert p_\pi)
    \\
    \mathcal{L}(q, \pi; H_{<t}) &= \mathbb{E}_{q(W_t \mid H_{<t})} \left[\mathbb{E}_{q(\mathcal{B}_N \mid x_t, W_t)}[\ln \Phi_t(\mathcal{B}_N)] - N \cdot \KL(q_W \Vert \pi_W)\right] \nonumber
    \\
    &\hspace{10em}  - \KL(q(W_t \mid H_{<t}) \Vert p(W_t \mid H_{<t}))
    \label{eq:elbo_split}
\end{align}
The history $H_{<t}$ conditioning is important to emphasize since the contextual bandit assumes stationary transition dynamics, i.e., $x_{t+1} \indep x_t$.
This setting implies episodic control for the actions $\mathcal{B}_N$.
As a result, the optimality variables $\mathcal{O}_{t+1} \indep \mathcal{O}_t$ so long as we also assume that the hypotheses $W_t \sim q$ are also independent of $\mathcal{O}_t, \forall t$.
This assumption is reasonable in a typical Bayesian filtering setting, it's analogous to assuming that the dynamics of the world (world-model) are independent of rewards you may try to maximize. 
The resulting lower-bound $\mathcal{L}(q, \pi; H_{<t})$ then contains the typical potential and policy KL-term that tilts the prior policy $\pi$ towards the batch-potential at a fixed hypothesis $W_t$.
The addition is the belief term that accumulates evidence about which hypotheses are plausible. 
We will optimize $q$ per factor by coordinate ascent.

\paragraph{Ascent in $q(\setB_N \mid W)$.}
The joint $q(\mathcal{B}_N, W)$ splits into two factors for coordinate ascent, the policy $q_W$ sits in an inner-optimization which is what we complete first.
Assume we obtain some history $H_{<t}$ which we use to sample a hypothesis $W_t \sim q(W_t \mid H_{<t})$, then write the inner-ELBO as $\mathcal{L}(q, \pi; x_t)$.
We solve the inner-optimization and explicitly write it as a function of $W_t$,
\begin{align}
    q^*(\setB_N \mid W) &\eqdef \frac{\pi_W \Phi}{Z_N(W)},
    \qquad
    Z_N(W) \eqdef \bbE_{\pi_W}[\Phi],
    \\
    \mathcal{L}(q_W, \pi_W; x_t) &= \bbE_{q(\setB_N \mid W)}\bigl[\ln \Phi\bigr] - N \cdot \KL\bigl(q_W(a_t \mid x_t) \,\|\, \pi_W(a_t \mid x_t) \bigr)
    \\
    &=\; \bbE_{q(\setB_N \mid W)}\Bigl[\ln \frac{\pi_W \Phi}{q(\setB_N \mid W)}\Bigr]
    \;=\; \bbE_{q(\setB_N \mid W)}\Bigl[\ln \frac{Z_N(W)\, q^*}{q(\setB_N \mid W)}\Bigr]
    \\
    &=\; \ln Z_N(W) \;-\; \KL\bigl(q(\setB_N \mid W) \,\|\, q^*\bigr)
    \\
    &\implies\; \max_{q(\setB_N \mid W)} \mathcal{L}(q_W, \pi_W; x_t) \;=\; \ln Z_N(W), \quad \text{at } q(\setB_N \mid W) = q^*.
    \label{eq:gibbs_identity}
\end{align}
where the first line again uses the equivalence of Eq.~\ref{eq:posterior_policy_is_variational_optimization}.
This equivalence of ELBO maximization and analytical rewriting in terms of the posterior makes the last step seem somewhat trivial, but the important part is to write the optimization as a function of $W_t$ which we need to perform ascent for the variational belief.
This function is equal to the policy's partition function.

The rewriting is valid whenever $0 < Z_N(W) < \infty$, which holds even for potentials that assign zero to some batch-items (e.g., the tournament potential of Eq.~\eqref{eq:probability_of_tourney_win} for some choices of $C$ and $K$).
Such items receive $q^*_{\text{non-elite}} = 0$, which is fine as long as the per batch partition gives $Z_N(W) > 0$.

\paragraph{Ascent in $q(W)$.} 
With the E-step solved, Eq.~\eqref{eq:elbo_split} reduces to a belief-optimization: maximize $\bbE_{q(W_t \mid H_{<t})}[\ln Z_N(W_t)] - \KL(q(W_t \mid H_{<t}) \,\|\, p(W_t \mid H_{<t}))$ over the belief.
This maximization recovers the underlying objective,
\begin{align}
    &\bbE_{q(W_t \mid H_{<t})}\bigl[\ln Z_N(W_t)\bigr] - \KL\bigl(q(W_t \mid H_{<t}) \,\|\, p(W_t \mid H_{<t})\bigr) \nonumber
    \\
    &\hspace{6em} \;=\; \ln p_\pi(\setO_t \mid H_{<t}) \;-\; \KL\bigl(q(W_t \mid H_{<t}) \,\|\, p(W_t \mid \setO_t, H_{<t})\bigr).
    \label{eq:elbo_gap}
\end{align}
In other words, the belief-optimization minimizes divergence to the tilted Bayesian posterior.
Therefore, if we look at the role of the posterior $p(W_t \mid \setO_t, H_{<t})$ we can conclude that coordinate ascent in the belief $q(W_t)$ amounts to variational Bayesian filtering.
To make this optimization tractable, we start with a simple uninformative prior $q(W_0)=p(W_0)$ and perform iterative coordinate ascent (expectation-maximization).
We iterate the ascent step for the belief $q_{t+1}$ by setting the reference belief as the previous belief $q_t$ and writing the loss as $\ell_t(W) \eqdef -\ln Z_{N,t}(W)$.
This turns each ascent step into a proximal optimization step,
\begin{equation}
    q_{t+1} \;=\; \argmax_{q}\; \Bigl\{ \bbE_{q}\bigl[\ln Z_{N,t}\bigr] - \KL(q \,\|\, q_t) \Bigr\}
    \;=\; \argmin_{q}\; \Bigl\{ \bbE_{q}\bigl[\ell_t\bigr] + \KL(q \,\|\, q_t) \Bigr\} .
    \label{eq:belief_prox}
\end{equation}
stated equivalently as a minimization to match mirror-descent framing (which we need later for the implementation).
When $q$ is unrestricted, Eq.~\eqref{eq:belief_prox} has a closed-form solution, $q_{t+1} \propto q_t \, Z_{N,t}$, which uses Bayes' rule to update the previous belief with the likelihood of $\setO_t$.
This unrestricted closed-form solution is what motivates the variational coordinate-ascent update rule.

\paragraph{Solving Eq.~\eqref{eq:belief_prox} for a Gaussian.}
We will instead use a Gaussian belief for the variational posterior $q(W_t \mid H_{<t}) = \prod_a \mathcal{N}(w_a; \mu_{a,t}, \Lambda^{-1}_{a,t})$ as in the Setup paragraph, for this we first write the Gaussian in terms of its exponential family form with sufficient statistics $T(w) = (w, ww^\top)$.
Expanding the log-density shows what this form is,
\begin{equation}
    \ln \mathcal{N}(w; \mu, \Lambda^{-1})
    = w^\top \Lambda \mu \;-\; \tfrac12 w^\top \Lambda w \;+\; \mathrm{const}
    = \bigl\langle \Lambda\mu,\, w \bigr\rangle + \bigl\langle -\tfrac12\Lambda,\, ww^\top \bigr\rangle + \mathrm{const},
    \label{eq:expfam_expand}
\end{equation}
a linear function of the two statistics.
The coefficients multiplying the statistics $w$ are called the natural parameters $\nu$, and the means of $w$ under $q$ are called the expectation parameters $m$,
\begin{equation}
    \nu = \bigl(\Lambda\mu,\; -\tfrac{1}{2}\Lambda\bigr),
    \qquad
    m = \bigl(\bbE_q[w],\, \bbE_q[ww^\top]\bigr) = \bigl(\mu,\; \Lambda^{-1} + \mu\mu^\top\bigr),
    \label{eq:coords}
\end{equation}
Either pair identifies the same Gaussian and their normalizer links them.
If we write $A(\nu)$ for the log of the normalizer as a function of $\nu$ (i.e., the log-partition function of the Gaussian under the natural parametrization), then we can use a standard exponential-family identity to obtain $m = \nabla_\nu A(\nu)$.

This coordinate-system makes the proximal optimization step of Eq.~\eqref{eq:belief_prox} simple to perform for two reasons.
First, between two Gaussians the KL penalty equals the Bregman divergence of $A$ evaluated at their natural parameters.
Eq.~\eqref{eq:belief_prox} thus takes the standard form of a mirror-descent step, which is gradient descent with the squared Euclidean penalty to some reference replaced by the Bregman divergence, here that of $A$ \citep[Sec.~2.2]{khan_bayesian_2023} to $\nu_t$.
Second, gradients with respect to $\nu$ reduce to gradients with respect to $m$, which we can easily compute.
Since all terms in Eq.~\eqref{eq:belief_prox} depend on the belief only through $\nu$, we can differentiate the objective in $\nu$.

Completing this gives $\nabla_\nu \KL(q_\nu \,\|\, q_{\nu_t}) = \nabla^2 A(\nu)\,(\nu - \nu_t)$ for the penalty, and the chain rule through $m = \nabla A(\nu)$ gives $\nabla_\nu \bbE_{q_\nu}[\ell_t] = \nabla^2 A(\nu)\, \nabla_m \bbE_{q_\nu}[\ell_t]$ for the loss (negative log-likelihood).
Setting the sum of the two gradients to zero gives $\nabla^2_\nu A(\nu)\,\bigl(\nu - \nu_t + \nabla_m \bbE_{q_\nu}[\ell_t]\bigr) = 0$.
We can then use another exponential-family identity that gives the second-order derivative of $A$ as the covariance of the sufficient statistics, $\nabla^2_\nu A = \mathrm{Cov}_{q_\nu}[T(w)]$, which equals the Fisher information in natural coordinates.
Since a covariance is positive definite it is invertible, so multiplying the stationarity condition by $(\nabla^2_\nu A)^{-1}$ leaves $\nu - \nu_t + \nabla_m \bbE_{q_\nu}[\ell_t] = 0$, and thus,
\begin{align}
    \nu_{t+1} &= \nu_t - \nabla_m \bbE_q[\ell_t],
    \quad\text{i.e.,} \nonumber
    \\
    (\Lambda\mu)_{t+1} &= (\Lambda\mu)_t - \nabla_{m_1}\bbE_q[\ell_t],
    \qquad
    \bigl(-\tfrac12 \Lambda\bigr)_{t+1} = \bigl(-\tfrac12 \Lambda\bigr)_t - \nabla_{m_2}\bbE_q[\ell_t].
    \label{eq:belief_natgrad}
\end{align}
Since $\nabla_m \bbE_{q_\nu}[\ell_t]$ depends on $\nu$, stationarity puts $\nu_{t+1}$ on both sides (cross-dependence/ stationary-point).
The implementation instead evaluates the gradient at the previous belief $q_{\nu_t}$.

Finally, we need to express $\nabla_m \bbE_q[\ell_t]$ in terms of $(\mu, \Lambda)$ to unpack the learning rule.
Inverting Eq.~\eqref{eq:coords} gives $\mu = m_1$ and $\Lambda^{-1} = m_2 - m_1m_1^\top$, so the chain rule yields $\nabla_{m_1}\bbE_q[\ell_t] = \nabla_\mu \bbE_q[\ell_t] - 2(\nabla_{\Lambda^{-1}} \bbE_q[\ell_t])\,\mu$ and $\nabla_{m_2}\bbE_q[\ell_t] = \nabla_{\Lambda^{-1}} \bbE_q[\ell_t]$.
Substituting into Eq.~\eqref{eq:belief_natgrad}, the second component gives $\Lambda_{t+1} = \Lambda_t + 2\nabla_{\Lambda^{-1}} \bbE_q[\ell_t]$.
The first component gives $(\Lambda\mu)_{t+1} = \Lambda_t\mu_t - \nabla_\mu \bbE_q[\ell_t] + 2(\nabla_{\Lambda^{-1}} \bbE_q[\ell_t])\,\mu_t = \Lambda_{t+1}\mu_t - \nabla_\mu \bbE_q[\ell_t]$, where the last equality collects the two terms in $\mu_t$ into the new precision.
Then, the theorems of \citet{bonnet_transformations_1964} and \citet{price_useful_1958} convert the parameter gradients into expectations of the derivatives of the loss with respect to the arm's row $w$, $\nabla_\mu \bbE_q[\ell_t] = \bbE_q[\nabla_w\ell_t]$ and $\nabla_{\Lambda^{-1}} \bbE_q[\ell_t] = \tfrac12\bbE_q[\nabla^2_w\ell_t]$, which yields
\begin{equation}
    \boxed{\;
    \Lambda_{a,t+1} = \Lambda_{a,t} + \bbE_{q}\bigl[\nabla^2_{w_a} \ell_t\bigr],
    \qquad
    \mu_{a,t+1} = \mu_{a,t} - \Lambda_{a,t+1}^{-1}\,\bbE_{q}\bigl[\nabla_{w_a} \ell_t\bigr],
    \;}
    \quad \forall a \in \mathcal{A},
    \label{eq:belief_update}
\end{equation}
the Bayesian learning rule of \citet[Sec.~1.3.2]{khan_bayesian_2023} derived through our bound, matching their online Newton update with unit learning rate and, as in their conjugate-Bayes setting (their Sec.~5.1), the previous belief as the prior.
The increments to the precision are positive semi-definite (Eq.~\eqref{eq:ce_hessian}), so the precision only accumulates and the effective step $\Lambda_{t+1}^{-1}$ anneals as evidence arrives.

Note that the natural-parameter derivation placed the updated precision $\Lambda_{t+1}$ in front of the mean update by itself, so the mean step accounts for the most recent evidence.
This would not naturally occur if we derived the update through $(\mu, \Lambda)$ directly.
It is also the correct choice, since with $\Lambda_{t+1}$ the update reproduces the exact Gaussian posterior whenever the loss is quadratic in $W$, and with $\Lambda_t$ it does not.
The two derivatives of $\ell_t$ we need for computing Eq.~\eqref{eq:belief_update} are derived next.

\paragraph{Cross-entropy loss.}
Differentiating the log-partition with respect to $W$, we obtain,
\begin{equation}
    \nabla_W \ell_t
    = -\,\nabla_W \ln Z_{N,t}(W)
    = -\,\frac{\bbE_{\pi_W}\bigl[\Phi\, \nabla_W \ln \pi_W(\setB_N)\bigr]}{Z_{N,t}(W)}
    = -\,\bbE_{q^*}\bigl[\nabla_W \ln \pi_W(\setB_N)\bigr],
    \label{eq:logZ_grad}
\end{equation}
with $\ln \pi_W(\setB_N) = \sum_i \ln \pi_W(a_i \mid x_i)$ and using the definitions of $q^*$ and $Z_N(W)$ from Eq.~\eqref{eq:gibbs_identity}.
This is the gradient of the batch cross-entropy $-\bbE_{q^*}[\ln \pi_W(\setB_N)]$ against the fixed target $q^*$.
The tilt $\Phi$ also doesn't need numerical evaluation and sampling is sufficient.
The per arm gradients $w_a$ split up from $W$. 

Then we derive the Hessian needed for the precision update of Eq.~\eqref{eq:belief_update},
\begin{equation}
    \nabla^2_W \ell_t
    = \bbE_{q^*}\bigl[\nabla^2_W \bigl(-\ln \pi_W(\setB_N)\bigr)\bigr]
    \;-\; \bbE_{q^*}\bigl[(\nabla_W \ln q^*)\, (\nabla_W \ln \pi_W(\setB_N))^\top\bigr].
    \label{eq:exact_hessian}
\end{equation}
We use the generalized Gauss-Newton (GGN) approximation by keeping the first term only and dropping the second, so we essentially drop the derivative of $q^*$.
This enforces that the second-order term is positive semi-definite, which cannot be guaranteed with the dropped term included.
In other words, we differentiate as if $q^*$ were fixed even though it technically varies with $W$.

The GGN term is available per arm in closed form for the softmax policy.
Writing $-\ln \pi_W(a_i \mid x_i) = -z_{i,a_i} + \mathrm{logsumexp}(z_i)$ with logits $z_i = W\bar\phi(x_i)/\tau$ linear in $W$, the chain rule gives,
\begin{equation}
    \nabla^2_{w_a w_b}\bigl(-\ln \pi_W(a_i \mid x_i)\bigr)
    \;=\; \bigl(\delta_{ab}\, \pi_W(a \mid x_i) - \pi_W(a \mid x_i)\, \pi_W(b \mid x_i)\bigr)\, \bar\phi(x_i)\bar\phi(x_i)^\top / \tau^2,
    \label{eq:hessian_blocks}
\end{equation}
independent of the sampled action $a_i$.
The factorized belief holds one precision per arm, so Eq.~\eqref{eq:belief_update} uses only the $a = b$ blocks, summed over the batch,
\begin{equation}
    \nabla^2_{w_a} \ell_t \;=\; \textstyle\sum_{i=1}^{N}\, \pi_W(a \mid x_i)\bigl(1 - \pi_W(a \mid x_i)\bigr)\, \bar\phi(x_i)\bar\phi(x_i)^\top / \tau^2,
    \label{eq:ce_hessian}
\end{equation}
which with non-negative coefficients guarantees precision increments of Eq.~\eqref{eq:belief_update} that are positive semi-definite.

\paragraph{FTW's variational optimization cannot condition on hypotheses.}
The derivation above solves the inner optimization for $q^*$ separately at each hypothesis $W_t$, which is what makes $\ln Z_N(W)$ a function of $W$.
Our implementation for FTW however uses a replay buffer with $(x, a, r)$ tuples whose generating hypotheses are implicitly marginalized out, and tournament selection ranks by observed reward.
The tilt the implementation samples is therefore
\begin{equation}
    \tilde q^*(\setB_N) \;\propto\; \Phi(\setB_N)\; \bbE_{q(W)}\bigl[\pi_W(\setB_N)\bigr],
    \label{eq:mixture_tilt}
\end{equation}
the potential applied to the belief-mixture of policies: each step pushes actions rolled out from a fresh posterior draw $W \sim q_t$ into the buffer, so uniform draws from the buffer sample the mixture over recent beliefs, and keeping tournament winners tilts those samples by $\Phi$ without evaluating it.
This tilt coincides with the pointwise $q^*(\cdot \mid W)$ as the belief concentrates, and the update of Eq.~\eqref{eq:belief_update} is unchanged.
The restriction only affects the variational gradient on the policy and the belief step itself is still solved exactly as in Eq.~\eqref{eq:belief_update}, just with a different tilt.
It also strengthens the Gauss-Newton choice in Eq.~\eqref{eq:exact_hessian}, since $\tilde q^*$ carries no $W$ at all, the GGN is the exact Hessian.

\paragraph{How GRPO, DPO, and FTW instantiate the Hessian.}
The three algorithms of main consideration in this manuscript correspond to different choices of the potential $\Phi$ through our theoretical framework.
Within this construction, only GRPO cannot be interpreted as a probabilistic graphical model, due to the Jensen upper bound of Proposition~\ref{prop:grpo}.
This paragraph briefly summarizes how each potential produces its Hessian, and what decides it is not the potential itself but the route by which it enters the belief update: through the sampling of $q^*$ alone, through a closed-form loss that itself depends on $W$, or through a bound that admits no likelihood.
Throughout we apply the Gauss-Newton approximation, which guarantees positive semi-definite precision increments.

\emph{Filtered cross-entropy.} 
The potential $\Phi$ is the elite-filtering of Eq.~\eqref{eq:probability_of_tourney_win}, and FTW follows the derivation above literally: the potential enters only through the sampling of $q^*$.
The loss is therefore the cross-entropy against $q^*$ regardless of which $\Phi$ is in force, its exact Hessian is Eq.~\eqref{eq:ce_hessian}, and the Gauss-Newton step is FTW's only approximation.

\emph{Bradley-Terry.} 
The potential compares one $|\setB_N| = 2$ pair through the Plackett-Luce potential of Corollary~\ref{cor:dpo}, but DPO does not take the route above: its potential is available in closed form and depends on $W$ itself, so DPO descends it directly rather than sampling $q^*$.
That loss is a log-sigmoid of the delta, $-\ln \sigma(\beta\Delta)$ with $\Delta = \ln \pi_W(a_w \mid x_w) - \ln \pi_W(a_l \mid x_l)$ the log-probability gap between winner and loser (the uniform reference policy cancels).
Since $\Delta$ is nonlinear in $W$, its exact Hessian is indefinite, and the Gauss-Newton step keeps only the positive semi-definite part, making the belief the posterior of the Bradley-Terry model linearized about the current mean \citep{immer2021improving}.

\emph{Centred advantage.} 
Proposition~\ref{prop:grpo} explains why this case is different.
The exact Rao-Blackwellized coefficients $\alpha_n = \ln g_n - \ln\sum_k g_k$ are log-probabilities, hence $\alpha_n \le 0$, and the resulting loss is a non-negative combination of cross-entropies with a positive semi-definite Hessian.
The Jensen bound of Proposition~\ref{prop:grpo} replaces them by centred advantages $A_n = G_n - \tfrac{1}{N}\sum_k G_k$ of mixed sign, and as a consequence the Hessian directly becomes indefinite even for the GGN case.
Thus, we hold $\Lambda \equiv \Lambda_0$ and update only the mean, $\mu_{t+1} = \mu_t - \eta \Lambda_0^{-1}\bbE_q[\nabla_W\ell_t]$.
This mean-only update with a swept step $\eta$ is plain GRPO, and we additionally implement two controls.

These two controls are: GRPO-NPG which replaces the typical gradient-descent step by a natural-gradient preconditioner from the current batch, $\mu_{t+1} = \mu_t - (F_t + \Lambda_0)^{-1} \bbE_q[\nabla_W \ell_t]$ with $F_t$ the per-arm softmax Fisher of Eq.~\eqref{eq:ce_hessian}.
Second, GRPO-ONG accumulates this same Fisher into $\Lambda$ exactly as Eq.~\eqref{eq:belief_update} would, i.e., we heuristically insert the GRPO loss into FTW's belief learning rule.

\section{Experimental Setup}  
\label{appendix:experimental_setup}

\subsection{Hardware}
\label{appendix:experimental_setup:hardware}

Experiments were conducted across two cloud providers.
Bandit, Brax, and Search-R1 experiments ran on Google Cloud Platform (GCP) using \texttt{a2-highgpu-1g} (1$\times$ NVIDIA A100-40GB, 12 vCPUs, 85\,GB RAM, Intel Xeon Cascade Lake), \texttt{a2-ultragpu-1g} (1$\times$ NVIDIA A100-80GB, 12 vCPUs, 170\,GB RAM, Intel Xeon Cascade Lake), and \texttt{a3-highgpu-1g} (1$\times$ NVIDIA H100-80GB, 26 vCPUs, 234\,GB RAM, Intel Xeon Sapphire Rapids), each with 500\,GB SSD persistent disk.
The 40GB A100s were sufficient only to host the Search-R1 retrieval server (our own simplified faiss-index implementation to sidestep a bug in the base implementation that caused VRAM overheads), all other agents ran on the 80GB A100s or the H100.
Sokoban experiments required a larger vision-language model and were run on Lambda Cloud using a 2$\times$ NVIDIA H100-80GB instance (52 vCPUs, 442\,GB RAM, Intel Xeon Platinum 8480+, 5.5\,TB SSD).
The Search-R1 seed repetitions ran on a Lambda 8$\times$ NVIDIA A100-80GB instance, one run per GPU with a shared retrieval server.

Reproducing the experiments in this paper requires roughly 6{,}000 GPU-hours, of which Sokoban (4 algorithms, 5 seeds, about 3 days per run on 1 to 4 GPUs) accounts for about 3{,}400, Brax (3 algorithms, 10 seeds, about 2 days per run) for about 1{,}400, and Search-R1 for the remainder, corresponding to an estimated operating cost of \$25{,}000 to \$35{,}000 (USD) at on-demand rates.
This excludes setup, algorithm development, and debugging, which vary substantially with practitioner experience and were not separately tracked.

\subsection{Compute and Memory Footprint}
\label{appendix:experimental_setup:compute}

We briefly compare the computational footprint of FTW against GRPO, PPO, and REINFORCE++ on three items, data generation, the learning update, and memory.
For data generation, all four methods consume the same number of rollouts per step in our Sokoban runs, but GRPO must obtain them as $n$ rollouts of the same prompt, which requires resetting the environment to identical initial states, whereas FTW, PPO, and REINFORCE++ draw one rollout per prompt.
Prefix caching can offset part of the group overhead, but the duplication requirement remains.
The learning updates differ negligibly in cost apart from the critic.
FTW performs a cross-entropy loss on the filtered samples plus a KL penalty, GRPO and REINFORCE++ add PPO-style clipping and advantage normalization, and PPO additionally trains a second model of the same size as the policy, which incurs a further forward and backward pass per update.
In our Sokoban runs this meant 4 GPUs for PPO against 2 for GRPO and REINFORCE++ and 1 to 2 for FTW.
In terms of storage, FTW is the most expensive because of the replay buffer of past rollouts, 2{,}000 trajectories on Sokoban and 10{,}000 on Search-R1, held in CPU memory, whereas the other three methods are on-policy.
The overhead depends on the data modality, text environments such as Search-R1 have a far smaller footprint per trajectory than image-based ones such as Sokoban.
GRPO extended with a replay buffer, as done by \citet{arnal2026efficientrltrainingllms}, would match FTW's memory footprint, leaving the group duplication as the remaining difference.
REINFORCE++ is the cheapest on every axis and is also the method that lags in Figure~\ref{fig:llm_post_training}, which is what motivates paying for one of the three variance-reduction sources.

\subsection{Software Stack}
\label{appendix:experimental_setup:software}

Our experiments span four distinct environments: Bandits, Brax continuous control, Search-R1, and Sokoban.
Each environment required a separate software stack and varying setup, for three main reasons.
First, \emph{packaging conflicts}: the verl-agent fork we build on \citep{feng2025groupingroup} lags upstream \texttt{verl} and pins \texttt{torch 2.8.0+cu128}/\texttt{vLLM 0.11.0}.
Sokoban requires the legacy \texttt{gym 0.26.2} API (via \texttt{gym-sokoban}), which conflicts with the \texttt{gymnasium}-based stack used elsewhere.
Second, \emph{hardware heterogeneity across providers}: we ran on a mix of A100-40GB, A100-80GB, and H100-80GB instances (see Appendix~\ref{appendix:experimental_setup:hardware}), and VRAM size materially changed configuration.
Third, \emph{model choices dictated by baseline parity and modality}: Search-R1 uses \texttt{Qwen2.5-3B-Instruct} to match the scale reported in prior work, and Sokoban uses the vision-language \texttt{Qwen2.5-VL-3B-Instruct} to consume the RGB puzzle rendering directly.
The Brax environments were fortunately much simpler, as they only required a small MLP provided by the reference implementation of \citet{freeman2021brax}.
All cloud runs share Ubuntu 24.04.4 LTS, Python 3.12.3, NVIDIA driver 570.211.01, and set \texttt{NCCL\_NET=Socket}.
This socket setting made our runs slower, but it was required due to GCP-related issues.
Table~\ref{tab:software-stacks} lists the key package versions per environment.
For more specific details, we advise the reader to look at the verl-agent documentation \url{https://github.com/langfengQ/verl-agent} which has extensive documentation for the environment setup.

\begin{table*}[h]
    \centering
    \caption{
        Key software versions per experimental environment.
        Top section: policy model.
        Bottom section: core framework (first row) and other critical pinned packages.
        All LLM-agent runs additionally use \texttt{transformers 4.57.6} and \texttt{ray 2.54.1}. $^\dagger$\texttt{flash-attn} built from source to resolve ABI mismatch with \texttt{torch 2.8.0+cu128}.
        SearchR1 requires a separate server for the wiki retrieval, making it the most involved environment to set up.
        }
    \label{tab:software-stacks}
    \footnotesize
    \vspace{1em}
    \setlength{\tabcolsep}{4pt}
    \begin{tabular}{@{}lll@{}}
    \textbf{Bandit \& Brax} & \textbf{SearchR1} & \textbf{Sokoban} \\
    \midrule
    Logit-Table / MLP & Qwen2.5-3B-Instruct                  & Qwen2.5-VL-3B-Instruct \\
    \midrule
    JAX 0.9.2         & verl-agent 0.3.1.dev                  & verl-agent 0.3.1.dev \\
    brax 0.14.2       & torch 2.8.0+cu128                     & torch 2.8.0+cu128 \\
    mujoco-mjx 3.6.0  & vLLM 0.11.0                           & vLLM 0.11.0 \\
    flax 0.12.6       & flash-attn 2.7.4.post1$^\dagger$      & gym 0.26.2 \\
    optax 0.2.8       & faiss-gpu-cu12                        & gym-sokoban 0.0.6 \\
    \end{tabular}
\end{table*}

\subsection{Main Environments}
\label{appendix:experimental_setup:environments}

We evaluated across seven environments: four feature-isolating bandit problems (the three diagnostic bandits of Figure~\ref{fig:tabular_auc} and a linear contextual bandit), continuous control (HalfCheetah), a text-based multi-turn LLM-agent task (Search-R1), and a vision-language grid puzzle (Sokoban).

\paragraph{$\sigma$-trap bandit (Figure~\ref{fig:tabular_auc}, left).}
A two-armed Gaussian bandit with an optimal arm ($\mu = 10$, $\sigma = 0.5$) and a zero-mean distractor whose standard deviation is swept over $\sigma_{\mathrm{trap}} \in [0.5, 32768]$.
A risk-neutral learner commits to the optimal arm at every noise level, while a risk-seeking learner is eventually drawn to the distractor as its variance grows.
All methods draw $N = 8$ pulls per step for $T = 10^6$ steps with a shared learning rate of $10^{-4}$, except for the GRPO lr-adjusted version which scales the learning rate by $\sqrt{\nicefrac{\sigma_1}{\max_a \sigma_a}}$ to compensate for increased gradient noise but has to train for a $T' = 32\cdot T$ step budget to actually converge.
FTW uses $C{=}2$, $K{=}1$, $J{=}8$ per step over a FIFO buffer of max age $h$, our DPO ($\beta = 1$) version samples pairs from this same buffer mechanism.
Replay-buffered curves are averaged over $h \in \{16, 64, 256, 1024, 4096\}$.
We report the asymptotic $\pi(a^\ast)$ per noise level (10 seeds, bootstrap 95\% CIs).

\paragraph{Truncated Pareto-$\alpha$ bandit (Figure~\ref{fig:tabular_auc}, center).}
A ten-armed tabular bandit with multiplicative heavy-tailed returns $r = c_a (1 - u)^{-\nicefrac{1}{\alpha}}$, $u \sim \mathrm{Unif}$, truncated to $[1, 10^6]$ so that all moments exist, with arm locations $c_a = 1 + 0.5a$.
The tail index is swept over $\alpha \in \{0.05, \dots, 1.7\}$, where smaller $\alpha$ produces heavier upper tails and the untruncated moment-generating function diverges at $\alpha \le 1$.
We report average (bold) and sample (transparent) learning curves of $\pi(a^\ast)$, see also Appendix~\ref{appendix:results:bandit_sweep} for additional results on different reward distributions.

\paragraph{State-aliasing corridor (Figure~\ref{fig:tabular_auc}, right).}
The short-corridor task of \citet{sutton_reinforcement_2018} (Ex.~13.1): states $\{s_0, s_1, s_2\}$ with reward $-1$ per step, where the action effect in $s_1$ is reversed.
The policy is a single global softmax over the two actions, which aliases all states and makes the optimal policy stochastic, $\pi(\text{right}) = 2 - \sqrt{2}$.
We plot the L1 distance $|\pi(a) - \pi^\ast(a)|$ between the learned and optimal action distributions (represented as a 2D vector), averaged over a learning-rate population for 20k update steps (see also Appendix~\ref{appendix:results:state_aliasing} for additional results).

\paragraph{Linear contextual bandit.}
A contextual bandit with deterministic rewards and linear features that reports simple and cumulative regret to the optimal arm (all results report mean and $95\%$ bootstrap CI over 30 seeds).
We define the reward as $r(x, a) = \nicefrac{1}{2} + \nicefrac{1}{2}\,\theta_a^\top \tilde\phi(x)$ with $\|\theta_a\| = 1$ and $\|\tilde\phi(x)\| \le 1$, which bounds $r \in [0, 1]$ without clipping and fixes the marginal $\mathbb{E}_x[r(x,a)] = \nicefrac{1}{2}$.
As a result, no action is better on average and only positional information allows exploitation.
We achieve this by sampling for every task instance (i.e., one complete run for a given seed) the triple $\langle \omega, b, \theta \rangle$, such that $\phi_j(x) = \sqrt{\nicefrac{2}{d}}\cos(\omega_j x + b_j)$ with $\omega_j \sim \mathcal{N}(0, \ell^{-2})$, $b_j \sim \mathrm{Unif}[0, 2\pi)$, and lengthscale $\ell = 0.1$.
The features are zero-meaned over the context grid and divided by the largest $\ell_2$ norm $\max_{x'} \|\phi(x')\|_2$, so that $\|\tilde\phi(x)\|_2 \le 1$ everywhere with feature dimensionality $d=16$.
Contexts $x$ are drawn uniformly over a linspace of 256 points in $[0, 1]$.
Finally, we sample each arm's reward parameters uniformly from the unit-sphere: $\theta_a \sim \mathcal{N}(0, I_d)$ then l2 normalize to $\|\theta_a\| = 1$.
We keep one special ``safe'' arm with $\theta = 0$ such that the reward is $\nicefrac{1}{2}$ independent of context (safe but suboptimal).

All policy methods share the linear softmax policy $\pi(a \mid x) \propto \exp(W\bar\phi(x))_a$ over the bias-augmented features $\bar\phi(x) = (1, \tilde\phi(x))$, with $W \in \mathbb{R}^{|\mathcal{A}| \times (d+1)}$.
Appendix~\ref{appendix:belief} describes the complete parametrization and derivation of the learning rule. 
We used $|\mathcal{A}| = 20$ actions with a per-step batch-size (number of contexts) of $16$.
So, FTW and DPO draw $N = 16$ contexts and a single action, however, GRPO draws $4$ contexts per step with an on-policy group of $G = 4$ actions for the group-mean baseline.
As a baseline we use Thompson sampling \citep{thompson_likelihood_1933, chapelle_empirical_2011} via conjugate Bayesian linear regression per arm.

\paragraph{Brax continuous control \citep{freeman2021brax}.}
HalfCheetah with the Brax \texttt{spring} backend, episode length $1{,}000$, action repeat $1$, observation normalization enabled.
The policy is an MLP $[256,256]$ with ReLU activations and a \texttt{NormalTanhDistribution} head (brax \texttt{SACNetworks}).
\begin{itemize}
    \item Observations are continuous, with 18 dimensions.
    \item The action space is a bounded continuous vector in $[-1, 1]^D$, with $D = 6$.
    \item Rewards follow the standard Brax forward-velocity plus control-cost shaping (norm of the action).
    \item Episodes terminate when the agent reaches an unhealthy joint configuration.
\end{itemize}
The key metric is the evaluation episode return, computed with a fixed policy at lower temperature.

\textbf{Search-R1}~\citep{jin2025searchr1}.
Multi-turn open-domain QA with a retrieval tool, framed as a POMDP where the agent alternates between internal reasoning and tool invocation until it commits to an answer.
At each turn the prompt contains the question and the interaction history, where prior search queries and the retrieved passages are interleaved as \texttt{<search>...</search>} and \texttt{<information>...</information>} blocks.
The LLM emits chain-of-thought inside \texttt{<think>...</think>} and then chooses exactly one of two actions: a new query in \texttt{<search>...</search>} (control returns to the retriever) or a final answer in \texttt{<answer>...</answer>} (terminating the episode).
Emitting both or neither consumes a turn.
Retrieved passages are obtained by dense lookup against an E5 index of the Wikipedia-18 corpus (\texttt{wiki-18.jsonl}, 14\,GB) compressed from a 61\,GB flat index to an 815\,MB IVF-PQ index (\texttt{e5\_IVFPQ.index}) and served from a dedicated retrieval venv on port 8000, returning the top-3 passages per query.
Reward is an exact-match indicator on the final answer against the gold label, with an additional $-0.1$ penalty for malformed emissions (missing or duplicate action tags).
Train set: 169{,}615 questions from \texttt{PeterJinGo/nq\_hotpotqa\_train} (a mix of Natural Questions and HotpotQA).
Test set, on which the main-text accuracies are computed, is the 51{,}713-question eval set bundled with the same release, drawn from NQ, TriviaQA, PopQA, HotpotQA, 2WikiMultihopQA, Musique, and Bamboogle.
Training-time curves use a fixed 1{,}280-question slice of it.
Max turns: 4.
Prompt/response budgets: 4{,}096 / 512 tokens.

\textbf{Sokoban}~\citep{SchraderSokoban2018}.
Vision-language box-pushing puzzle (\texttt{gym\_sokoban}) on a 6$\times$6 grid with one box and a generator search depth of 30 (the \texttt{num\_gen\_steps} parameter controlling puzzle complexity at instantiation, not the per-episode horizon).
At each turn the environment's \texttt{rgb\_array} rendering is supplied to the policy as an image via \texttt{data.image\_key=images}, alongside the textual task description and a short history of prior actions.
The vision-language model emits chain-of-thought inside \texttt{<think>...</think>} followed by a single direction in \texttt{<action>...</action>}, with admissible set \{up, down, left, right\}.
Each direction is dispatched as the corresponding \texttt{Push} action in \texttt{gym\_sokoban}, and unparsable or illegal emissions consume a turn.
Rewards follow the \texttt{gym\_sokoban} defaults: $-0.1$ per step, $\pm 1$ on pushing the box on/off its target, and $+10$ upon solving, with an additional $-0.1$ format penalty consistent with the main text.
Train/validation split: 32 / 128 puzzles.
Episode horizon: 15 turns.
Prompt/response budgets: 1{,}024 / 512 tokens, with image tokens accounted for separately by the VLM processor and not counted against the prompt budget.

\subsection{Hyperparameters}
\label{appendix:experimental_setup:hyperparameters}

Hyperparameters are reported in two tables: Brax continuous control (Table~\ref{tab:hyperparams-brax}) and the shared LLM-agent RL configuration (Table~\ref{tab:hyperparams-shared}).
For every baseline, SAC and the Brax network architecture from \citet{freeman2021brax}, and the GRPO defaults from verl-agent \citep{feng2025groupingroup}, we retain the published values wherever possible to keep comparisons fair.

\paragraph{Brax continuous control.}
All methods share the same optimizer, replay buffer, observation normalization, and network architecture.
Algorithm-specific differences (Table~\ref{tab:hyperparams-brax}) concern the actor learning objective and the update-to-data (UTD) ratio.

\begin{table}[ht]
    \centering
    \caption{
        Brax HalfCheetah hyperparameters.
        All methods share $\mathrm{lr}{=}3{\times}10^{-4}$ (Adam), $\gamma{=}0.99$, replay size $2{\times}10^{6}$, $64$ parallel envs, batch size $256$, episode length $1000$, $10^{9}$ total env steps, warmup $8192$ transitions, and $10$ seeds.
        Any parameters not mentioned are all kept at their default values from the Brax~0.14.2 \citep{freeman2021brax} implementation of SAC \citep{haarnoja_sac_2018}.
    }
    \label{tab:hyperparams-brax}
    \renewcommand{\arraystretch}{1.4}
    \vspace{1em}
    \setlength{\tabcolsep}{4pt}
    \begin{tabular}{@{}lll@{}}
    \textbf{Parameter}           & \textbf{FTW}                & \textbf{SAC}                                         \\
    \midrule
    Actor objective              & BC on top-$K$ winners       & $\alpha\log\pi - Q$                                  \\
    Elite filter $C$ / $J$ / $K$ & $64$ / $256$ / $\{1,4\}$    & n/a                                                   \\
    Updates per env step         & $8$ (BC)                    & $1$                                                  \\
    Entropy                      & $0.05$ (fixed)              & auto, target $\bar{H}{=}{-}\tfrac{1}{2}|\mathcal{A}|$ \\
    Critic                       & n/a                         & clipped double-$Q$, Polyak $\tau$                    \\
    \end{tabular}
\end{table}

\paragraph{LLM-agent RL.}
For both Search-R1 and Sokoban, all settings in Table~\ref{tab:hyperparams-shared} are shared across all tested algorithms.
The replay buffer size is listed at the bottom of Table~\ref{tab:hyperparams-shared}.
GRPO replaces FTW's elite filtering with group-relative advantages over $n$ rollouts per prompt, combined with a PPO-clip objective rather than the cross-entropy SFT loss.
PPO replaces the group baseline with a learned critic, a second Qwen2.5-VL-3B-Instruct initialised from the base model and trained at lr $10^{-5}$ with value clipping $0.5$, using GAE with $\lambda{=}1$ and $\gamma{=}1$ so the advantage is the Monte-Carlo return minus the value estimate.
REINFORCE++ \citep{hu2025reinforcestabilizingcriticfreepolicy} keeps a single rollout per prompt and normalises returns by the global batch mean and standard deviation, with the same PPO-clip objective, KL penalty, and entropy coefficient as GRPO.
Both baselines run directly through the verl-agent implementation \citep{feng2025groupingroup} by switching its advantage estimator, with all other settings held fixed.
The FTW configurations differ only in the minibatch size $C$ and the number of elites per minibatch $K$.
On Search-R1 we use $C{=}4$ with $K{=}1$, and on Sokoban $C{=}5$ with $K \in \{1, 2, 3\}$ retained trajectories, sweeping from strong to weak selection pressure.
The limiting case $C{=}2$, $K{=}1$ is binary selection, which coincides with DPO's pairwise potential and is not run here.
For all FTW configurations, the number of minibatches per gradient step $J$ is set proportional to $K$ such that each update receives $J\cdot K$ training targets at a fixed effective batch size.
The cross-environment differences in Table~\ref{tab:hyperparams-shared} (e.g.\ KL coefficient, discount $\gamma$) are inherited from the published baseline settings for each benchmark rather than tuned for FTW.
Default system prompts for the LLMs are kept at their default values as provided by the packages pinned in Table~\ref{tab:software-stacks}.

\begin{table*}[ht]
    \centering
    \caption{
        LLM-agent hyperparameters.
        All runs use \texttt{bfloat16} rollouts and optimize with AdamW \citep{loshchilov2019decoupled}.
        Unless otherwise stated, we use the exact hyperparameter defaults as described in the respective baselines as implemented by verl-agent~0.3.1.dev \citep{feng2025groupingroup}.
        For FTW, the number of minibatches per update is $J = 128 / K$, so every update uses $J \cdot K = 128$ training targets regardless of $K$.
        }
    \label{tab:hyperparams-shared}
    \renewcommand{\arraystretch}{1.2}
    \vspace{1em}
    \setlength{\tabcolsep}{4pt}
    \begin{tabular}{@{}lll@{}}
    \textbf{Parameter}                            & \textbf{SearchR1}        & \textbf{Sokoban} \\
    \midrule
    Learning rate                                 & $1{\times}10^{-6}$       & $1{\times}10^{-6}$ \\
    Weight decay                                  & $0.01$                   & $0.01$ \\
    Gradient clip                                 & $1.0$                    & $1.0$ \\
    Rollout batch (prompts)                       & $256$                    & $32$ \\
    Rollouts per prompt $n$ (GRPO only)           & $5$                      & $8$ \\
    Mini-batch size                               & $512$                    & $64$ \\
    Epochs per update                             & $1$                      & $1$ \\
    KL coefficient                                & $10^{-3}$                      & $10^{-3}$ \\
    Entropy coefficient                           & $0.001$                      & $0.001$ \\
    Invalid-action penalty                        & $0.01$                   & $0.1$ \\
    $\gamma$ (multi-turn discount)                & $0.95$                   & $1.0$ \\
    Temperature (train / val)                     & $1.0$ / $0.4$            & $1.0$ / $0.4$ \\
    Top-$p$ / top-$k$                             & $1.0$ / $-1$             & $1.0$ / $-1$ \\
    Max prompt / response                         & $4096$ / $512$           & $1024$ / $512$ \\
    Max model len (vLLM)                          & $8192$                   & $1536$ \\
    \midrule
    Replay buffer size $N$                              & $10\;000$                   & $2000$ \\
    \end{tabular}
\end{table*}

\clearpage
\section{Supplementary Results}
\label{appendix:results}

This appendix collects results that complement the main paper.
Section~\ref{appendix:results:training} reports in-distribution learning curves for the LLM-agent experiments (test-time curves appear in the main text, Figure~\ref{fig:llm_post_training}).
Section~\ref{appendix:results:sigma_toggle} dissects the FTW selection operator on the $\sigma$-trap, grading the polynomial filter by selection pressure and sweeping the temperature of an exponential-rank alternative.
Section~\ref{appendix:results:bandit_sweep} extends the Pareto-$\alpha$ sweep of Figure~\ref{fig:tabular_auc} to four additional return distributions, isolating which features of the return distribution drive the failure mode of cardinal-exponential potentials.
Section~\ref{appendix:results:state_aliasing} expands the state-aliasing experiment with a learning-rate sweep, showing that the polynomial-vs-exponential concentration gap is not an artifact of a single learning rate (which dampens learning instability).
Section~\ref{appendix:results:tie_saturation} identifies a solution bias of elite filtering over the replay buffer when returns saturate and comparisons are dominated by ties, complementing the robustness results above with the adversarial regime of the rank-based potential.

\subsection{Training-time Success Rates for Sokoban and Search-R1}
\label{appendix:results:training}

\begin{figure}[h]
    \centering
    \begin{minipage}[t]{0.49\linewidth}
        \vspace{0pt}
        \includegraphics[width=\linewidth]{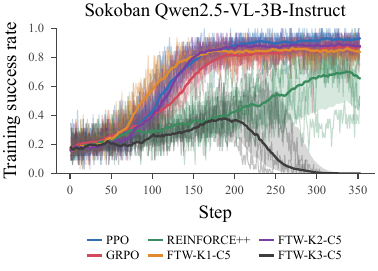}
    \end{minipage}
    \hspace{\fill}
    \begin{minipage}[t]{0.49\linewidth}
        \vspace{0pt}
        \includegraphics[width=\linewidth]{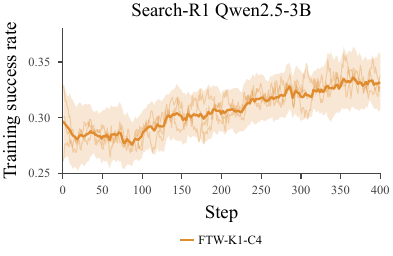}
    \end{minipage}
    \caption{
    Online episode success rate on Sokoban (left) and Search-R1 (right) with EMA-smoothed means and $\pm$ 1 standard deviation over seeds, individual seeds are drawn faintly.
    Note that, the Search-R1 training split covers only NQ and HotpotQA, so the right panel is not on the same question distribution as the main-text reported test accuracy.}
    \label{appendix:figure:training_curves}
\end{figure}

Figure~\ref{appendix:figure:training_curves} shows success rates on training episodes that complement the test-time evaluations in the main text.
On Sokoban (left), our FTW-K1-C5 and FTW-K2-C5 and the GRPO and PPO baselines cluster almost indistinguishably together, whereas REINFORCE++ trains markedly slower and with larger seed variance.
In contrast, our FTW-K3-C5 peaks on average near $0.4$ and then collapses to zero.
This collapse can partly be attributed to Sokoban's effectively binary returns.
Whenever a group contains fewer than $K$ successes the remaining elite slots are always filled with zeros, and such groups are far more common at $K{=}3$ than at $K{=}1$, where one success out of $C$ already suffices.
Although the imitated elite threshold still averages above the buffer mean in expectation (Appendix~\ref{ap:prop:risk_table}), the share of imitated zeros grows as the success rate falls and this drags each consecutive update further down.
Ultimately this is a finite-sample effect rather than a flaw in the operator, a larger replay buffer would smooth the K3 curve out at the cost of slower learning.
On Search-R1 (right), FTW-K1-C4 dips slightly during initial training, and then steadily improves.
Note that the training success rate is not directly comparable to the test success rate, due to the validation set only covering NQ and HotpotQA questions.

\subsection{Varying FTW Selection Operators for the noisy Bandit-trap}
\label{appendix:results:sigma_toggle}

\begin{figure}[t]
    \centering
    \includegraphics[width=0.5\linewidth]{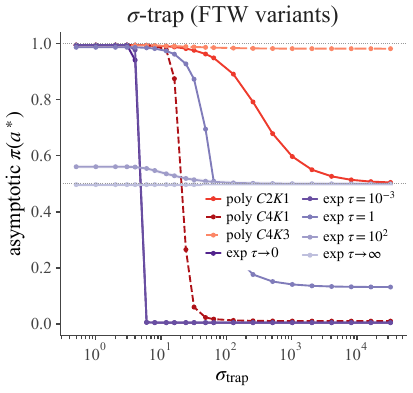}
    \caption{$\sigma$-trap solutions for FTW selection variants (10 seeds; shaded: bootstrap 95\% CIs).}
    \label{appendix:figure:sigma_toggle}
\end{figure}

Extending upon Figure~\ref{fig:tabular_auc} (left) we provide more examples of the FTW filtering pressure on the $\sigma$-trap in Figure~\ref{appendix:figure:sigma_toggle}.
All variants share the identical backbone of FTW, they only differ in how the filtering (winner selection) is applied to the replay buffer.
As in the main text, results are averaged over buffer sizes $h \in \{16, 64, 256, 1024, 4096\}$ for FTW-polynomial and with $h \in \{16, 64\}$ for the exponential filtering on the ranking, as FTW-exp requires sorting the entire replay buffer, which gets expensive.
Note that the FTW-exp-marg from the main text (Fig.~\ref{fig:tabular_auc}) filtered using a softmax over the buffer returns, whereas this experiments uses a multinomial over their ordered rank (argsort index) for a given temperature $\tau$.
This was done to eliminate the ordinal vs. marginal confounding argument of our FTW operator.

The polynomial selection (red lines, Figure~\ref{appendix:figure:sigma_toggle}) are determined by the selection pressure $\nicefrac{C}{K}$.
We see for a mild $C{=}4, K{=}3$ filter that the risk-profile stays minimal, even at larger $\sigma$.
In contrast the more aggressive $C{=}4, K{=}1$ filter collapses onto the trap arm, as suspected.
For the exponential-ordinal filter, we see that the greedy limit $\tau \to 0$ (buffer argmax) collapses onto the trap even swifter than the aggressive polynomial configuration. 
Unsurprisingly, the other extreme $\tau\rightarrow \infty$ stays constant to its 1/2 prior.
The temperatures inbetween $\tau \in \{0.01, 1.0, 100.0\}$ form a type of U-curve that finds the optimal arm for moderate or low noise (temp = 1.0) before interpolating back to uniform for large $\tau$ or back to the overly aggressive $\tau=0$ version.

\subsection{Bandit Distribution Sweep}
\label{appendix:results:bandit_sweep}

\begin{figure}[p]
    \centering
    \includegraphics[width=0.55\linewidth]{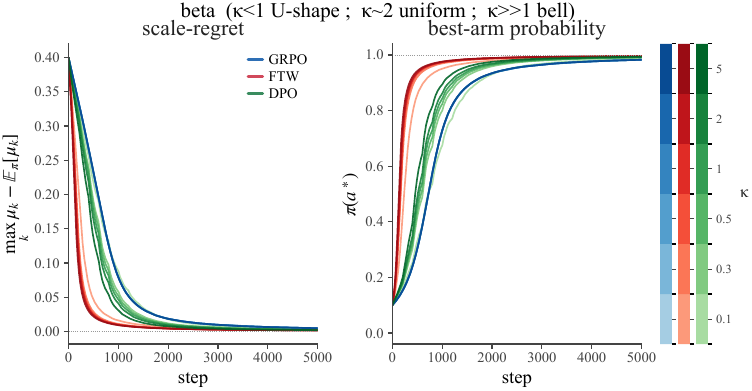}
    \includegraphics[width=0.55\linewidth]{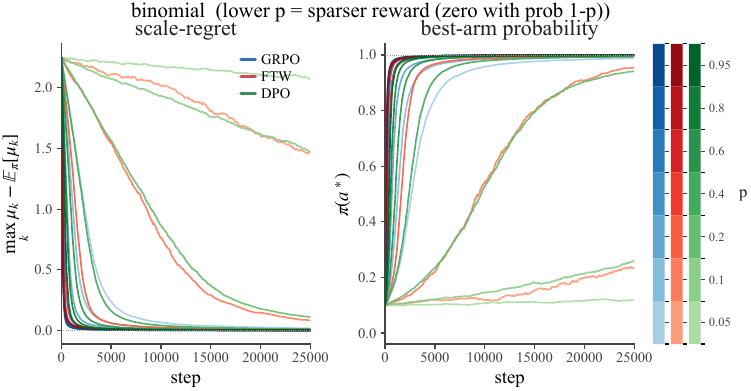}
    \includegraphics[width=0.55\linewidth]{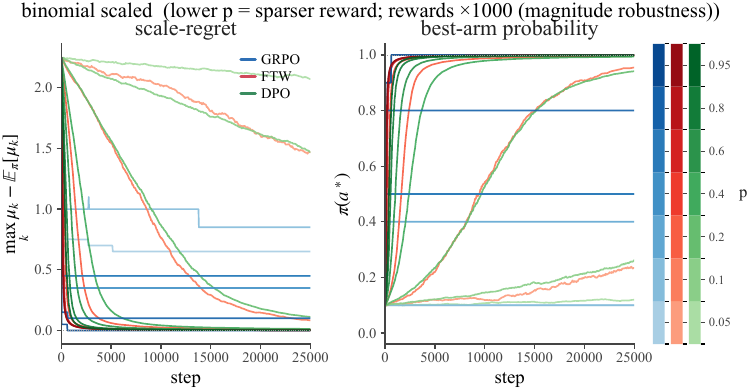}
    \includegraphics[width=0.55\linewidth]{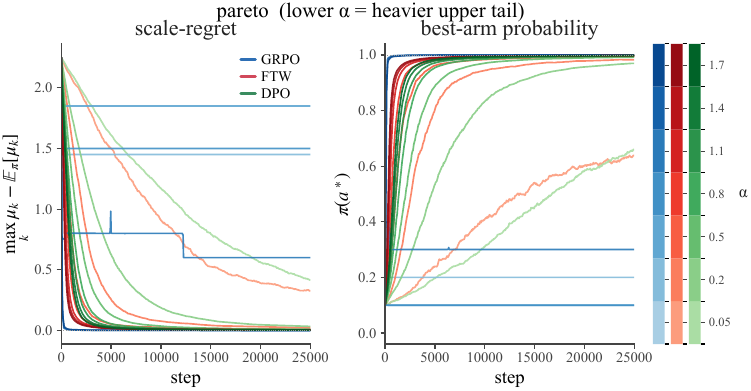}
    \includegraphics[width=0.55\linewidth]{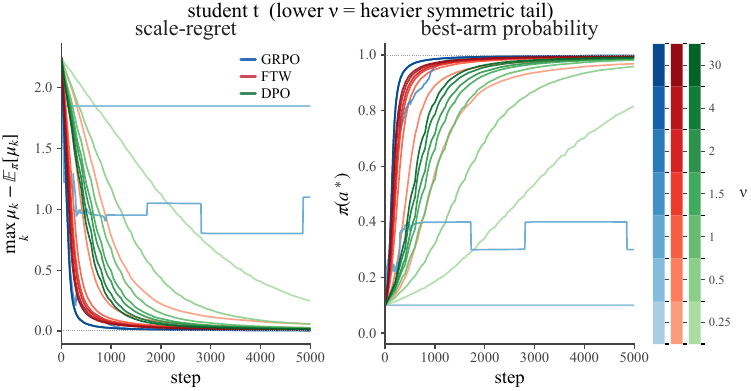}
    \caption{Bandit sweep across five return distributions (Beta, Binomial, scaled Binomial, Pareto, Student-$t$). Left column shows the scale-regret $\max_k \mu_k - \mathbb{E}_\pi[\mu_k]$, right column shows the probability of selecting the optimal arm. Color encodes the distribution-shape parameter, and we use the same legend as our Pareto-$\alpha$ result in the main text.}
    \label{appendix:figure:bandit_distributions_supp}
\end{figure}

Figure~\ref{appendix:figure:bandit_distributions_supp} extends the Pareto-$\alpha$ sweep of Figure~\ref{fig:tabular_auc} (center) to four additional return distributions.
The five rows isolate distinct features of the return distribution and check whether the theoretical risk profile of Proposition~\ref{prop:risk} predicts the empirical failure modes correctly.

\textbf{Beta ($\kappa$ shape).}
Bounded support, where varying $\kappa$ moves between U-shaped, uniform, and bell-shaped returns.
All three methods solve the bandit, and FTW shows slightly faster scale-regret reduction in the early phase.
As expected, no method is destabilized when returns are bounded, regardless of shape.

\textbf{Binomial.}
The clean sparsity sweep.
All three methods identify the optimal arm.
GRPO converges visibly slower than the rank-based methods at low $p$, consistent with its cardinal treatment of returns under sparse signals.

\textbf{Binomial scaled ($\times 1000$).}
The same Bernoulli-rate sweep with rewards scaled by $1000$, isolating magnitude robustness from sparsity.
GRPO's curves remain bounded but the slow-convergence gap at low $p$ widens substantially under the scale.
FTW and DPO are scale-invariant by construction (rank-based), and their curves overlap their unscaled counterparts.
This confirms that magnitude (not sparsity per se) is the destabilizing factor for the cardinal baseline.

\textbf{Pareto-$\alpha$.}
Reproduces the main-text result for the non-truncated version of the distribution.
GRPO fails to identify the optimal arm at heavy-tail $\alpha \leq 1$, where the moment-generating function diverges.
FTW and DPO remain stable, as predicted by the rank-based potentials' well-defined fixed point under any continuous return distribution.

\textbf{Student-$t$ ($\nu$).}
Symmetric, polynomially-tailed distribution on $\mathbb{R}$.
At $\nu \leq 2$ the variance of returns becomes undefined, and GRPO's group-mean baseline destabilizes accordingly, with scale-regret curves drifting noticeably.
FTW and DPO remain unaffected, since the rank-based potential depends only on the order statistic.
This confirms that the destabilization of GRPO under heavy tails is not specific to the one-sided support of Pareto.

Across all five rows, the rank-based potentials never destabilize, while GRPO destabilizes precisely in the regimes where the moment-generating function diverges.
This matches the prediction of Proposition~\ref{prop:risk}.
The cardinal-exponential potential has an undefined fixed point under heavy tails, while rank-based potentials depend only on the order statistic and are insensitive to tail behavior.

\subsection{State-Aliasing under Learning-Rate Variation}
\label{appendix:results:state_aliasing}

\begin{figure}[h]
    \centering
    \includegraphics[width=0.9\linewidth]{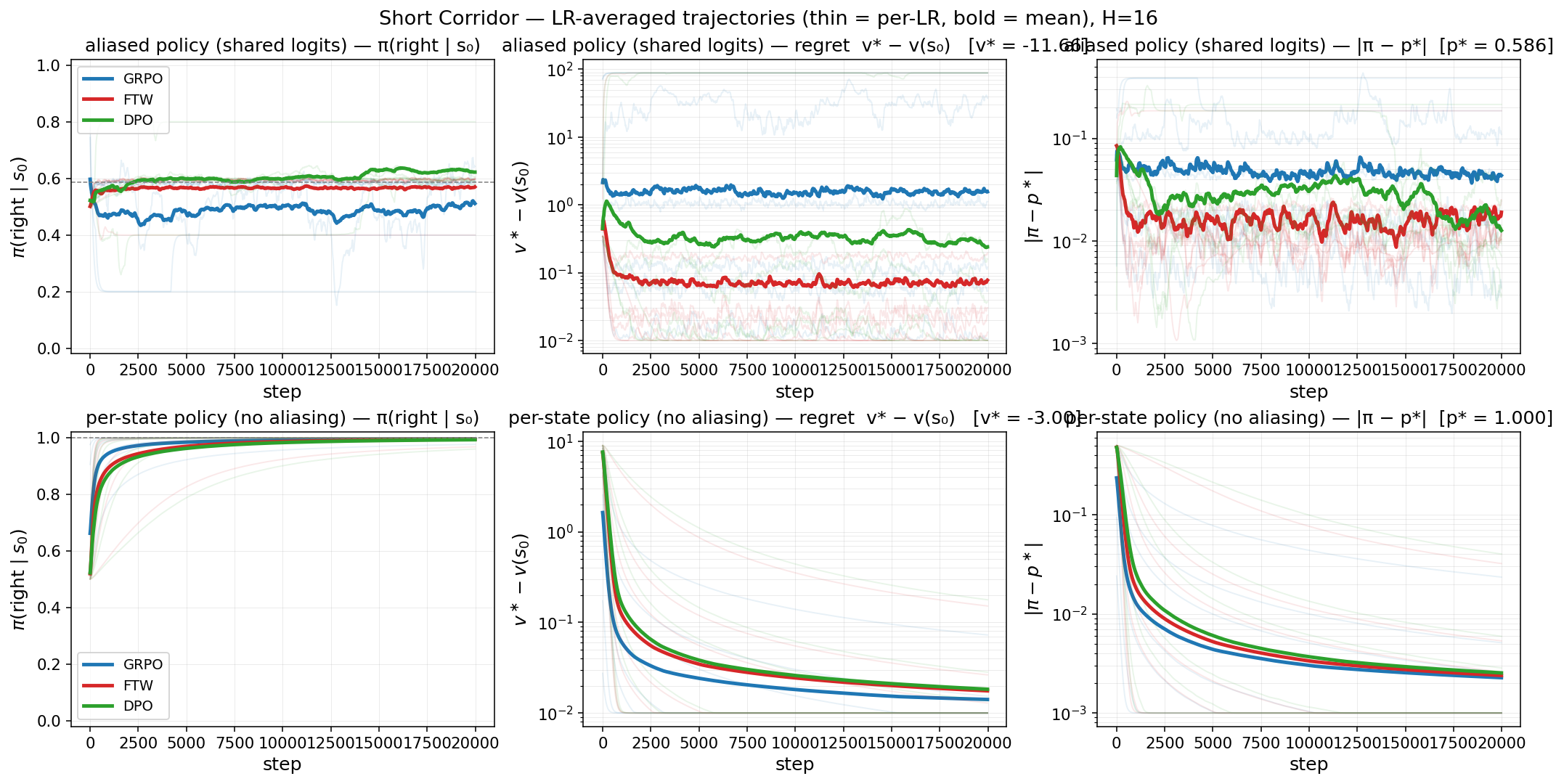}
    \caption{State-aliasing (corridor) task with a learning-rate sweep at horizon $H=16$. We plot average L1 distance $|\pi(a) - \pi^*(a)|$ between the learned and optimal action distributions, averaged across learning rates within each method. Lower is better.}
    \label{appendix:figure:corridor_lr}
\end{figure}

Figure~\ref{appendix:figure:corridor_lr} extends the state-aliasing result of Figure~\ref{fig:tabular_auc} (right) by sweeping the learning rate.
The main text reports a single curve per method.
Here we average across a learning-rate population to verify that FTW's lower L1 distance is not a property of a favorable learning-rate choice.
The ordering FTW $<$ DPO, GRPO holds across the swept range, supporting the claim that polynomial concentration retains more useful stochasticity under partial observability than exponential concentration, independently of step-size tuning.

\subsection{FTW under Tie-Saturated Rewards}
\label{appendix:results:tie_saturation}

This subsection briefly characterizes the effect of tie-breaking when returns frequently saturate around similar values, which is an important discussion item as it eliminates the discriminative signal that elite-filtering is supposed to provide. 
This is common under Binomial returns or sparse discrete reward settings, like our LLM post-training on SearchR1.
We isolate the effect in a two-armed bandit with Bernoulli returns $G \in \{0, 1\}$ and success probabilities $p_1 = 0.999$ and $p_2 < p_1$, whose difference defines the gap $\delta \eqdef p_1 - p_2$, swept over $\delta \in \{0.099, 0.009, 0.004, 0.001, 10^{-4}\}$.
FTW is configured as in the main text with $C = 2$, $K = 1$, $J = 32$, and we ablate the buffer's maximum sample age $h$, where $h = \infty$ never deletes data. 
On-policy GRPO is used as a reference.

\paragraph{Mechanism.}
Consider a single FTW update in this setting.
When both arms succeed with probability close to one, the two sampled returns in a pairwise comparison are almost always $1$ and $1$, which makes the arms close to indistinguishable.
To make this precise, let $\hat{q}(a)$ denote the buffer-occupancy of arm $a$.
The probability that arm $a$ wins from $b$ under the sketched Bernoulli is $p_a(1 - p_b)$, and suppose it wins half of the tied cases giving $\Pr(a \text{ beats } b) = \tfrac{1}{2}(1 + p_a - p_b)$.
This gives the general tilting distribution as a special case of Eq.~\eqref{eq:probability_of_tourney_win},
\begin{equation}
    \Pr(\text{elite} = a) \;\propto\; \hat{q}(a)\,\bigl(1 + p_a - \mathbb{E}_{b \sim \hat{q}}[p_b]\bigr).
    \label{eq:tie_winner_dist}
\end{equation}
Where $\mathbb{E}_{b \sim \hat{q}}[p_b] = \hat{q}(1)\,p_1 + \hat{q}(2)\,p_2$ is the average success probability of a random buffer sample.
In expectation, the better arm receives the factor $1 + \hat{q}(2)\,\delta$ and the worse arm the factor $1 - \hat{q}(1)\,\delta$.
Selection therefore prefers the better arm, but only by a factor of at most $1 + \delta$ regardless of the buffer contents, so under saturation elite-filtering is almost indistinguishable from the buffer occupancy.

Under saturation, learning then entails imitating the buffer plus an $\mathcal{O}(\delta)$ nudge toward the better arm.
Thus, the policy can only improve as fast as the buffer itself moves.
With a finite sample age $h$, old data expires, the buffer tracks the recent and improving policy, and the nudge compounds over updates.
With $h = \infty$ nothing expires.
Each new batch is then a vanishing $\nicefrac{1}{t}$ fraction of the buffer, so the policy is perpetually trained to imitate its own near-uniform early history.

GRPO does not inherit this property for a simple reason.
Its advantage $G_i - \bar{G}$ is exactly zero on a tied group, so where elite filtering must imitate, GRPO simply performs no update and waits for an informative batch.
Thus, FTW's learning dynamic is to oscillate around indistinguishable winners, whereas GRPO stays put. 

\begin{figure}[t]
    \centering
    \includegraphics[width=0.99\linewidth]{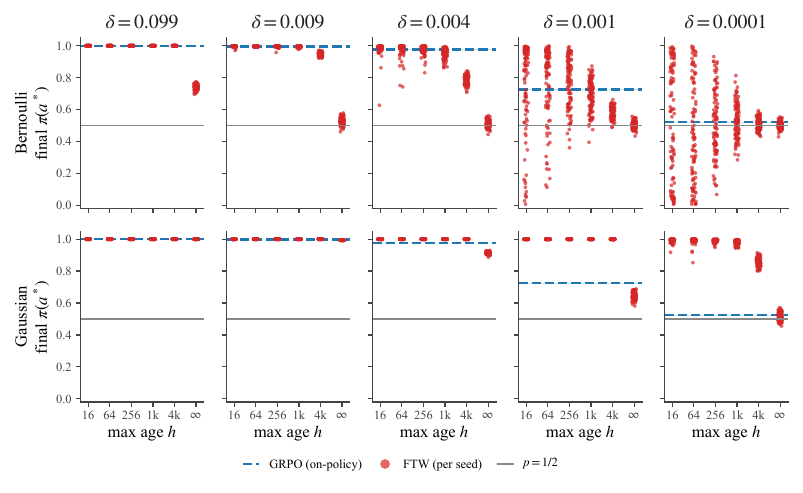}
    \caption{Population solutions after $T = 8 \times 10^5$ steps of training vs.\ maximum sample age $h$ (100 seeds). Dashed line: on-policy GRPO mean and 95\% CI. The top row shows results under Bernoulli returns and the bottom row under homoscedatic Gaussian noise $\mathcal{N}(\mu_a, 0.01)$ returns.
    }
    \label{appendix:figure:tie_strips}
\end{figure}

\begin{figure}[t]
    \centering
    \includegraphics[width=0.99\linewidth]{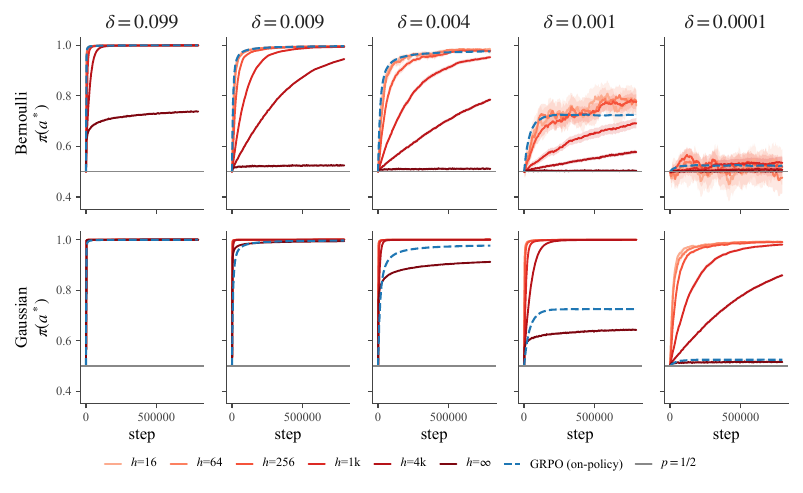}
    \caption{Mean learning curves for Figure~\ref{appendix:figure:tie_strips} (shaded: 95\% CI).
    Evidently, convergence takes proportionally longer as $\delta$ decreases.
    }
    \label{appendix:figure:tie_curves}
\end{figure}

\paragraph{Result.}
Figure~\ref{appendix:figure:tie_strips} shows the population of solutions across buffer sizes, and Figure~\ref{appendix:figure:tie_curves} shows the corresponding average learning curves.
Under Bernoulli returns (top row), the never-delete buffer anchors FTW at chance for small $\delta$, and only the large gap $\delta = 0.099$ produces a marginal shift toward the optimal arm, as the tilt in Eq.~\eqref{eq:tie_winner_dist} scales with $\delta$.
Finite buffers avoid the anchor because dropping old samples lets the policy improvement compound, but smaller buffers turn tie-breaking into selection and update noise, whereas GRPO's tied advantages vanish and its learning is noticeably more stable.
Under Gaussian returns (bottom row), every finite buffer for FTW finds the optimal arm whereas GRPO's learning signal seems to get stale for small $\delta$.

The key takeaway is that tie-breaking is inherent to any RL algorithm on saturated rewards, but the ordinal tilting of methods like FTW induces different learning dynamics than typical marginal methods.
The results present a buffer-size trade-off: larger buffers reduce selection noise and in expectation resolve ties through more accurate non-parametric estimation, at the cost of smoothing and slower policy movement.

\subsection{Epistemic Contextual Bandit}
\label{appendix:results:epistemic}

This subsection is a case study in extracting the \emph{maximum} policy improvement that a selection operator and replay buffer allow.
The mirror-descent update of Eq.~\eqref{eq:belief_update} sets the learning rate and the KL weight exactly as derived through our theory, eliminating the need for tuning.
At least this goes for FTW and DPO, GRPO in contrast cannot be instantiated as a graphical model and its centred-advantage loss admits no posterior precision update rule that yields a positive-definite covariance (Appendix~\ref{appendix:belief}).
Thus, GRPO is run in three imperfect variants to approximate FTW's and DPO's learning rule, one version where normal GRPO is tuned for a learning rate, one where we assume that the covariance is learned through FTW's update rule (ONG), and one where we use the local Fisher information as a gradient preconditioner (NPG).
The environment is the linear contextual bandit of Appendix~\ref{appendix:experimental_setup:environments}, rewards are deterministic and bounded, every arm has the same marginal $\mathbb{E}_x[r(x,a)] = \nicefrac{1}{2}$, so all uncertainty is epistemic and positional structure needs to be exploited for low regret.

The reason why we structured the learning rule and policy for this case study in this particular way is because the explicit Gaussian belief and the bootstrapping from the replay buffer form an analogous stand-in to how data is generated and resampled in typical LLM post-training. 
The LLM predictive (policy) behaves like an amortized posterior predictive that inserts its samples into the buffer. 
Resampling from that buffer exposes implicitly what the model does and does not know, at least structurally.
This case-study then tests: \emph{given} an agent whose uncertainty is represented explicitly and changing over time, does elite filtering act on that uncertainty and reduce it over time?
An affirmative answer would hint at the idea that our framework and instantiation through FTW induces knowledge-seeking in LLM post-training.
This question is also different to plain in-context learning, as data shouldn't only reduce posterior uncertainty, they should also concentrate on the optimal policy.

\begin{figure}[t]
    \centering
    \includegraphics[width=\linewidth]{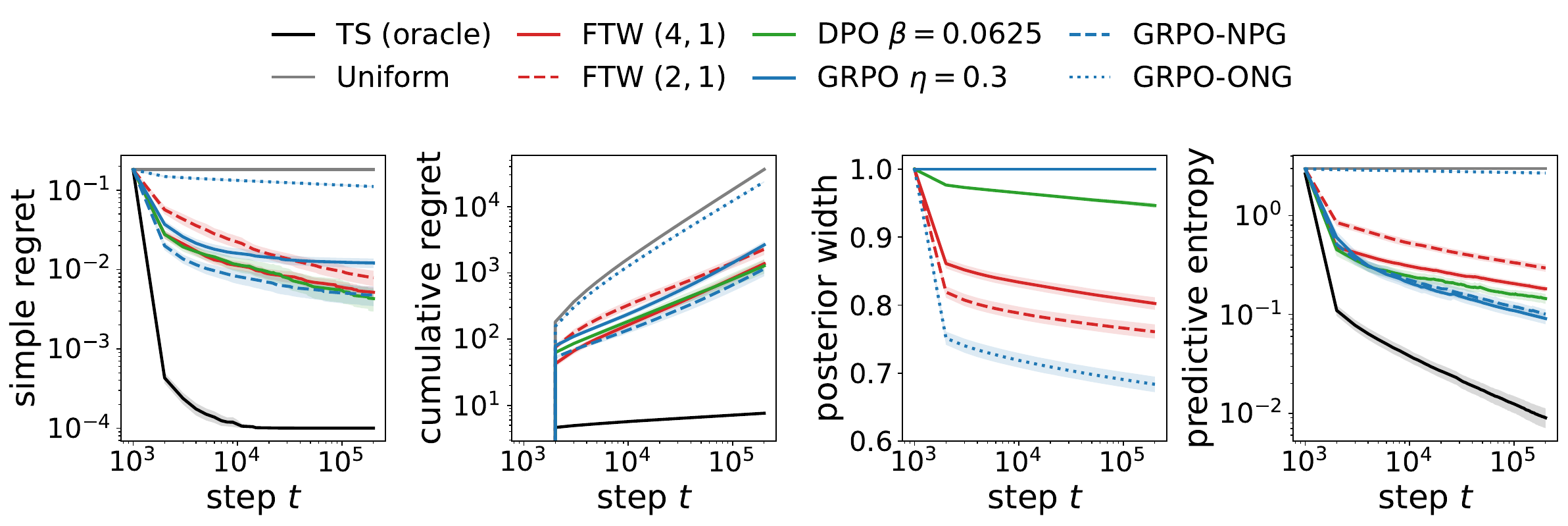}
    \includegraphics[width=\linewidth]{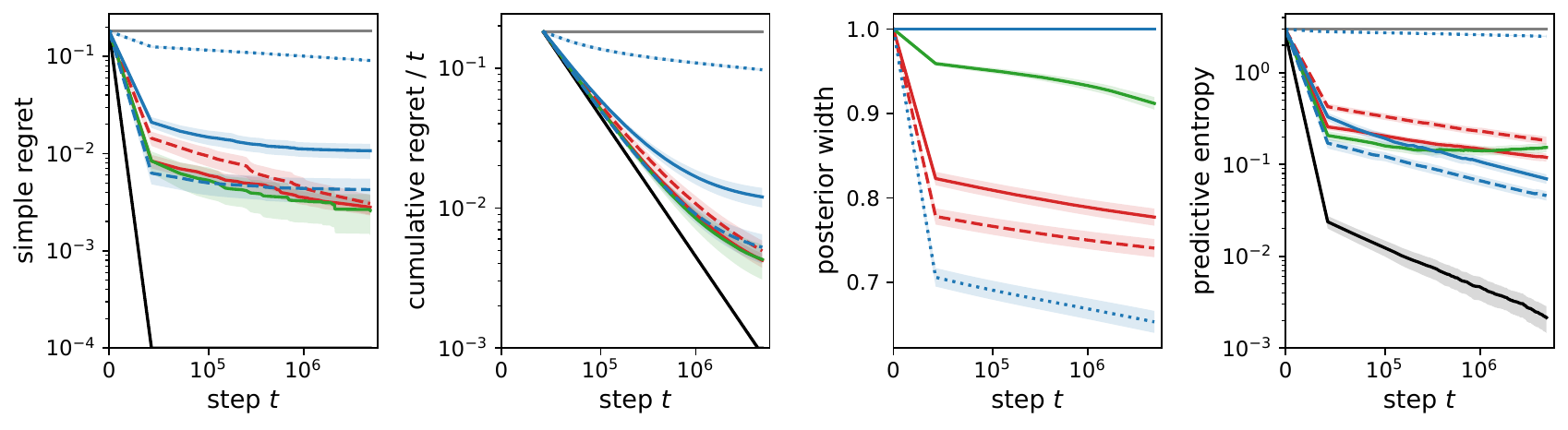}
    \caption{Learning curves at $h = 128$ (30 seeds, mean and Gaussian $95\%$ bands). Left to right: offline simple regret of the posterior predictive, cumulative regret divided by $t$, posterior width, and predictive entropy. Colors denote the method family and line styles its variants. Thompson sampling is an oracle handed the true linear-Gaussian structure whereas the policy methods must learn it through the variational learning rule.}
    \label{appendix:fig:epistemic_headline}
\end{figure}

\begin{figure}[t]
    \centering
    \includegraphics[width=\linewidth]{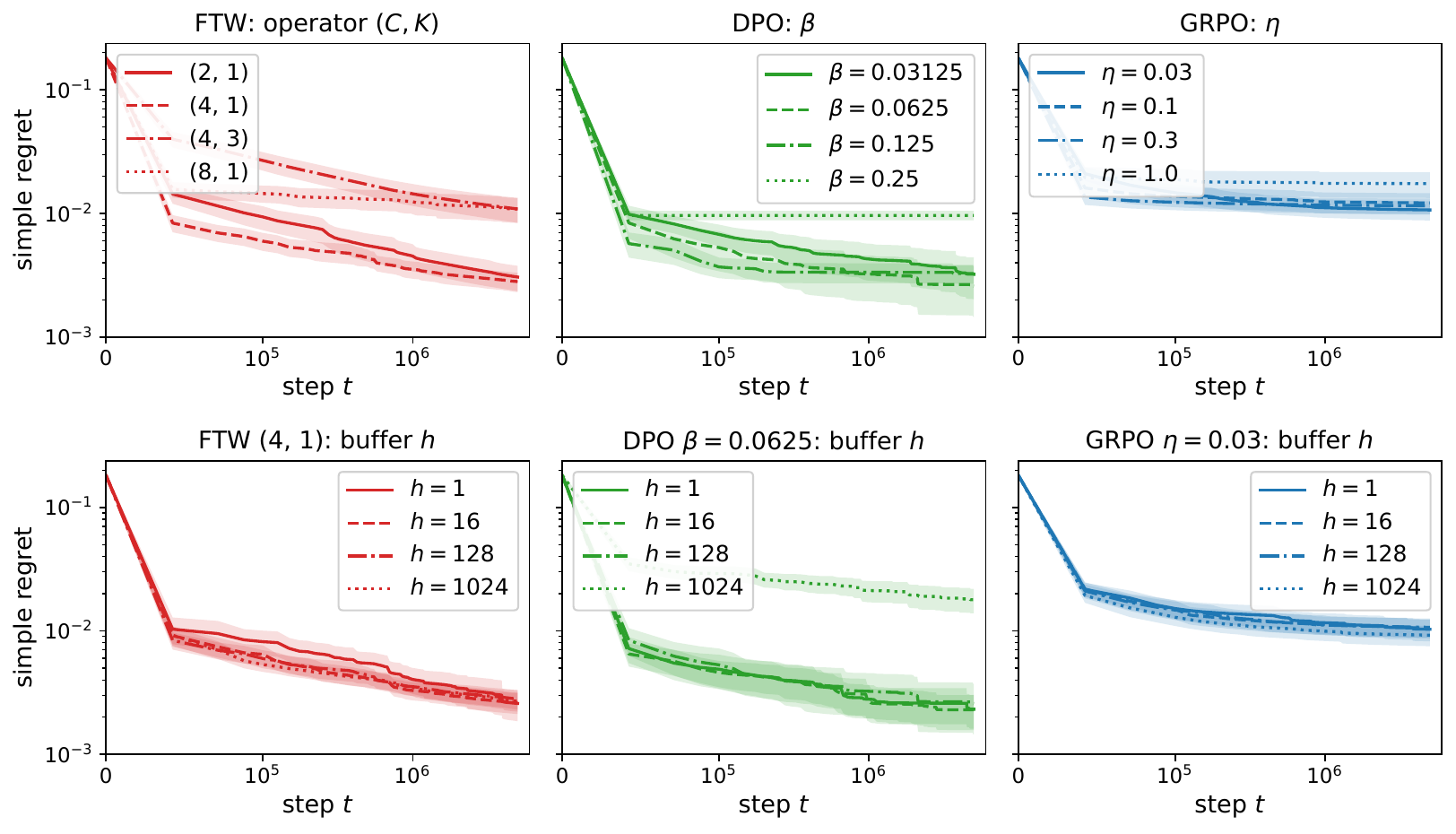}
    \caption{Offline simple regret (30 seeds, Gaussian $95\%$ bands). Top row: each method's concentration knob at $h = 128$. Bottom row: the buffer-depth ladder at each method's best knob setting.}
    \label{appendix:fig:epistemic_sweeps}
\end{figure}

\paragraph{Results.}
Figure~\ref{appendix:fig:epistemic_headline} shows the learning curves of the main configurations, we show regret-metrics alongside the belief and predictive diagnostics.
Then Figure~\ref{appendix:fig:epistemic_sweeps} sweeps each method's concentration parameters (top row) and the buffer depth (bottom row).

One somewhat obvious finding is that too little selection pressure never concentrates the policy, and too much prematurely converges on early winners before the belief has identified them, so regret decreases faster for intermediate selection pressure under each methodology (Figure~\ref{appendix:fig:epistemic_sweeps}, top row).
The effect is of similar size for FTW and DPO, roughly a fourfold spread in final simple regret across their knobs (Figure~\ref{appendix:fig:epistemic_sweeps}, top row).
In general, FTW and DPO are statistically indistinguishable at their best parameter-configurations.

What is also interesting is how the belief precision differs between FTW and DPO, despite achieving similar regret. 
The posterior width for FTW decreases more sharply than for DPO, and its predictive entropy ends lower as well (right of Figure~\ref{appendix:fig:epistemic_headline}).
At equal regret, DPO thus retains a wider belief and a more diffuse predictive than FTW.
GRPO's belief metrics are either stale or overcontract due to the model misspecification or mismatch that prevents implementing a proper update rule in our framework.

Plain GRPO converges prematurely at every fixed step size, whereas the Fisher preconditioner of GRPO-NPG removes most of this gap in both simple and cumulative regret (Figure~\ref{appendix:fig:epistemic_headline}).
The misspecified FTW precision update used by GRPO-ONG contracts the belief before the mean has moved, so its policy stays near uniform and it achieves the worst regret of all tested methods aside from the uniform upper bound.
Finally, the buffer-depth ladder (Figure~\ref{appendix:fig:epistemic_sweeps}, bottom row) shows FTW insensitive to $h$ over three orders of magnitude, whereas DPO degrades sharply at $h = 1024$ and GRPO is flat at its higher regret level.

\end{document}